\documentclass{article}

    \usepackage[final]{neurips_2020}

\usepackage{natbib}
\usepackage[utf8]{inputenc} % allow utf-8 input
\usepackage[T1]{fontenc}    % use 8-bit T1 fonts
\usepackage{hyperref}       % hyperlinks
\usepackage{url}            % simple URL typesetting
\usepackage{booktabs}       % professional-quality tables
\usepackage{amsfonts}       % blackboard math symbols
\usepackage{nicefrac}       % compact symbols for 1/2, etc.
\usepackage{microtype}      % microtypography
\usepackage{subfigure}
\usepackage{graphicx}
\usepackage{multirow}
\usepackage{multicol}
\usepackage{amsthm, amssymb, color}

\usepackage{amsmath,amsfonts,bm}

\def\eqref#1{equation~\ref{#1}}
\def\1{\bm{1}}

\def\mX{{\bm{X}}}

\DeclareMathAlphabet{\mathsfit}{\encodingdefault}{\sfdefault}{m}{sl}
\SetMathAlphabet{\mathsfit}{bold}{\encodingdefault}{\sfdefault}{bx}{n}

\newcommand{\E}{\mathbb{E}}

\newcommand{\Var}{\mathrm{Var}}

\usepackage{algorithm}
\usepackage{algorithmic}

\newcommand{\xeta}{x}

\newcommand{\by}{{\boldsymbol y}}

\newcommand{\bx}{{\boldsymbol x}}

\newcommand{\bw}{{\boldsymbol w}}

\newcommand{\MX}{{\mathcal{X}}}
\newcommand{\bchi}{{\boldsymbol{\chi}}}
\newcommand{\btheta}{{\boldsymbol{\theta}}}

\newcommand{\bTheta}{{\boldsymbol \Theta}}

\newcommand{\bT}{{\boldsymbol T}}

\definecolor{darkblue}{rgb}{0.0, 0.0, 0.55}
\definecolor{cornellred}{rgb}{0.7, 0.11, 0.11}
\definecolor{darkred}{rgb}{0.55, 0.0, 0.0}
\definecolor{red(ryb)}{rgb}{1.0, 0.15, 0.07}
\definecolor{red}{rgb}{1.0, 0.0, 0.0}
\hypersetup{
  pdffitwindow=true,
  pdfstartview={FitH},
  pdfnewwindow=true,
  colorlinks,
  linktocpage=true,
  linkcolor=red(ryb),
  urlcolor=red(ryb),
  citecolor=darkblue
}

\newtheorem{theorem}{Theorem}

\newtheorem{lemma}{Lemma}

\newtheorem{assump}{Assumption}

\title{A Contour Stochastic Gradient Langevin Dynamics Algorithm for Simulations of Multi-modal Distributions}

\author{%
  Wei Deng \\
  Department of Mathematics\\
  Purdue University\\
  West Lafayette, IN, USA \\
  \texttt{weideng056@gmail.com} \\
  \And
  Guang Lin \\
  Departments of Mathematics \& \\
  School of Mechanical Engineering \\
  Purdue University \\
  West Lafayette, IN, USA \\
  \texttt{guanglin@purdue.edu} \\
  \And
  Faming Liang \thanks{To whom correspondence should be addressed: Faming Liang.}\\
  Departments of Statistics \\
  Purdue University \\
  West Lafayette, IN, USA \\
  \texttt{fmliang@purdue.edu} \\
}

\begin{document}

\maketitle 

\begin{abstract}
We propose an adaptively weighted stochastic gradient Langevin dynamics algorithm (SGLD), so-called contour stochastic gradient Langevin dynamics (CSGLD), for Bayesian learning in big data statistics. The proposed algorithm is essentially a \emph{scalable dynamic importance sampler}, which automatically \emph{flattens} the target distribution such that the simulation for a multi-modal distribution can be greatly facilitated. Theoretically, we prove a stability condition and establish the asymptotic convergence of the self-adapting parameter to a {\it unique fixed-point}, regardless of the non-convexity of the original energy function; we also present an error analysis for the weighted averaging estimators. Empirically, the CSGLD algorithm is tested on multiple benchmark datasets including CIFAR10 and CIFAR100. The numerical results indicate its superiority %over  the existing state-of-the-art sampling algorithms 
to avoid the local trap problem in training deep neural networks.

\end{abstract}

\section{Introduction}

AI safety has long been an important issue in the deep learning community. A promising solution to the problem is  Markov chain Monte Carlo (MCMC), which leads to asymptotically correct uncertainty quantification for deep neural network (DNN) models. However, traditional MCMC algorithms \citep{Metropolis1953,Hastings1970} 
are not scalable to big datasets that deep learning models rely on, although they have achieved significant successes in many scientific areas such as statistical physics and bioinformatics. 
  It was not until the study of stochastic gradient Langevin dynamics (SGLD) \citep{Welling11} that resolves the scalability issue encountered in Monte Carlo computing for big data problems. Ever since, a variety of scalable 
  stochastic gradient Markov chain Monte Carlo (SGMCMC) 
  algorithms have been developed based on strategies such as Hamiltonian dynamics \citep{Chen14, yian2015, Ding14}, Hessian approximation \citep{Ahn12, Li16, Simsekli2016}, and higher-order numerical schemes \citep{Chen15, Li19}. Despite their theoretical guarantees in statistical inference \citep{Chen15, Teh16, VollmerZW2016} and non-convex optimization \citep{Yuchen17, Maxim17, Xu18}, these algorithms often converge slowly, which makes them hard to be
  used for efficient uncertainty quantification for many AI safety problems.

To develop more efficient SGMCMC algorithms, we 
seek inspirations from traditional MCMC algorithms, such as simulated annealing \citep{Kirkpatrick83optimizationby}, parallel tempering \citep{PhysRevLett86, Geyer91}, and flat histogram algorithms \citep{Berg1991Multicanonical,wang2}. In particular, simulated annealing proposes to decay temperatures to increase the hitting probability to the global optima \citep{Mangoubi18}, which, however, often gets stuck into a local optimum with a fast cooling schedule.
% and not appropriate for statistical inference any more. 
 Parallel tempering proposes to swap positions of neighboring Markov chains according to an acceptance-rejection rule. However, 
 under the mini-batch setting, it often requires a large correction which is known to deteriorate its performance \citep{deng2020}. The flat histogram algorithms, 
such as the multicanonical \citep{Berg1991Multicanonical} 
and Wang-Landau \citep{wang2} algorithms, 
were first proposed to sample discrete states of Ising models by yielding a flat histogram in the energy space, and then extended as a general dynamic importance sampling algorithm, the so-called stochastic approximation Monte Carlo (SAMC) algorithm \citep{liang05, Liang07, LiangPL2009}. Theoretical studies \citep{leli2008, Liang10, Fort15} support the efficiency of the flat histogram algorithms in Monte Carlo computing for small data problems. 
However, it is 
still unclear how to adapt the flat histogram idea to accelerate the convergence of SGMCMC, ensuring 
efficient uncertainty quantification for AI safety problems. 

This paper proposes the so-called contour stochastic gradient Langevin dynamics (CSGLD) algorithm, which successfully extends the flat histogram idea to SGMCMC. 
Like the SAMC algorithm \citep{liang05, Liang07, LiangPL2009}, 
CSGLD works as a dynamic importance sampling algorithm, which adaptively adjusts the target measure at each iteration and accounts for the bias introduced thereby by  importance weights. However, theoretical analysis for the two types of dynamic importance sampling algorithms can be quite different due to the fundamental difference in their 
transition kernels. We proceed by justifying the stability condition for CSGLD based on the perturbation theory, and establishing ergodicity of CSGLD based on newly 
developed theory for the convergence of adaptive SGLD. 
Empirically, we test the performance of CSGLD through a few experiments. It achieves remarkable performance on some 
synthetic data, UCI datasets, and computer vision datasets 
such as CIFAR10 and CIFAR100.

\section{Contour stochastic gradient Langevin dynamics}

Suppose we are interested in sampling from a probability measure $\pi(\bx)$ with the density given by  
 \begin{equation} \label{CSGLDeq1}
 \pi(\bx) \propto \exp(-U(\bx)/\tau), \quad \bx \in \MX,
 \end{equation}
 where $\MX$ denotes the sample space,
 $U(\bx)$ is the energy function, and 
 $\tau$ is the temperature. It is known that when $U(\bx)$ is highly non-convex, SGLD can mix very slowly \citep{Maxim17}. To accelerate the convergence, we exploit the flat histogram idea in SGLD.
 %stochastic gradient Markov chain Monte Carlo. 
 
 Suppose that we have partitioned the sample space 
 $\MX$ into $m$ subregions based on the energy function 
  $U(\bx)$:   $\MX_1=\{\bx: U(\bx) \leq {u}_1\}$, $\MX_2=\{\bx: {u}_1 < U(\bx) \leq {u}_2\}$, $\ldots$, $\MX_{m-1}=\{\bx: {u}_{m-2} < U(\bx) \leq {u}_{m-1} \}$, and
 $\MX_{m}=\{\bx: U(\bx) >{u}_{m-1} \}$,
 where $-\infty <  {u}_1 < {u}_2 < \cdots < {u}_{m-1} <\infty$ are specified by the user. For convenience, we 
 set $u_0=-\infty$ and $u_m=\infty$. Without loss of 
 generality, we assume ${u}_{i+1}-{u}_{i}=\Delta u$ for $i=1,\ldots,m-2$. We propose to simulate from a flattened density 
 \begin{equation} \label{1keq1} 
  \varpi_{\Psi_{\btheta}}(\bx) \propto \frac{\pi(\bx)}{\Psi^{\zeta}_{\btheta}(U(\bx))},
 \end{equation}
 where $\zeta>0$ is a hyperparameter controlling the geometric property of the flatted density (see Figure \ref{fig: 4a} for illustration),  and $\btheta=(\theta(1), \theta(2), \ldots, \theta(m))$ 
 is an unknown vector 
 taking values in the space: 
 \begin{equation}\small
     \bTheta=\left\{\left(\theta(1),\theta(2),\cdots, \theta(m)\right)\big|0<\theta(1),\theta(2),\cdots, \theta(m)<1 \text{ and } \sum_{i=1}^m \theta(i)=1 \right\}.
 \end{equation}
%\paragraph{A na\"{i}ve flat-histogram SGLD} 
\subsection{A na\"{i}ve contour SGLD} It is known if we set \footnote{$1_{A}$ is an indicator function that takes value $1$ if event $A$ occurs and $0$ otherwise.}
% \begin{equation}
% \small
% \begin{split}
%      \zeta=1, \text{ } \Psi_{\btheta}(U(\bx))= \sum_{i=1}^m \theta(i) 1_{u_{i-1} < U(\bx) \leq u_i},\text{ } \theta(i)=\theta_{\star}(i)=\int_{\bchi_i}\pi(\bx)d\bx \text{ for } i\in\{1,2,\cdots, m\},
% \end{split}
% \end{equation}
\begin{equation}
% \small
\label{proposal_iter}
\begin{split}
     &\text{(i)    }\ \zeta=1 \text{ and }\Psi_{\btheta}(U(\bx))= \sum_{i=1}^m \theta(i) 1_{u_{i-1} < U(\bx) \leq u_i}, \\
      &\text{(ii)   } \theta(i)=\theta_{\star}(i), \text{where }\theta_{\star}(i) =\int_{\bchi_i}\pi(\bx)d\bx \text{ for } i\in\{1,2,\cdots, m\},\\
\end{split}
\end{equation}
the algorithm will act like the SAMC algorithm \citep{Liang07},  
yielding a flat histogram in the space of energy (see the pink curve in Figure \ref{fig: 4b}). Theoretically, 
such a density flattening strategy enables a sharper logarithmic Sobolev inequality and accelerates 
the convergence of simulations \citep{leli2008, Fort15}.
However, such a density flattening setting only works under the framework of the Metropolis algorithm \citep{Metropolis1953}. %, which leads to the vanishing-gradient problem for SGLD because $\frac{\partial \Psi_{\btheta}(u)}{\partial u}=0$ and $\nabla \log \varpi_{\Psi_{\btheta}}(\bx)=\nabla \log \pi(\bx)$ almost everywhere. As a result, it fails to simulate samples from the trial density (\ref{1keq1}).
A na\"{i}ve application of the step function in formula (\ref{proposal_iter}(i)) to SGLD results in $\frac{\partial \log \Psi_{{\theta}}(u)}{\partial u}=\frac{1}{\Psi_{\theta}(u)}\frac{\partial \Psi_{{\theta}}(u)}{\partial u}=0$ almost everywhere, which leads to the \emph{vanishing-gradient problem} for SGLD. Calculating the gradient for the na\"{i}ve contour SGLD, we have 
\begin{equation*}
\small{
    \nabla_{x} \log \varpi_{\Psi_{\theta}}(x)=-\left[1+ \zeta \tau\frac{\partial \log{\Psi_{\theta}}(u)}{\partial u} \right] \frac{\nabla_{x} U(x)}{\tau}=-\frac{\nabla_{x} U(x)}{\tau}.}
\end{equation*} 
As such, the na\"{i}ve algorithm behaves like SGLD and fails to simulate from the flattened density (\ref{1keq1}).

%and 
\subsection{How to resolve the vanishing gradient} To tackle this issue, we propose to set $\Psi_{\btheta}(u)$ 
as a piecewise continuous function: 
\begin{equation}\label{new_design}
\Psi_{\btheta}(u)= \sum_{i=1}^m \left(\theta(i-1)e^{(\log\theta(i)-\log\theta(i-1)) \frac{u-u_{i-1}}{\Delta u}}\right) 1_{u_{i-1} < u \leq u_i},
\end{equation}
where $\theta(0)$ is fixed to $\theta(1)$ for simplicity. A direct calculation shows that
 \begin{equation}
 \label{marK}
 \small
 \begin{split}
     \nabla_{\bx} \log \varpi_{\Psi_{\btheta}}(\bx)
  &=-\left[1+ \zeta \tau\frac{\partial \log{\Psi_{\btheta}}(u)}{\partial u} \right] 
   \frac{\nabla_{\bx} U(\bx)}{\tau}\\
  &=-\left[1+ \zeta \tau {\frac{\log\theta(J(\bx))-\log\theta((J(\bx)-1)\vee 1)}{\Delta u}} \right] 
   \frac{\nabla_{\bx} U(\bx)}{\tau},
 \end{split}
 \end{equation}
where $J(\bx) \in\{1,2,\cdots, m\}$ denotes the index  that $\bx$ belongs to, i.e., $u_{J(\bx)-1}< U(\bx)\leq u_{J(\bx)}$. \footnote[4]{Eq.(\ref{marK}) shows a practical numerical scheme. An alternative is presented in the supplementary material. } 

\subsection{Estimation via stochastic approximation}
Since $\btheta_{\star}$ is unknown, we propose to estimate it on the fly under the framework of 
stochastic approximation \citep{RobbinsM1951}. 
%Since $\btheta_{\star}=\left(\int_{\bchi_1}\pi(\bx)d\bx, \int_{\bchi_2}\pi(\bx)d\bx,\cdots, \int_{\bchi_m}\pi(\bx)d\bx\right)$ is unknown in practice, we propose to adaptively estimate it based on stochastic approximation theory. 
%Inspired by \citep{Fort15, liang05, Liang07}, 
Provided that a scalable transition kernel $\Pi_{\bm{\theta_{k}}}(\bm{x}_{k}, \cdot)$ is available 
 and the energy function $U(\bx)$ on the full data 
 can be efficiently evaluated, the weighted density  $\varpi_{\Psi_{\btheta}}(\bx)$ can be simulated by 
 iterating between the following steps:
\begin{equation}
\begin{split}
\label{sa_framework}
    &\text{ (i) Simulate $\bm{x}_{k+1}$ from $\Pi_{\bm{\theta_{k}}}(\bm{x}_{k}, \cdot)$, which admits $\varpi_{\bm{\theta}_{k}}(\bm{x})$ as
the invariant distribution,} \\
    % &\text{(ii) Update $\btheta_k$ by setting $\bm{\theta}_{k+1}=\bm{\theta}_{k}+\omega_{k+1} H(\bm{\theta}_{k}, \bm{x}_{k+1})$.}
    &\text{(ii) $\theta_{k+1}(i)={\theta}_{k}(i)+\omega_{k+1}{\theta}_{k}^{\zeta}( J(\bx_{k+1}))\left(1_{i= J(\bx_{k+1})}-{\theta}_{k}(i)\right)$ \text{ for } $i\in\{1,2,\cdots,m\}.$}
\end{split}    
\end{equation}
where $\btheta_k$ denotes a working estimate of $\btheta$ at the $k$-th iteration. We expect that in a long run, such an algorithm can achieve \emph{an optimization-sampling equilibrium} such that $\btheta_{k}$ converges to the fixed point $\btheta_{\star}$ and the random vector $\bx_{k}$ converges weakly to the distribution $\varpi_{\Psi_{\btheta_{\star}}}(\bx)$.

To make the algorithm scalable to big data, we adopt the Langevin transition kernel for drawing samples 
at each iteration, for which a mini-batch of data 
can be used to accelerate computation. In addition, evaluating the energy $U(\bx)$ on the full data can be quite expensive, 
while it is free to obtain the stochastic energy $\widetilde{U}(\bx)$ in evaluating 
the stochastic  gradient 
$\nabla_{\bx} \widetilde{U}(\bx)$ due to the nature 
of auto-differentiation \citep{paszke2017}.
For this reason, we propose a stochastic index $J_{\widetilde U}(\bx)$
\begin{equation}\label{stoch_index}
    J_{\widetilde U}(\bx)=\sum_{i=1}^m i 1_{u_{i-1}<\frac{N}{n} \widetilde U(\bx)\leq u_i},
\end{equation}
where $N$ is the sample size of the full dataset and $n$ is the mini-batch size. Let $\{\epsilon_k\}_{k=1}^{\infty}$ and $\{\omega_k\}_{k=1}^{\infty}$ denote the learning rates and step sizes for SGLD and stochastic approximation, respectively. Given the above notations, the proposed algorithm can be presented in Algorithm \ref{alg:CSGLD},
 which can be viewed as  a \emph{scalable Wang-Landau algorithm} for deep learning and big data problems. 
 
\subsection{Related work} 
Compared to the existing MCMC algorithms, the proposed   algorithm has a few innovations:
 \begin{algorithm}[tb]
   \caption{Contour SGLD Algorithm. One can conduct a resampling step from the pool of importance samples according to the importance weights to obtain the original distribution. For more scalable updates, one can adopt the stochastic approximation scheme Eq.(\ref{novel_SA_scheme}) in \citet{icsgld}.}
   \label{alg:CSGLD}
\begin{algorithmic}
   \STATE {\bfseries [1.] (Data subsampling)} Simulate a mini-batch of data of size $n$ from the whole dataset of size $N$; Compute the stochastic gradient $\nabla_{\bx}\widetilde U(\bx_k)$ and stochastic energy $\widetilde U(\bx_k)$.

   \STATE {\bfseries [2.] (Simulation step)}
   Sample $\bx_{k+1}$ using the SGLD algorithm based on $\bx_k$ and $\btheta_k$, i.e.,
  \begin{equation} \label{SGLDeq6}
  \footnotesize
 \begin{split}
  \small{\bx_{k+1}=\bx_k - \epsilon_{k+1} \frac{N}{n} \underbrace{\left[1+ 
   \zeta\tau\frac{\log {\theta}_{k}(J_{\widetilde U}(\bx_k)) - \log{\theta}_{k}((J_{\widetilde U}(\bx_k)-1)\vee 1)}{\Delta u}  \right]}_{\text{gradient multiplier}}  
    \nabla_{\bx} \widetilde U(\bx_k) +\sqrt{2 \tau \epsilon_{k+1}} \bw_{k+1}}, 
 \end{split}
  \end{equation}
  where $\bw_{k+1} \sim N(0,\bm{I}_d)$, $d$ is the dimension,
  $\epsilon_{k+1}$ is the learning rate,
  and $\tau$ is the temperature. 

  \STATE {\bfseries [3.] (Stochastic approximation)} Update the estimate of $\theta(i)$'s
 for $i=1,2,\ldots,m$ by setting
  \begin{equation} \label{updateeq}
 {\theta}_{k+1}(i)={\theta}_{k}(i)+\omega_{k+1}{\theta}_{k}^{\zeta}(J_{\widetilde U}(\bx_{k+1}))\left(1_{i=J_{\widetilde U}(\bx_{k+1})}-{\theta}_{k}(i)\right), 
 \end{equation} 
 where $1_{i=J_{\widetilde U}(\bx_{k+1})}$ is an indicator function which equals 1 if $i= J_{\widetilde U}(\bx_{k+1})$ and 0 otherwise.
\vspace{-0.04in}
\end{algorithmic}
% \vspace{-1.5em}
\end{algorithm}
 
First, CSGLD is an adaptive MCMC algorithm based on the \emph{Langevin transition kernel} instead of the {\it Metropolis transition kernel} \citep{Liang07, Fort15}. As a result, the existing convergence theory for the Wang-Landau algorithm does not apply. 
To resolve this issue, we first prove a stability condition for CSGLD based on the perturbation theory, and then verify regularity conditions for the solution of the Poisson equation so that the fluctuations of the mean-field system induced by CSGLD get controlled, which eventually ensures convergence of CSGLD.

Second, the use of the stochastic index $J_{\widetilde U}(\bx)$ in Eq.(\ref{stoch_index}) avoids the evaluation of $U(\bx)$ on the full data and thus significantly accelerates the computation of the algorithm, although it leads to a small bias, depending on the variance of the energy estimators, in parameter estimation. Compared to other methods, such as using a fixed sub-dataset to estimate $U(\bx)$, the implementation is much simpler. Moreover, combining the variance reduction of the noisy energy estimators \citep{deng_VR}, the bias also decreases to zero asymptotically as $\epsilon\rightarrow 0$. %It also shows a potential to eliminate the bias in parallel settings where the exact energy can be obtained by aggregating the energy estimators in each chain.  

Third, unlike the existing SGMCMC algorithms \citep{Welling11, Chen14, yian2015}, CSGLD works as a \emph{dynamic importance sampler} which \emph{flattens} the target distribution and \emph{reduces the energy barriers} for the sampler to traverse between different regions of the energy landscape (see Figure \ref{fig: 4a} for illustration). The sampling bias introduced thereby is accounted for by the importance weight $\theta^{\zeta}(J_{\widetilde U}(\cdot))$. Interestingly, CSGLD possesses a {\it self-adjusting mechanism} to ease escapes from local traps, which is similar to the self-repulsive dynamics \citep{mao_mcmc} and can be explained in Figure \ref{trajectory}. That is, in order to escape from local traps, CSGLD is sometimes forced to \textbf{move toward higher energy regions by adopting negative learning rates}. This is a very attractive feature for simulations of multi-modal distributions.

\section{Theoretical study of the CSGLD algorithm} \label{convergSect}

In this section, we study the convergence of CSGLD algorithm under the framework of stochastic approximation and show the ergodicity property based on weighted averaging estimators.
 
\subsection{Convergence analysis} \label{FMalg}

Following the tradition of stochastic approximation analysis, we rewrite the updating rule (\ref{updateeq}) 
as 
 \begin{equation}
     \btheta_{k+1}=\btheta_k+\omega_{k+1} \widetilde H(\btheta_k,\bx_{k+1}),
 \end{equation}
where $\widetilde H(\btheta,\bx)=(\widetilde H_1(\btheta,\bx), \ldots, 
 \widetilde H_m(\btheta,\bx))$ is a random field function with 
 \begin{equation}
 \label{H_}
     \widetilde H_i(\btheta,\bx)={\theta}^{\zeta}(J_{\widetilde U}(\bx))\left(1_{i= J_{\widetilde U}(\bx)}-{\theta}(i)\right), \quad i=1,2,\ldots,m.
 \end{equation}
Notably, $\widetilde H(\btheta,\bx)$ works under an empirical measure $\varpi_{\btheta}(\bx)$ which approximates the invariant measure $\varpi_{\Psi_{\btheta}}(\bx)\propto\frac{\pi(\bx)}{\Psi^{\zeta}_{\btheta}(U(\bx))}$ asymptotically as $\epsilon\rightarrow 0$ and $n\rightarrow N$. 
%Moreover, we introduce an auxiliary measure $\varpi_{\btheta}(\bx)\propto \frac{\pi(\bx)}{\theta^{\zeta}(J(\bx))}$ to facilitate the integration of $\widetilde H(\btheta, \bx)$. Note that $\varpi_{\btheta}(\bx)$ is sufficiently close to $\varpi_{\Psi_{\btheta}}(\bx)$ since $\sum_{i=1}^m \theta(i)1_{u_{i-1} < u \leq u_i}\rightarrow \Psi_{\btheta}(u)$ as $m\rightarrow \infty$ (see Lemma x in the appendix).
As shown in Lemma \ref{convex_main}, 
we have the mean-field equation 
\begin{equation} \label{fixedeq}
h(\btheta)=\int_{\MX}\widetilde H(\btheta,\bx)  \varpi_{\btheta}(\bx) d\bx= Z_{\btheta}^{-1} \left(\btheta_{\star}+\varepsilon \beta(\btheta)-\btheta\right)=0, 
\end{equation}
where $\btheta_{\star}=(\int_{\MX_1}\pi(\bx)d\bx, \int_{\MX_2}\pi(\bx)d\bx, \ldots, \int_{\MX_m}\pi(\bx)d\bx)$, $Z_{\btheta}$ is the normalizing constant, $\beta(\btheta)$ is a perturbation term, $\varepsilon$ is a small error depending on $\epsilon, n$ and $m$. 
The mean-field equation implies that for any 
$\zeta>0$, $\btheta_k$ converges to a small neighbourhood of $\btheta_{\star}$. By applying perturbation theory and setting the Lyapunov function  $\mathbb{V}(\btheta)=\frac{1}{2}\|\btheta_{\star}-\btheta\|^2$, we can establish the stability condition:  
\begin{lemma}[Stability, informal version of Lemma \ref{convex_appendix__}] 
\label{convex_main}
Given a small enough $\epsilon$ (learning rate), a large enough $n$ (batch size) and $m$ (partition number), there is a constant $\phi=\inf_{\btheta} Z_{\btheta}^{-1}>0$ such that the mean-field $h(\btheta)$ satisfies $$\forall \btheta \in \bTheta, \langle h(\btheta), \btheta - \btheta_{\star}\rangle \leq  -\phi\|\btheta - \btheta_{\star}\|^2+\mathcal{O}\left(\epsilon+\frac{1}{m}+\sup_{\bx}\Var(\xi_n(\bx))\right),$$
where $\Var(\xi_n(\bx))$ denotes the variance of the noise of the stochastic energy estimator $\xi_n(\cdot)$ of batch size $n$ and the variance decays to $0$ as $n\rightarrow N$.
\end{lemma}

Together with the tool of Poisson equation  \citep{Albert90, andrieu05}, which controls the fluctuation of $\widetilde H(\btheta, \bx)-h(\btheta)$, we can 
establish convergence of $\btheta_k$ in Theorem \ref{thm:1}, whose proof is given in the supplementary 
material. 
 
 %, where $Z_{\btheta}=\sum_{i=1}^m \frac{\int_{\bchi_i \pi(\bx)d\bx}}{\Psi(u_i)^{\zeta}}$. 
\begin{theorem}[$L^2$ convergence rate, informal version of Theorem \ref{latent_convergence}]
\label{thm:1}
Given standard smoothness and dissipativity assumptions, a small enough learning rate $\epsilon_k$, a large partition number $m$ and a large batch size $n$, $\btheta_k$ converges to $\btheta_{\star}$ such that
 \begin{equation*}
    \E\left[\|\bm{\theta}_{k}-\bm{\theta}_{\star}\|^2\right]=\mathcal{O}\left( \omega_{k}+\sup_{i\geq k_0}\epsilon_i+\frac{1}{m} +\sup_{\bx}\Var(\xi_n(\bx))\right),
\end{equation*}
where $k_0$ is some large enough integer and $\btheta_{\star}=(\int_{\MX_1}\pi(\bx)d\bx, \int_{\MX_2}\pi(\bx)d\bx, \ldots, \int_{\MX_m}\pi(\bx)d\bx)$. 
\end{theorem}

\subsection{Ergodicity and dynamic importance sampler}  \label{FMalg2}

CSGLD belongs to the class of adaptive MCMC algorithms,
but its transition kernel is based on SGLD instead of the Metropolis algorithm. As such, the ergodicity theory for traditional adaptive MCMC algorithms \citep{RobertsRosenthal2007, AndrieuMoulines2006, Fortetal2011, Liang10} is not directly applicable. To tackle this issue, we conduct the following theoretical study. First, rewrite (\ref{SGLDeq6}) as 
  \begin{equation} \label{SGLDeq8}
     \bx_k- \epsilon\left( \nabla_{\bx} 
     \widehat{L}(\bx_k,\btheta_{\star}) + 
     \Upsilon(\bx_k,\btheta_k,\btheta_{\star})\right)+\mathcal{N}({0, 2\epsilon \tau\bm{I}}),
 \end{equation}
where $\nabla_{\bx} \widehat{L}(\bx_k,\btheta_{\star})= \frac{N}{n} \left[1+ 
   \frac{\zeta\tau}{\Delta u}  \left(\log \theta_{\star}({J}(\bx_k))-\log\theta_{\star}(({J}(\bx_k)-1)\vee 1) \right) \right]  
    \nabla_{\bx} \widetilde U(\bx_k)$, the bias term 
    $\Upsilon(\bx_k,\btheta_k,\btheta_{\star})=
    \nabla_{\bx} \widetilde{L}(\bx_k,\btheta_k)-
    \nabla_{\bx} \widehat{L}(\bx_k,\btheta_{\star})$,
     and 
 $\nabla_{\bx} \widetilde{L}(\bx_k,\btheta_{k})= \frac{N}{n} \left[1+ 
   \frac{\zeta\tau}{\Delta u}  \left(\log \theta_{k}(J_{\widetilde U}(\bx_k))-\log\theta_{k}((J_{\widetilde U}(\bx_k)-1)\vee 1) \right) \right]  
    \nabla_{\bx} \widetilde U(\bx_k)$.  The order of the bias is figured out in Lemma C1 in the supplementary material based on the results of Theorem \ref{thm:1}.
    
Next, we show how the empirical mean $\frac{1}{k}\sum_{i=1}^k f(\bx_i)$ deviates from the posterior mean $\int_{\MX}f(\bx)\varpi_{\Psi_{\btheta_{\star}}}(\bx)d\bx$. Note that this is a direct application of Theorem 2 of \citet{Chen15} by treating $\nabla_{\bx} \widehat{L}(\bx,\btheta_{\star})$ as the stochastic gradient of a target distribution and 
 $\Upsilon(\bx,\btheta,\btheta_{\star})$ as the bias of the stochastic gradient. Moreover, considering that $\varpi_{\widetilde \Psi_{\btheta_{\star}}}(\bx)\propto
\frac{\pi(\bx)}{\theta_{\star}^{\zeta}(J(\bx))}\rightarrow\varpi_{ \Psi_{\btheta_{\star}}}$ as $m\rightarrow \infty$ based on Lemma B4 in the supplementary material, we have the following
% , for which  based on a fixed learning rate $\epsilon$ and only an extra bias term in the order of $\frac{1}{k}\sum_{i=1}^k \|\E[\Upsilon(\bx_i,\btheta_i)]\|$ is injected. 
 %Combining the $L^2$ convergence of $\btheta_k$, the smoothness condition and the approximation of $\varpi_{\Psi_{\btheta_{\star}}}$ by $\varpi_{\btheta_{\star}}$, we are able to control the extra bias term and have the first ergodicity result for the averaging estimators.

\begin{lemma}[Convergence of the Averaging Estimators, informal version of Lemma \ref{avg_converge_appendix}]
\label{avg_converge}
Suppose the smoothness, dissipativity and other mild assumptions hold. For any bounded function $f$, we have
\begin{equation*}
\small
\begin{split}
    \left|\E\left[\frac{\sum_{i=1}^k f(\bx_i)}{k}\right]-\int_{\bchi} f(\bx)\varpi_{\widetilde \Psi_{\btheta_{\star}}}(d\bx)\right|&= \mathcal{O}\left(\frac{1}{k\epsilon}+\sqrt{\epsilon}+\sqrt{\frac{\sum_{i=1}^k \omega_k}{k}}+\frac{1}{\sqrt{m}}+\sup_{\bx}\sqrt{\Var(\xi_n(\bx))}\right), \\
\end{split}
\end{equation*}
where $\varpi_{\widetilde \Psi_{\btheta_{\star}}}(\bx)= \frac{1}{Z_{\btheta_{\star}}} 
\frac{\pi(\bx)}{\theta_{\star}^{\zeta}(J(\bx))}$ and $Z_{\btheta_{\star}}=\sum_{i=1}^m \frac{\int_{\MX_i} \pi(\bx)d\bx}{\theta_{\star}(i)^{\zeta}}$.
\end{lemma}

Finally, we consider the problem of estimating the quantity  $\int_{\MX} f(\bx) \pi(\bx) d\bx$. Recall that $\pi(\bx)$ is the target distribution that we would like to make inference for. To estimate this quantity, we naturally 
 consider the weighted averaging estimator $\frac{\sum_{i=1}^k\theta_{i}^{\zeta}( J_{\widetilde U}(\bx_i)) f(\bx_i)}{ 
\sum_{i=1}^k\theta_{i}^{\zeta}( J_{\widetilde U}(\bx_i))}$ by treating $\theta^{\zeta}(J_{\widetilde U}(\bx_i))$ as
the dynamic importance weight of the sample $\bx_i$ 
for $i=1,2,\ldots,k$. The convergence of this
estimator is established in Theorem \ref{wavg_esti}, which can be proved by repeated applying 
 Theorem \ref{thm:1} and Lemma \ref{avg_converge} 
 with the details given in the supplementary material.
  
\begin{theorem}[Convergence of the Weighted Averaging Estimators, informal version of Theorem \ref{w_avg_converge_appendix}]
\label{wavg_esti} Given the smoothness, dissipativity and other mild assumptions, for any bounded function $f$, we have
\label{w_avg_converge}
\begin{equation*}
\footnotesize
\begin{split}
    \left|\E\left[\frac{\sum_{i=1}^k\theta_{i}^{\zeta}
    (J_{\widetilde U}(\bx_i)) f(\bx_i)}{\sum_{i=1}^k \theta_{i}^{\zeta} ( J_{\widetilde U}(\bx_i))}\right]-\int_{\bchi} f(\bx)\pi(d\bx)\right|&= \mathcal{O}\left(\frac{1}{k\epsilon}+\sqrt{\epsilon}+\sqrt{\frac{\sum_{i=1}^k \omega_k}{k}}+\frac{1}{\sqrt{m}}+\sup_{\bx}\sqrt{\Var(\xi_n(\bx))}\right).\\
\end{split}
\end{equation*}
\end{theorem}
The bias of the weighted averaging estimator decreases 
if one applies a larger batch size, a finer sample space partition, a smaller learning rate $\epsilon$, and smaller step sizes $\{\omega_k\}_{k\geq 0}$. Admittedly, the
order of this bias is slightly larger than   $\mathcal{O}\left(\frac{1}{k\epsilon}+\epsilon\right)$
 achieved by the standard SGLD. We note that this is necessary as simulating from the flattened distribution $\varpi_{\Psi_{\btheta_{\star}}}$ often leads to a much faster convergence, see e.g. the green curve v.s. the purple curve in Figure \ref{fig: 4c}.

\paragraph{Discussions on more scalable updates} The stochastic approximation update (\ref{sa_framework})(ii) yields a global stability property but may not be scalable enough in some big data problems. For more scalable updates, one can adopt an elegant stochastic approximation update proposed in \citet{icsgld} 
\begin{equation}
\label{novel_SA_scheme}
    \theta_{k+1}(i)={\theta}_{k}(i)+\omega_{k+1}{\theta}_{k}( J(\bx_{k+1}))\left(1_{i= J(\bx_{k+1})}-{\theta}_{k}(i)\right),
\end{equation}
where $\theta(i)$ converges to a smoother estimate of ${\left(\int_{\MX_i}\pi(\bx)d\bx)\right)}^{\frac{1}{\zeta}}$ instead of the original target $\int_{\MX_i}\pi(\bx)d\bx$. Since high energy region often yields exponentially decreasing probability mass, the exponent $\frac{1}{\zeta}$ given $\zeta\gg 1$ greatly facilitates the estimation tasks.

\section{Numerical studies}

% \subsection{A Gaussian mixture distribution}
\subsection{Simulations of multi-modal distributions}

\paragraph{A Gaussian mixture distribution} The first numerical study is to test the performance of CSGLD on a Gaussian mixture distribution $\pi(\bx)=0.4 N(-6,1)+0.6 N(4,1)$. In each experiment, the algorithm was  run for $10^7$ iterations. We fix the temperature $\tau=1$ and the learning rate $\epsilon=0.1$. The step size for stochastic approximation follows $\omega_k=\frac{1}{k^{0.6}+100}$. The sample space is partitioned into 50 subregions with $\Delta u=1$. The stochastic gradients are simulated by injecting additional random noises following $N(0,0.01)$ to the exact gradients. For comparison, SGLD is chosen as the baseline algorithm 
 and implemented with the same setup as CSGLD. We repeat the experiments 10 times and report the average and the  
associated standard deviation. 

We first assume that $\btheta_{\star}$ is known and plot the energy functions for both $\pi(\bx)$ and $\varpi_{\Psi_{\btheta_{\star}}}$ with different values
of $\zeta$.  Figure \ref{fig: 4a} shows that the original energy function has a rather large energy barrier which strongly affects the communication between two modes of the distribution. In contrast, CSGLD samples from a modified energy function, which yields a flattened landscape and reduced energy barriers. For example, with $\zeta=0.75$, the energy barrier for this example is {\it greatly reduced from 12 to as small as 2}. Consequently, the local trap problem can be greatly alleviated. Regarding the bizarre 
peaks around $x=4$, we leave the study in the supplementary material.

 \begin{figure*}[htbp]
  \vspace{-0.03in}
    \centering
    \subfigure[Original v.s. trial energies]{
    \begin{minipage}[t]{0.3\linewidth}
    \centering
    \label{fig: 4a}
    \includegraphics[scale=0.24]{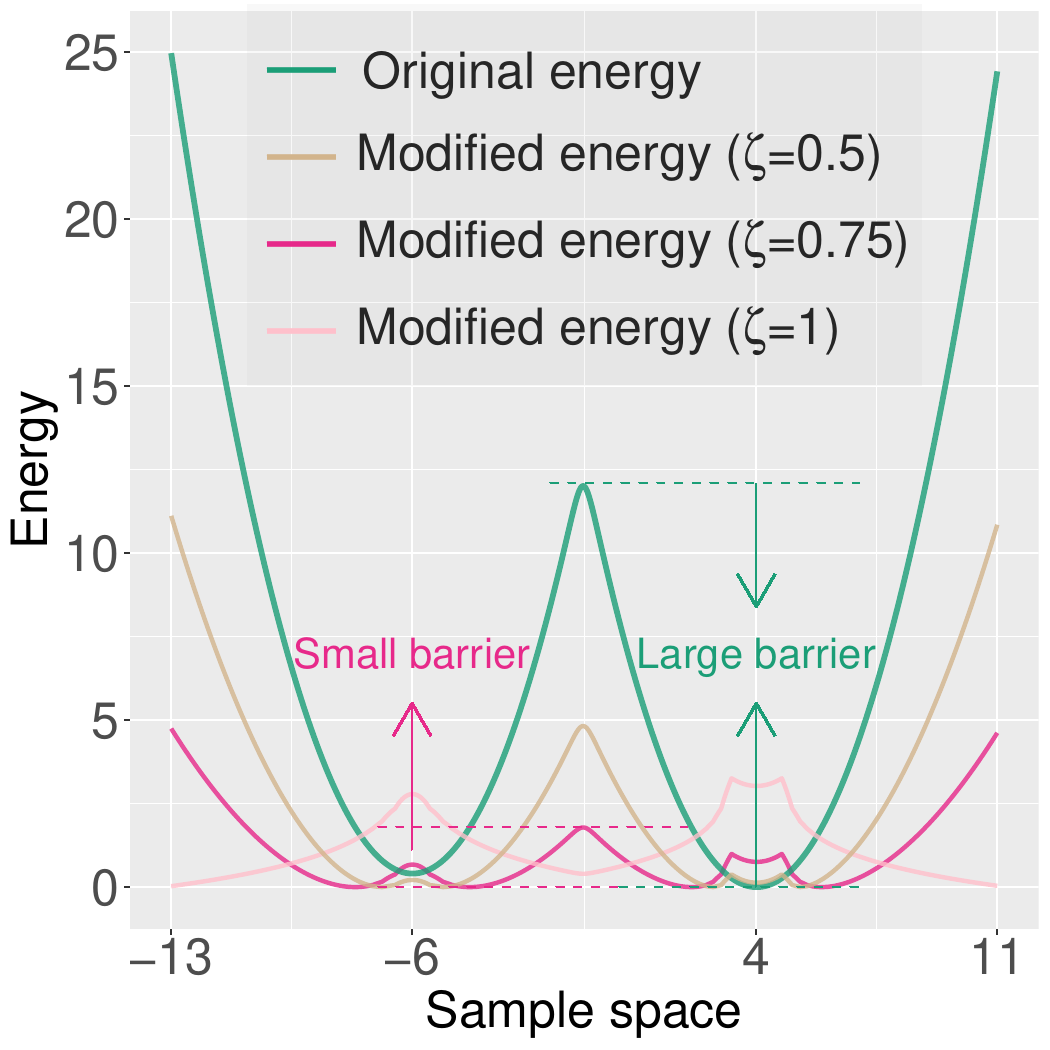}
    \end{minipage}%
    }%
    \subfigure[$\btheta$'s estimates and histograms]{
    \begin{minipage}[t]{0.3\linewidth}
    \centering
    \label{fig: 4b}
    \includegraphics[scale=0.24]{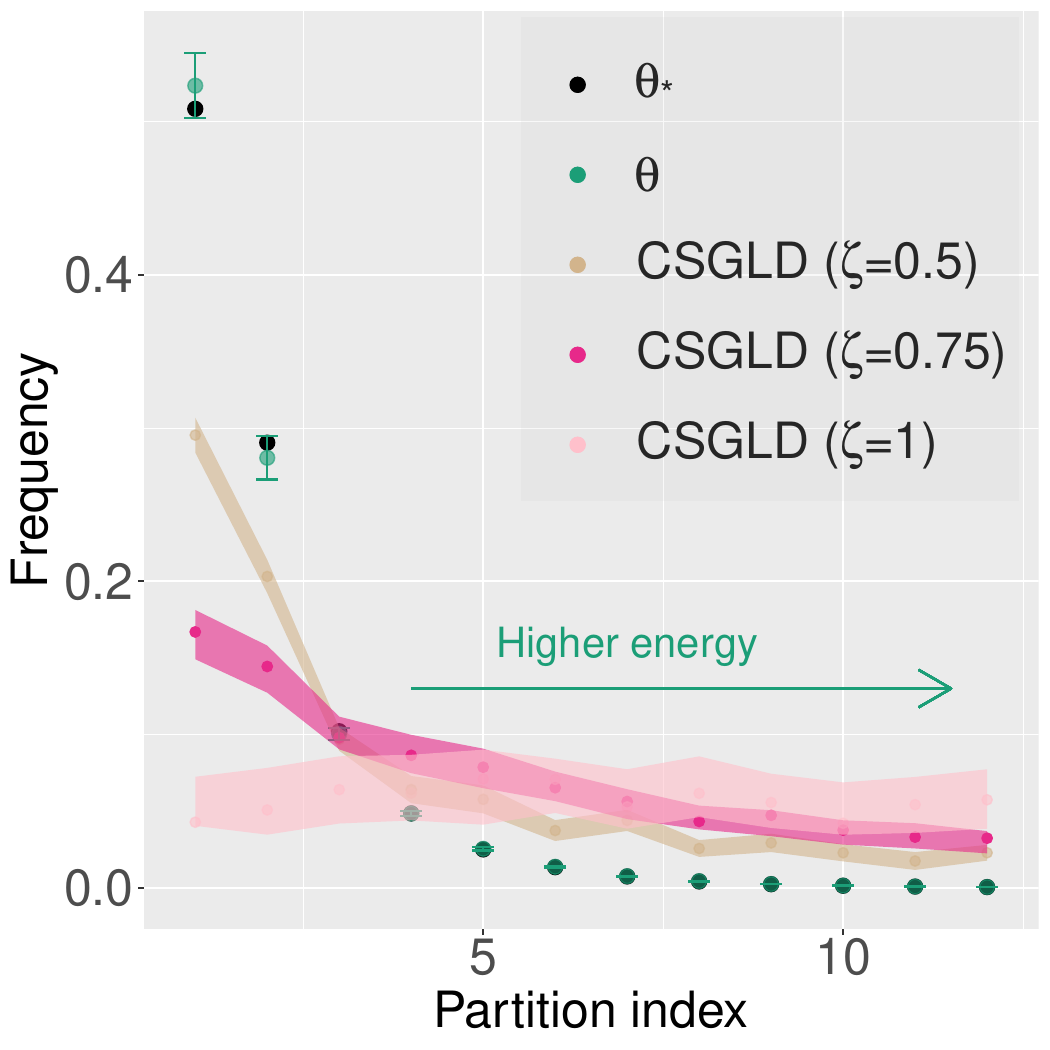}
    \end{minipage}%
    }%
    \subfigure[Estimation errors]{
    \begin{minipage}[t]{0.3\linewidth}
    \centering
    \label{fig: 4c}
    \includegraphics[scale=0.24]{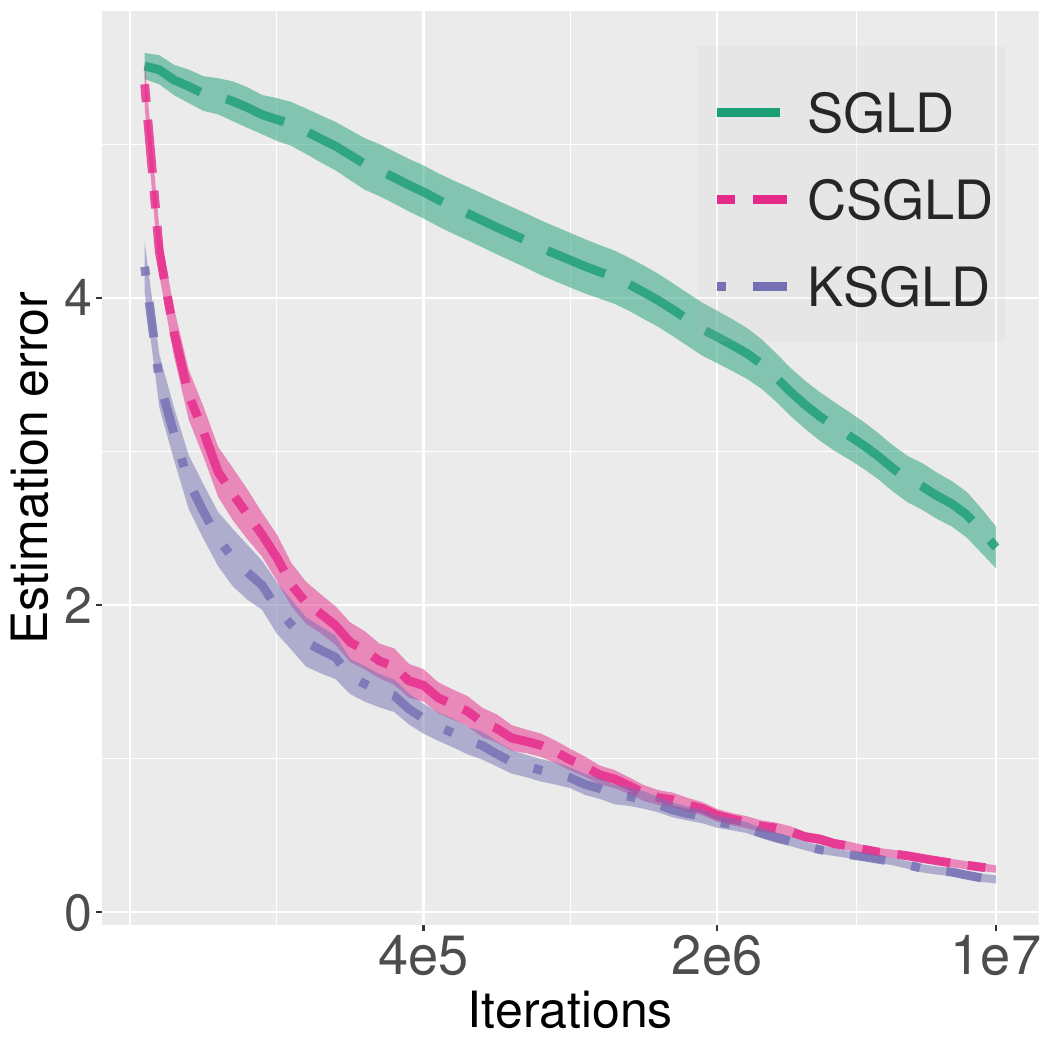}
    \end{minipage}%
    }%
  \vspace{-0.06in}
  \caption{Comparison between SGLD and CSGLD: figure (a) presents only the first 12 partitions for an illustrative purpose; KSGLD in  figure (c) is implemented by assuming $\btheta_{\star}$ is known.}
  \label{CSGLD_stats}
  \vspace{-0.17in}
\end{figure*}

%Since the exact value of $\btheta_{\star}$ is unknown, we propose to adaptively estimate it through stochastic approximation. 

Figure  \ref{fig: 4b} summarizes the estimates of $\btheta_{\star}$ with $\zeta=0.75$, which matches the ground truth value of $\btheta_{\star}$ very well. Notably, we see that $\theta_{\star}(i)$ decays exponentially fast as the partition index $i$ increases, which indicates the exponentially decreasing probability of visiting high energy regions and a severe local trap problem. CSGLD tackles this issue by adaptively updating the transition kernel or, equivalently, the invariant distribution such 
that the sampler moves like a ``random walk'' in the space of energy.  In particular, setting $\zeta=1$ leads to a flat histogram of energy (for the samples produced by CSGLD).

To explore the performance of CSGLD in quantity estimation with the weighed averaging estimator, 
we compare CSGLD ($\zeta=0.75$) with SGLD and KSGLD in estimating \textcolor{black}{the posterior mean $\int_{\MX} \bx \pi(\bx)d\bx$}, where KSGLD was implemented by assuming $\btheta_{\star}$ is known and sampling from $\varpi_{\Psi_{\btheta_{\star}}}$ directly. Each algorithm was run for 10 times, and we recorded the mean absolute 
estimation error along with iterations. As shown in Figure \ref{fig: 4c}, the estimation error of SGLD decays quite slow and rarely converges due to the high energy barrier. On the contrary, KSGLD converges much faster, which shows the advantage of sampling from a flattened distribution $\varpi_{\Psi_{\btheta_{\star}}}$. Admittedly, $\btheta_{\star}$ is unknown in practice. CSGLD instead adaptively updates its invariant distribution while optimizing the parameter ${\btheta}$ until \emph{an optimization-sampling equilibrium} is reached. In the early period of the run, CSGLD converges slightly slower than KSGLD, but soon it becomes as efficient as KSGLD.

Finally, we compare the sample path and learning rate for CSGLD and SGLD. As shown in Figure \ref{fig: 3a}, SGLD tends to be trapped in a deep local optimum for an exponentially long time. CSGLD, in contrast, possesses a {\it self-adjusting mechanism} for escaping from local traps. In the early period of a run, CSGLD might suffer 
from a similar local-trap problem as SGLD (see Figure \ref{fig: 3b}). In this case, the components of $\btheta$ corresponding to the current subregion will  increase very fast, eventually rendering \textcolor{black}{a smaller or even negative
gradient multiplier} which \emph{bounces the sampler back to high energy regions}. To illustrate the process, we plot a bouncy zone and an absorbing zone
%, whose widths depend on the partition 
%on the energy landscape,  %where the bouncy zone has a small enough or even negative learning rate and the absorbing zone has a similar learning rate before. 
in Figure \ref{fig: 3c}. The bouncy zone enables the sampler to ``jump'' over large energy barriers to explore other modes. As the run continues, $\btheta_k$ converges to $\btheta_{\star}$. Figure \ref{fig: 3d} shows that larger bouncy ``jumps'' (in red lines) can potentially be induced in the bouncy zone, which occurs in both local and global optima. Due to the {\it self-adjusting mechanism}, CSGLD has the local trap problem much alleviated.

\begin{figure*}[!ht]
\vspace{-1em}
  \centering
  \subfigure[SGLD paths]{\label{fig: 3a}\includegraphics[scale=0.18]{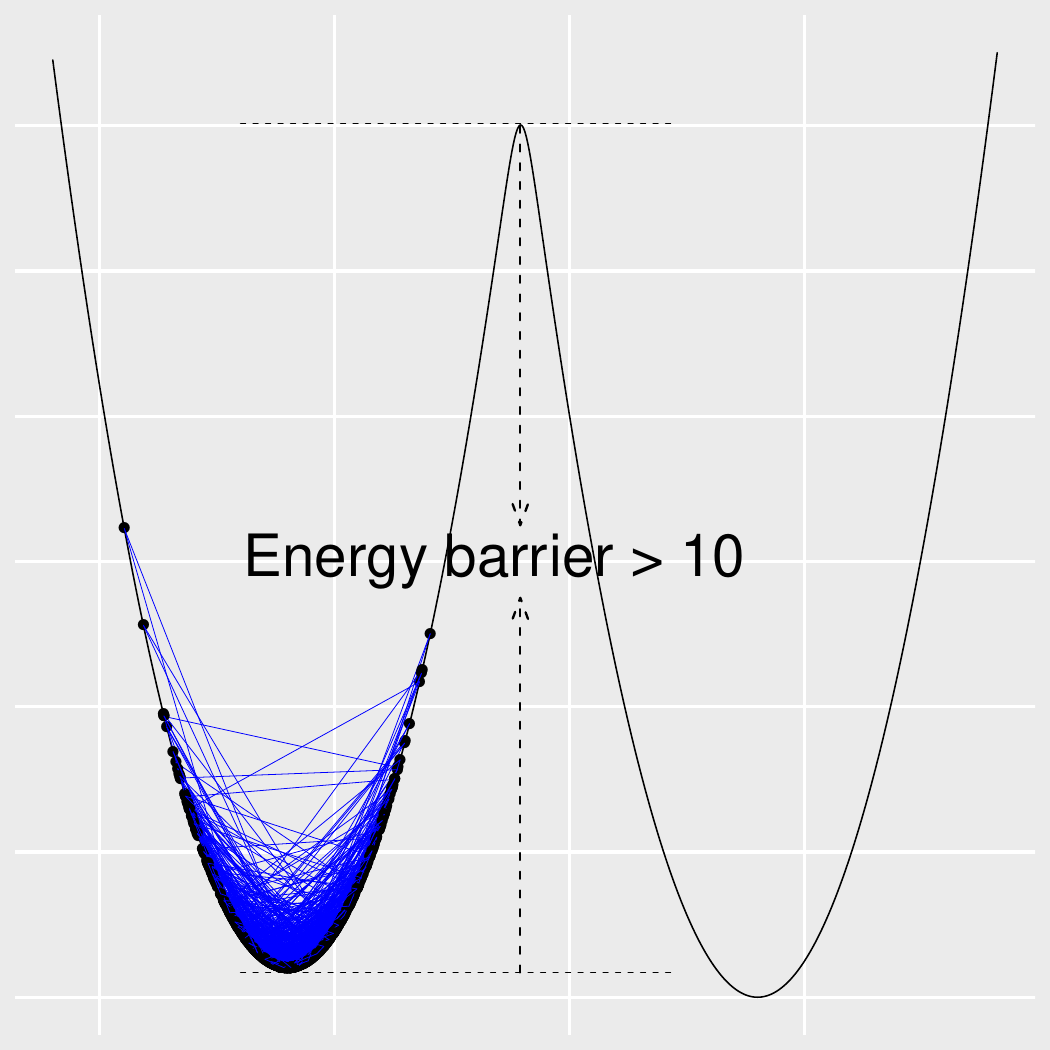}}\quad
  \subfigure[CSGLD paths (early) ]{\label{fig: 3b}\includegraphics[scale=0.18]{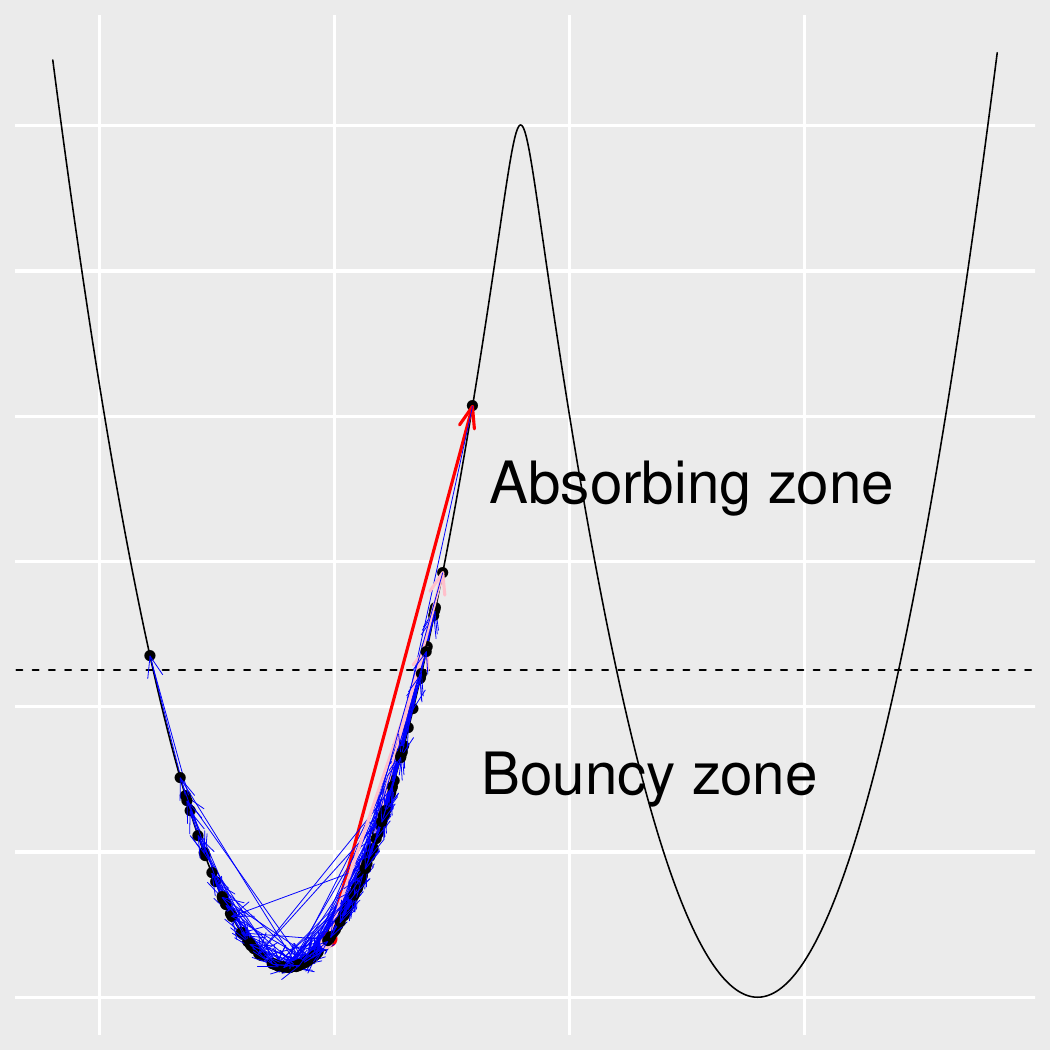}}\quad
  \subfigure[CSGLD paths (mid)]{\label{fig: 3c}\includegraphics[scale=0.18]{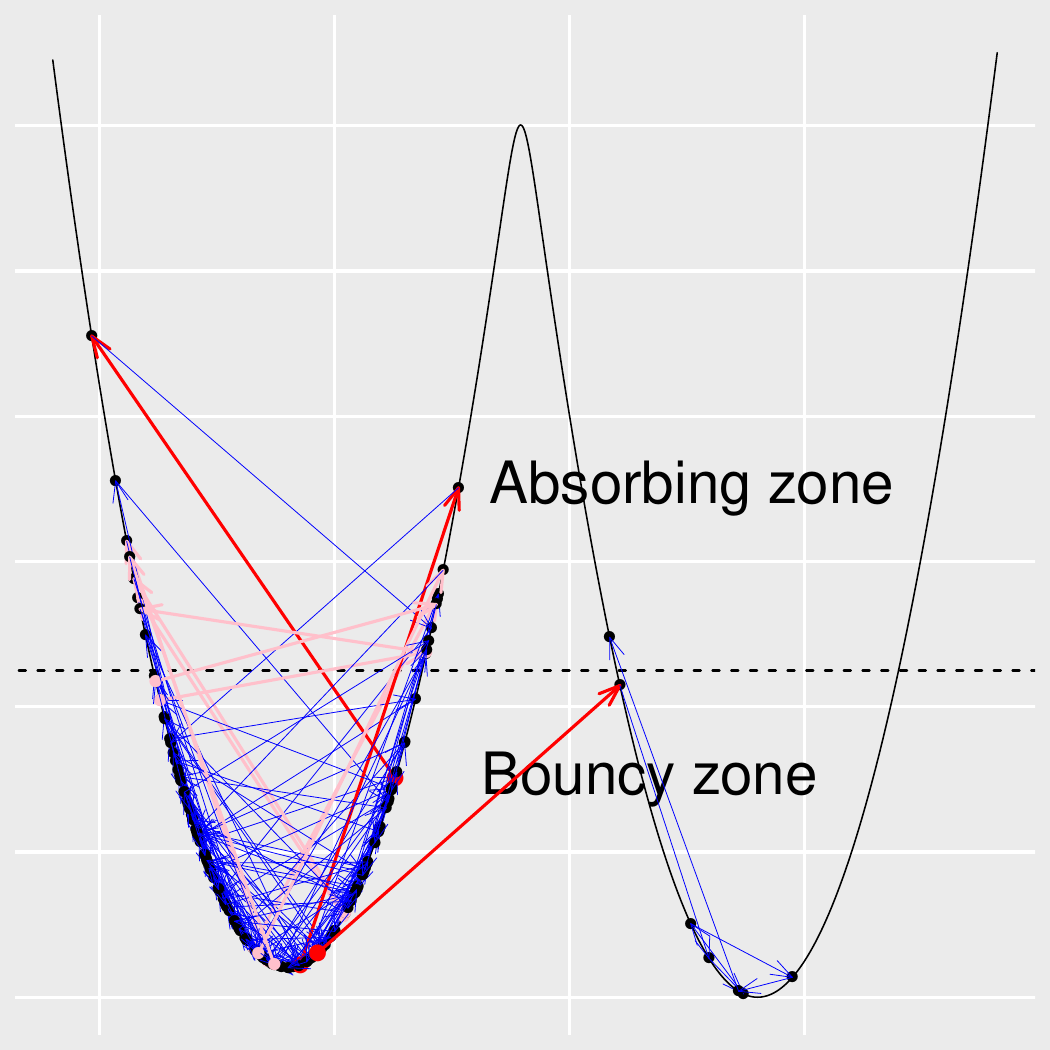}}\quad
  \subfigure[CSGLD paths (late)]{\label{fig: 3d}\includegraphics[scale=0.18]{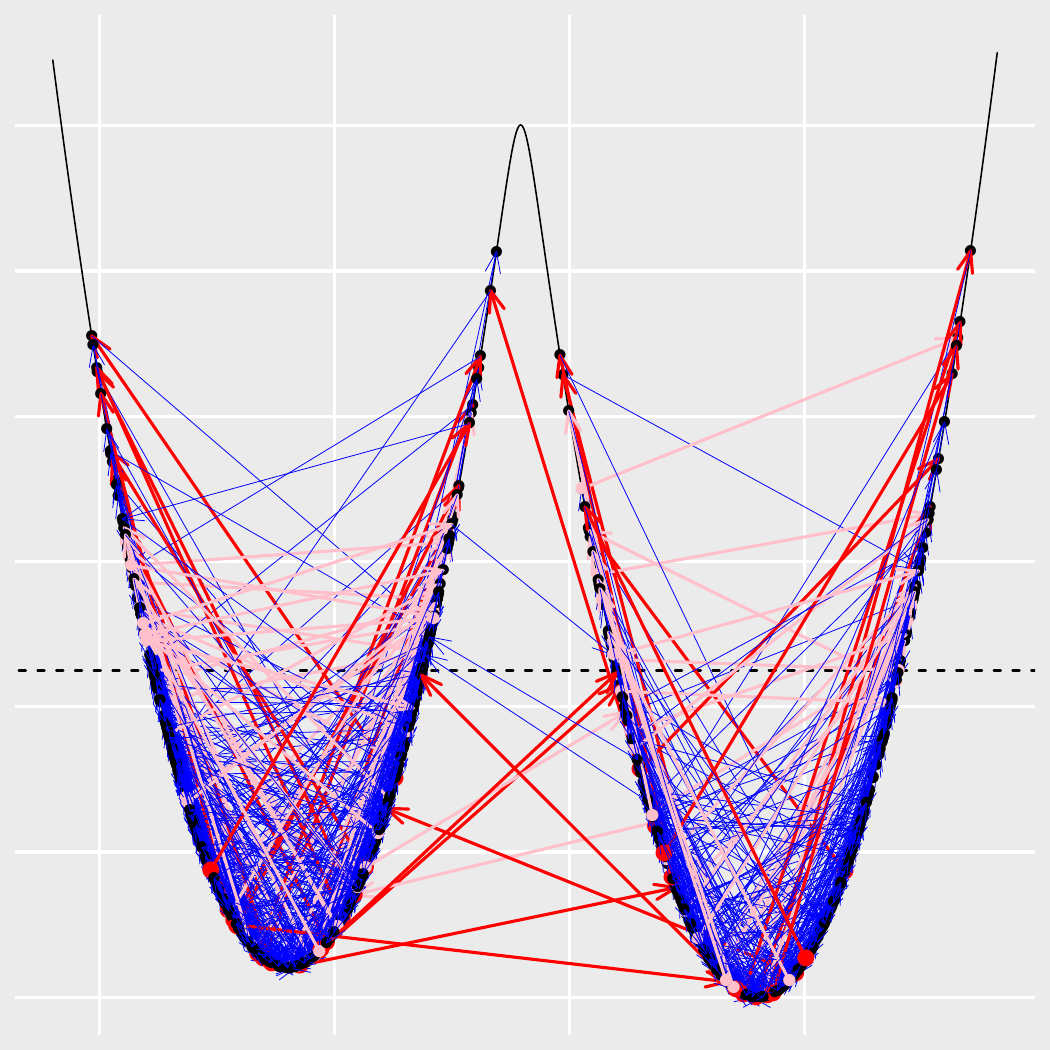}}
  \vspace{-0.6em}
  \caption{Sample trajectories of SGLD and CSGLD: figures (a) and (c) are implemented by 100,000 iterations with a thinning factor 100 and $\zeta=0.75$, while figure (b) utilizes a thinning factor 10.}
  \label{trajectory}
  \vspace{-1.1em}
\end{figure*}

\paragraph{A synthetic multi-modal distribution} We next simulate from a distribution $\pi(\bx)\propto e^{-U(\bx)}$, where $U(\bx)=\sum_{i=1}^2 \frac{x(i)^2-10\cos(1.2\pi x(i))}{3}$ and $\bx=(x(1), x(2))$. We compare CSGLD with SGLD, replica exchange SGLD (reSGLD) \citep{deng2020}, and SGLD with cyclic learning rates (cycSGLD) \citep{ruqi2020} and detail the setups in the supplementary material. Figure \ref{fig: msa} shows that the distribution contains nine important modes, where the center mode has the largest probability mass and the four modes on the corners have the smallest mass. We see in Figure \ref{fig: msb} that SGLD spends too much time in local regions and only identifies three modes. cycSGLD has a better ability to explore the distribution by leveraging large learning rates cyclically. However, as illustrated in Figure \ref{fig: msc}, such a mechanism is still not efficient enough to resolve the local trap issue for this problem. reSGLD proposes to include a high-temperature process to encourage exploration and allows interactions between the two processes via appropriate swaps. We observe in Figure \ref{fig: msd} that reSGLD obtains both the exploration and exploitation abilities and yields a much better result. However, the noisy energy estimator may hinder the swapping efficiency and it becomes difficult to estimate a few modes on the corners. As to our algorithm, CSGLD first simulates the importance samples and recovers the original distribution according to the importance weights. We notice that the samples from CSGLD can traverse freely in the parameter space and eventually achieve a remarkable performance, as shown in Figure \ref{fig: mse}.

\begin{figure*}[!ht]
\vspace{-0.5em}
  \centering
  \subfigure[Ground truth]{\label{fig: msa}\includegraphics[scale=0.24]{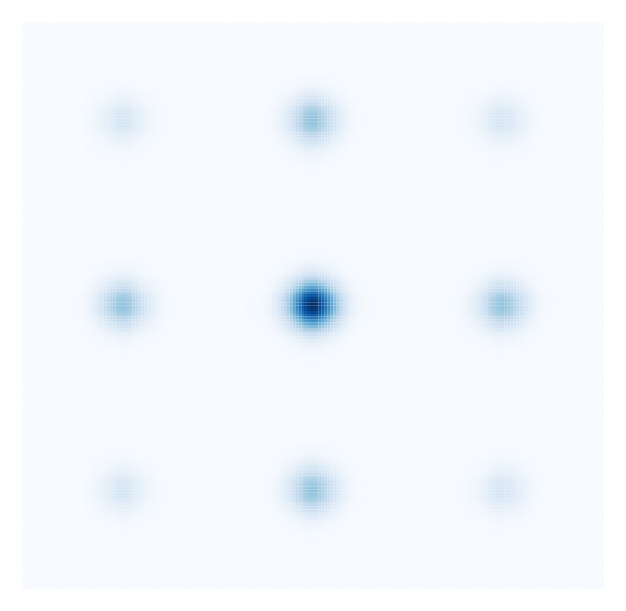}}\;\;
  \subfigure[SGLD]{\label{fig: msb}\includegraphics[scale=0.24]{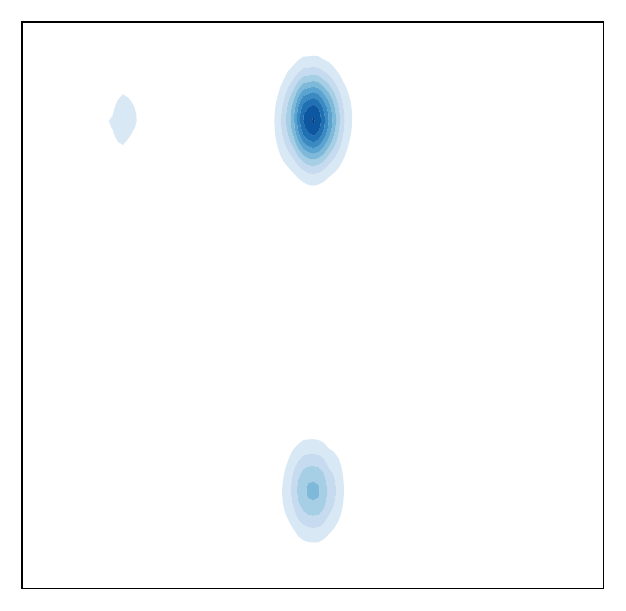}}\;\;
  \subfigure[cycSGLD]{\label{fig: msc}\includegraphics[scale=0.24]{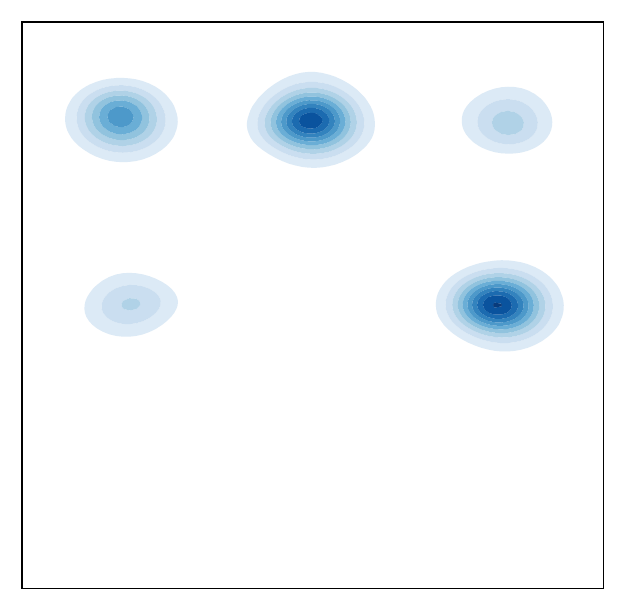}}\;\;
  \subfigure[reSGLD]{\label{fig: msd}\includegraphics[scale=0.24]{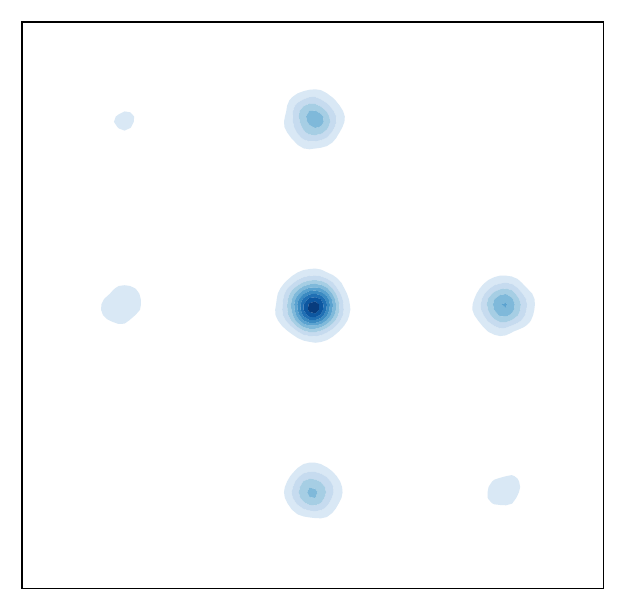}}\;\;
  \subfigure[CSGLD]{\label{fig: mse}\includegraphics[scale=0.24]{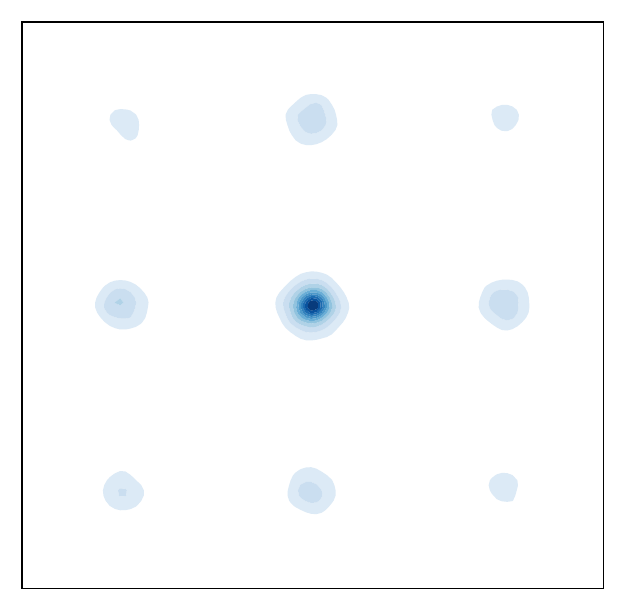}}
  \vspace{-0.09in}
  \caption{Simulations of a multi-modal distribution. A resampling scheme is used for CSGLD.}
  \label{multi-modal_simulations}
  \vspace{-1em}
\end{figure*}

\subsection{UCI data}
\label{reg_UCI}

We tested the performance of CSGLD on the \textbf{UCI} regression datasets. For each dataset, we normalized all features and randomly selected 10\% of the observations for testing. Following \citep{Jose_adam_15}, we modeled the data using a Multi-Layer Perception (MLP) with a single hidden layer of 50 hidden units. We set the mini-batch size $n=50$ and trained the model for 5,000 epochs. The learning rate was set to 5e-6 and the default $L_2$-regularization 
coefficient is 1e-4. For all the datasets, we used the stochastic energy $\frac{N}{n}\widetilde U(\bx)$ to evaluate the partition index. We set the energy bandwidth $\Delta u=100$. We fine-tuned the 
temperature $\tau$ and the hyperparameter $\zeta$. For a fair comparison, each algorithm was run 10 times with fixed seeds for each dataset. 
In each run, the performance of the algorithm was evaluated by averaging over 50 models, where the averaging estimator was used for SGD and SGLD and the weighted averaging estimator was used for CSGLD. As shown in Table~\ref{UCI}, SGLD outperforms the stochastic gradient descent (SGD) 
algorithm for most datasets 
due to the advantage of a sampling algorithm in obtaining more informative modes. Since all these datasets are small, there is only very limited potential for improvement. Nevertheless, CSGLD still consistently outperforms all the baselines including SGD and SGLD. 

The contour strategy proposed in the paper can be naturally extended to SGHMC \citep{Chen14, yian2015} without affecting the theoretical results. In what follows, we adopted a numerical method proposed by \citet{Saatci17} to avoid extra hyperparameter tuning. We set the momentum term to 0.9 and simply inherited all the other parameter settings 
used in the above experiments. In such a case, we compare the contour SGHMC (CSGHMC) with the baselines, including M-SGD (Momentum SGD) and SGHMC. The comparison indicates that some improvements can be achieved by including the momentum.

\begin{table}[!htb]
  \centering
  \vspace{-0.3em}
%   \small
%   \begin{tabular}{0.8\textwidth}{c|ccccccc}
\caption{Algorithm evaluation using average root-mean-square error and its standard deviation.}
  \begin{tabular}{c|ccccccc}
    \toprule
    % \multicolumn{4}{c}{Criteo Dataset}                   \\
    % \midrule
    Dataset &   Energy &  Concrete &   Yacht  & Wine  \\
    Hyperparameters ($\tau/\zeta$) &   1/1  & 5/1  & 1/2.5 & 5/10 \\
    \midrule
     SGD  & 1.13$\pm$0.07  & 4.60$\pm$0.14 & 0.81$\pm$0.08 &  0.65$\pm$0.01 \\
    SGLD  & 1.08$\pm$0.07  & 4.12$\pm$0.10 & 0.72$\pm$0.07 & 0.63$\pm$0.01 \\
     CSGLD  & \textbf{1.02$\pm$0.06} & \textbf{3.98$\pm$0.11} & \textbf{0.69$\pm$0.06} &  \textbf{0.62$\pm$0.01} \\
      \midrule
      M-SGD & 0.95$\pm$0.07 & 4.32$\pm$0.27 & 0.73$\pm$0.08 &  0.71$\pm$0.02 \\
      SGHMC & 0.77$\pm$0.06 & 4.25$\pm$0.19 & \textbf{0.66$\pm$0.07} &  0.67$\pm$0.02 \\
     CSGHMC & \textbf{0.76$\pm$0.06} & \textbf{4.15$\pm$0.20} & 0.72$\pm$0.09 & \textbf{0.65$\pm$0.01} \\
    \bottomrule
  \end{tabular}
  \label{UCI}
  \vspace{-0.2in}
\end{table}

\subsection{Computer vision data} \label{cv_cifar}
\vspace{-0.05in}
This section compares only CSGHMC with M-SGD and SGHMC due to the popularity of momentum in accelerating computation for computer vision datasets. We keep partitioning the sample space according to the stochastic energy $\frac{N}{n} \widetilde U(\bx)$, where a mini-batch data of size $n$ is randomly chosen from the full dataset of size $N$ at each iteration. Notably, such a strategy significantly accelerates the computation of CSGHMC. As a result, CSGHMC has almost the same computational cost as SGHMC and SGD. To reduce the bias associated with the stochastic energy, we choose a large batch size $n=1,000$. For more discussions on the hyperparameter settings, we refer readers to section \ref{more_scalable_target} in the supplementary material.

{\bf CIFAR10} is a standard computer vision dataset with 10 classes and 60,000 images, for which 50,000 images were used for training and the rest for testing. We modeled the data using a Resnet of 20 layers (Resnet20) \citep{kaiming15}. In particular, for CSGHMC, we considered a partition of the energy space in 200 subregions, where the energy bandwidth was set to $\Delta {u}=1000$.  We trained the model for a total of 1000 epochs and evaluated the model every ten epochs based on two criteria, namely, best point estimate (BPE) and Bayesian model average (BMA). We repeated each experiment 10 times and reported in Table \ref{full_cifar} the average prediction accuracy and the standard deviation. 
  
%%% new
In the first set of experiments, all the algorithms utilized a fixed learning rate $\epsilon=2e-7$ and a fixed temperature $\tau=0.01$ under the Bayesian setting.%\footnote[3]{Setting $\epsilon=2e-7$ is equivalent to fixing $\epsilon=0.01$ using the default average cross entropy loss.}. 
SGHMC performs quite similarly to M-SGD, both obtaining around 90\% accuracy in BPE and 92\% in BMA. Notably, in this case, simulated annealing is not applied to any of the algorithms and achieving the state-of-the-art is quite difficult. However, BMA still consistently outperforms  BPE, implying the great potential of advanced MCMC techniques in deep learning. Instead of simulating from $\pi(\bx)$ directly, CSGHMC adaptively simulates from a flattened distribution $\varpi_{\btheta_{\star}}$ and adjusts the sampling bias by dynamic importance weights. As a result, the weighted averaging estimators obtain an improvement by as large as 0.8\% on BMA. In addition, the flattened distribution facilitates optimization and the increase in BPE is quite significant. 

In the second set of experiments, we employed a decaying schedule on both learning rates and temperatures (if applicable) to obtain simulated annealing effects. For the learning rate, we fix it at $2\times 10^{-6}$ in the first 400 epochs and then decayed it by a factor of $1.01$ at each epoch. For the temperature, we consistently decayed it by a factor of $1.01$ at each epoch. We call the resulting algorithms by saM-SGD, saSGHMC, and saCSGHMC, respectively. Table \ref{full_cifar} shows that the performances of all 
algorithms are increased quite significantly, where the fine-tuned baselines already obtained the state-of-the-art results. Nevertheless, saCSGHMC further improves BPE by 0.25\% and slightly improve the highly optimized BMA by nearly 0.1\%.

{\bf CIFAR100} dataset has 100 classes, each of which contains 500 training images and 100 testing images. We follow a similar setup as CIFAR10, except that $\Delta u$ is set to 5000. For M-SGD, BMA can be better than BPE by as large as 5.6\%. CSGHMC has led to an improvement of 3.5\% on BPE and 2\% on BMA, which further demonstrates the superiority of advanced MCMC techniques. Table \ref{full_cifar} also shows that with the help of both simulated annealing and importance sampling, saCSGHMC can outperform the highly optimized baselines by almost 1\% accuracy on BPE and 0.7\% on BMA. The significant improvements show the advantage of the proposed method in training DNNs.

\begin{table}[htbp]
% \vspace{-0.4em}
\begin{center}
\caption{Experiments on CIFAR10 \& 100 using Resnet20, where BPE and BMA are short for best point estimate and Bayesian model average, respectively. %We repeat each experiment 10 times and emphasize significant improvements with bold colors.
}\label{full_cifar}
\begin{tabular}{ccccc}
\hline
\multirow{2}{4em}{Algorithms} & \multicolumn{2}{c}{CIFAR10} & \multicolumn{2}{c}{CIFAR100} \\
 & BPE & BMA & BPE & BMA \\ 
\hline
 M-SGD & 90.02$\pm$0.06  & 92.03$\pm$0.08  & 61.41$\pm$0.15  & 67.04$\pm$0.12  \\
SGHMC  & 90.01$\pm$0.07 & 91.98$\pm$0.05 & 61.46$\pm$0.14  & 66.43$\pm$0.11  \\
{CSGHMC}  & \textbf{90.87}$\pm$\textbf{0.04} & \textbf{92.85}$\pm$\textbf{0.05} & \textbf{63.97$\pm$0.21} & \textbf{68.94$\pm$0.23} \\
\hline
 saM-SGD & 93.83$\pm$0.07  & 94.25$\pm$0.04 & 69.18$\pm$0.13 & 71.83$\pm$0.12 \\
saSGHMC  & 93.80$\pm$0.06 & 94.24$\pm$0.06 & 69.24$\pm$0.11 & 71.98$\pm$0.10 \\
{saCSGHMC}  & \textbf{94.06$\pm$0.07} & 94.33$\pm$0.07 & \textbf{70.18$\pm$0.15} & \textbf{72.67$\pm$0.15} \\
\hline
\end{tabular}
\vspace{-0.3in}
\end{center}
\end{table}

\section{Conclusion}
We have proposed CSGLD as a general scalable Monte Carlo algorithm for both simulation and optimization tasks. 
CSGLD automatically adjusts the invariant distribution during simulations to facilitate escaping from local traps and traversing over the entire energy landscape. 
The sampling bias introduced thereby is accounted for by 
dynamic importance weights. 
We  proved a stability condition for the mean-field system induced by CSGLD together with the convergence of its self-adapting parameter $\btheta$ to a unique fixed point $\btheta_{\star}$. We established the convergence of a weighted averaging estimator for CSGLD. The bias of the estimator decreases as we employ a finer partition, a larger mini-batch size, and smaller learning rates and step sizes. We tested CSGLD and its variants on a few examples, which show their great potential in deep learning and big data computing. % pave the way for future research in other dynamic importance samplers and adaptive biasing force techniques for big data problems.

% \newpage
\section*{Broader Impact}

Our algorithm shows a potential to achieve free mode explorations in complex systems such as deep neural networks and greatly avoid the local trap problem. %ensures AI safety by providing more robust predictions and helps build a safer environment. 
It is an extension of the flat histogram algorithms from the Metropolis kernel to the Langevin kernel and paves the way for future research in various dynamic importance samplers and adaptive biasing force (ABF) techniques for big data problems. The Bayesian community and the researchers in the area of Monte Carlo methods will enjoy the benefit of our work. %To our best knowledge, the negative society consequences are not clear and no one will be put at disadvantage. 

\section*{Acknowledgment} 
Liang's research was supported in part by the grants DMS-2015498, R01-GM117597 and R01-GM126089. Lin acknowledges the support from NSF (DMS-1555072, DMS-1736364), BNL Subcontract 382247, W911NF-15-1-0562, and DE-SC0021142.

% Ethical aspects: No ethics involved; Societal consequences: This may help build a more reliable AI and a safer environment; Positive outcomes: Our algorithm ensures AI safety by providing more robust predictions; Negative outcomes \& consequences of failure: not clear; Who benefit from the algorithm: Researchers, engineers and students; Who may be put at disadvantage: No one; Whether the method leverages biases in the data: No.

% \section*{Acknowledgements}
% We thank the reviewers for their suggestions. We acknowledge the support from the Bilsland Dissertation Fellowship (Deng), the National Science Foundation DMS-1555072, DMS-1736364, DMS-1821233 (Lin) and DMS-1818674 (Liang) and the GPU grant program from NVIDIA.

\bibliography{mybib}
\bibliographystyle{plainnat}
%\bibliographystyle{plain}

%\end{document}

\appendix
\newpage
$\newline$
\appendix
\begin{large}
\begin{center}
    \textbf{Supplimentary Material for ``A Contour Stochastic Gradient Langevin Dynamics Algorithm for Simulations of Multi-modal Distributions''}
\end{center}
\end{large}

\renewcommand\thesection{\Alph{section}}
\def\qed{ \ \vrule width.2cm height.2cm depth0cm\smallskip}
\newcommand\myeq{\stackrel{\mathclap{\normalfont\mbox{A}}}{=}}

\newcommand{\la}{\langle}
\def \proof{{\noindent \bf Proof\quad}}

The supplementary material is organized as follows: Section \ref{review} provides a  review 
for the related methodologies, Section \ref{convergence} proves the stability condition and convergence of the self-adapting parameter, Section \ref{ergodicity} establishes the ergodicity of the contour stochastic gradient Langevin dynamics (CSGLD) algorithm,  and Section \ref{ext} provides more discussions for the algorithm. 

\section{Background on stochastic approximation and Poisson equation}
\label{review}

\subsection{Stochastic approximation}
Stochastic approximation \citep{Albert90} provides a standard framework for the development of adaptive algorithms. Given a random field function $\widetilde H(\bm{\btheta}, \bm{\bx})$, the goal of the stochastic approximation algorithm is to find the solution to the  mean-field equation $h(\btheta)=0$, i.e., solving
\begin{equation*}
\begin{split}
\label{sa00}
h(\btheta)&=\int_{\MX} \widetilde H(\bm{\theta}, \bm{\bx}) \varpi_{\bm{\theta}}(d\bm{\bx})=0,
\end{split}
\end{equation*}
where $\bx\in \MX \subset \mathbb{R}^d$, $\btheta\in\bTheta \subset \mathbb{R}^{m}$, $\widetilde H(\btheta,\bx)$ is a random field 
function and $\varpi_{\btheta}(\bx)$ is a distribution function of $\bx$ depending on the parameter $\btheta$. The stochastic approximation  algorithm works by repeating the following iterations
\begin{itemize}
\item[(1)] Draw $\bm{x}_{k+1}\sim\Pi_{\bm{\theta_{k}}}(\bm{x}_{k}, \cdot)$, where $\Pi_{\bm{\theta_{k}}}(\bm{x}_{k}, \cdot)$ is a transition kernel that admits $ \varpi_{\bm{\theta}_{k}}(\bm{x})$ as
the invariant distribution,

\item[(2)] Update $\bm{\theta}_{k+1}=\bm{\theta}_{k}+\omega_{k+1} \widetilde H(\bm{\theta}_{k}, \bm{x}_{k+1})+\omega_{k+1}^2 \rho(\bm{\theta}_{k}, \bm{x}_{k+1}),$
where $\rho(\cdot,\cdot)$ denotes a bias term. 
\end{itemize}

The algorithm differs from the Robbins–Monro algorithm \citep{RobbinsM1951} in that $\bx$ is simulated from a transition kernel $\Pi_{\bm{\theta_{k}}}(\cdot, \cdot)$ instead of the exact distribution $\varpi_{\bm{\theta}_{k}}(\cdot)$. As a result, a Markov state-dependent noise $\widetilde H(\btheta_k, \bx_{k+1})-h(\btheta_k)$ is generated, which requires some regularity conditions to control the fluctuation $\sum_k \Pi_{\btheta}^k (\widetilde H(\btheta, \bx)-h(\btheta))$. Moreover, it supports a more general form where a bounded bias term $\rho(\cdot,\cdot)$ is allowed without affecting the theoretical properties of the algorithm.

\subsection{Poisson equation}

Stochastic approximation generates a nonhomogeneous Markov chain $\{(\bx_k, \btheta_k)\}_{k=1}^{\infty}$, for which the convergence theory can be studied based on the Poisson equation 
\begin{equation*}
    \mu_{\btheta}(\bm{x})-\mathrm{\Pi}_{\bm{\theta}}\mu_{\bm{\theta}}(\bm{x})=\widetilde H(\bm{\theta}, \bm{x})-h(\bm{\theta}),
\end{equation*}
where $\Pi_{\bm{\theta}}(\bm{x}, A)$ is the transition kernel for any Borel subset $A\subset \MX$ and $\mu_{\btheta}(\cdot)$ is a function on $\MX$.
The solution to the Poisson equation exists when 
the following series converges:
\begin{equation*}
    \mu_{\btheta}(\bx):=\sum_{k\geq 0} \Pi_{\btheta}^k (\widetilde H(\btheta, \bx)-h(\btheta)).
\end{equation*}
That is, the consistency of the estimator $\btheta$ can be established by controlling the perturbations of $\sum_{k \geq 0} \Pi_{\btheta}^k (\widetilde H(\btheta, \bx)-h(\btheta))$ via imposing some regularity conditions on $\mu_{\btheta}(\cdot)$. Towards this goal, \citet{Albert90} gave 
the following regularity conditions on $\mu_{\btheta}(\cdot)$ to ensure the convergence of the adaptive algorithm:

{\it There exist a function  $V: \MX \to [1,\infty)$, and a constant $C$ such that for all $\bm{\theta}, \bm{\theta}'\in \bm{\bTheta}$,}
\begin{equation*}
\begin{split}
\|\mathrm{\Pi}_{\bm{\theta}}\mu_{\btheta}(\bx)\|&\leq C V(\bx),\quad
\|\mathrm{\Pi}_{\bm{\theta}}\mu_{\bm{\theta}}(\bx)-\mathrm{\Pi}_{\bm{\theta'}}\mu_{\bm{\theta'}}(\bx)\|\leq C\|\bm{\theta}-\bm{\theta}'\| V(\bx),  \quad 
\E[V(\bx)]\leq \infty,\\
\end{split}
\end{equation*}
which requires only the first order smoothness. In contrast, the ergodicity theory by \citet{mattingly10} and \citet{VollmerZW2016} 
relies on the much stronger 4th order smoothness.

% Poisson equation has been widely used in ergodic theory and adaptive algorithms to prove the desired limit of a time-average. Consider the infinitesimal generator $\mathcal{L}$ of the overdamped Langevin diffusion and let $\phi$ solve the Poisson equation
% \begin{equation}
%     \mathcal{L}\phi(\bx):=g-\bar g, 
% \end{equation}
% where $g$ is a test function and $\bar g$ is the expectation of $g$ over the Gibbs measure, defined as $\bar g=\int_{\bX}g(\bx) \varpi(d\bx)$. It is known that in a d-dimensional torus $\mathbb{T}^d$ and under elliptic settings, there is a unique solution for the Poisson equation, which is at least k+2-order smooth given a k-order smooth test function $g$ \citep{mattingly10}. To extend the ergodic average from $\mathbb{T}^d$ to $\mathbb{R}^d$, \citep{VollmerZW2016} established the required assumptions to establish the existence of smooth solutions of Poisson equation for stochastic gradient Langevin dynamics.

\section{Stability and convergence analysis for CSGLD} \label{convergence}

\subsection{CSGLD algorithm} \label{Alg:app}

To make the theory more general, we slightly extend CSGLD by allowing a higher order bias term. The resulting algorithm works by iterating between the following two steps:
\begin{itemize}
\item[(1)] Sample $\bm{x}_{k+1}=\bx_k- \epsilon_k\nabla_{\bx} \widetilde L(\bx_k, \btheta_k)+\mathcal{N}({0, 2\epsilon_k \tau\bm{I}}), \ \ \ \ \ \ \ \ \ \ \ \ \ \ \ \ \ \ \ \ \ \ \ \ \ \ \ \ \ \ \ \ \ \ \ \ \ \ \ \ \ \ \ \ \ \ \ \ \ \ \ \ \ \ \ \ \ \ \ \ (\text{S}_1)$

\item[(2)] Update $\bm{\theta}_{k+1}=\bm{\theta}_{k}+\omega_{k+1} \widetilde H(\bm{\theta}_{k}, \bm{x}_{k+1})
+\omega_{k+1}^2 \rho(\bm{\theta}_{k}, \bm{x}_{k+1}),
\ \ \ \ \ \ \ \ \ \ \ \ \ \ \ \ \ \ \ \ \ \ \ \ \ \ \ \ \ \ \ \ \ \ \ \ \ \  (\text{S}_2)$
\end{itemize}
where $\epsilon_k$ is the learning rate, $\omega_{k+1}$ is the step size, $\nabla_{\bx} \widetilde L(\bx, \btheta)$ is the 
stochastic gradient given by 
\begin{equation}
    \nabla_{\bx} \widetilde{L}(\bx,\btheta)= \frac{N}{n} \left[1+ 
   \frac{\zeta\tau}{\Delta u}  \left(\textcolor{black}{\log \theta(J_{\widetilde U}(\bx))-\log\theta((J_{\widetilde U}(\bx)-1)\vee 1)} \right) \right]  
    \nabla_{\bx} \widetilde U(\bx),
\end{equation}
$\widetilde H(\btheta,\bx)=(\widetilde H_1(\btheta,\bx), \ldots, 
 \widetilde H_m(\btheta,\bx))$ is a random field function with
\begin{equation}
\label{def_tilde_H}
     \widetilde H_i(\btheta,\bx)={\theta}^{\zeta}(J_{\widetilde U}(\bx))\left(1_{i= J_{\widetilde U}(\bx)}-{\theta}(i)\right), \quad i=1,2,\ldots,m,
\end{equation}
for some constant $\zeta>0$, and $\rho(\btheta_k,\bx_{k+1})$ is a bias term.

\subsection{Convergence of parameter estimation} 
\label{App:convergence}

To establish the convergence of $\btheta_k$, we make the following assumptions:

\begin{assump}[Compactness] \label{ass2a} 
The space $\Theta$ is compact such that $\inf_{\Theta} \theta(i) >0$ for any  $i\in \{1,2,\ldots,m\}$; the perturbation term $\rho(\btheta, \bx)$ is also uniformly bounded for any $\bx \in \MX$ and $\btheta \in \bTheta$.
\end{assump}

To simplify the proof, we consider a slightly stronger assumption such that $\inf_{\Theta} \theta(i)>0$ holds for any $i \in \{1,2,\ldots,m\}$. To relax this assumption, we refer interested readers to \citet{Fort15} where the recurrence property was proved for the sequence $\{\btheta_k\}_{k\geq 1}$ of a similar algorithm. Such a property guarantees $\btheta_k$ to visit often 
enough to a desired compact space, rendering the convergence of the sequence. 

By Assumption \ref{ass2a} and Eq.(\ref{def_tilde_H}), it is easy to conclude that there exists a constant $Q>0$ such that for any $\btheta\in \bTheta$ and $\bx \in \MX$, 
\begin{equation}
\label{compactness}
     \|\btheta\|\leq Q, \quad  
     \|\widetilde H(\btheta, \bx)\|\leq Q, \quad 
     \|\rho(\btheta, \bx)\|\leq Q.
\end{equation}

% \begin{assump}[Smoothness]
% \label{ass2}
% $\nabla_{\bx} L(\bm{\xeta}, \bm{\theta})$ is $M$-smooth, namely, for any $\bx, \bx'\in \bX$, $\bm{\theta}, \bm{\theta}'\in \bTheta$,
% \begin{equation}
% \label{ass_2_1_eq}
% \begin{split}
% \|\nabla_{\bx} L(\bx, \btheta)-\nabla_{\bx} L(\bm{\bx}', \btheta')\| & \leq M\|\bx-\bx'\|+M\|\btheta-\btheta'\|, \\
% %\|\nabla_{\bx} L(\bx, \btheta)-\nabla_{\bx} L(\bm{\bx}, \btheta')\| & \leq M\|\btheta-\btheta'\|, \\
% \end{split}
% \end{equation}
% \end{assump}

\textcolor{black}{\begin{assump}[Smoothness]
\label{ass2}
$U(\bm{\xeta})$ is $M$-smooth; that is, there exists a constant $M>0$ such that for any $\bx, \bx'\in \MX$,
\begin{equation}
\label{ass_2_1_eq}
\begin{split}
\|\nabla_{\bx} U(\bx)-\nabla_{\bx} U(\bm{\bx}')\| & \leq M\|\bx-\bx'\|. \\
\end{split}
\end{equation}
\end{assump}}

Smoothness is a standard assumption in the study of convergence of SGLD, see e.g. \citet{Maxim17,Xu18}.

\begin{assump}[Dissipativity]
\label{ass3}
 There exist constants $\tilde{m}>0$ and $\tilde{b}\geq 0$ such that for any $\bx \in \MX$ and $\btheta \in \bTheta$, 
\label{ass_dissipative}
\begin{equation}
\label{eq:01}
\langle \nabla_{\bx} L(\bx, \btheta), \bx\rangle\leq \tilde{b}-\tilde{m}\|\bx\|^2.
% \textcolor{red}{\langle \nabla_{\bx} U(\bx), \bx\rangle\leq \tilde{b}-\tilde{m}\|\bx\|^2.
%}
\end{equation}
\end{assump}
This assumption ensures samples to move towards the origin regardless the initial point, 
which is standard in proving the geometric ergodicity of dynamical systems, see e.g. \citet{mattingly02, Maxim17, Xu18}.

\begin{assump}[Gradient noise] 
\label{ass4}
The stochastic gradient is unbiased, that is, 
\begin{equation*}
\E[\nabla_{\bx}\widetilde U(\bx_{k})-\nabla_{\bx} U(\bx_{k})]=0;
\end{equation*}
in addition, there exist some constants $M>0$ and 
$B>0$ such that
\begin{equation*} 
\E [ \|\nabla_{\bx}\widetilde U(\bx_{k})-\nabla_{\bx} U(\bx_{k})\|^2 ] \leq M^2 \|\bx\|^2+B^2,
\end{equation*}
where the expectation $\E[\cdot]$ is taken with respect to the distribution of the noise component in $\nabla_{\bx} \widetilde{U}(\bx)$.
\end{assump}

Lemma \ref{convex_appendix__} establishes a stability condition for CSGLD, which implies potential 
convergence of $\btheta_k$.

\begin{lemma}[Stability, restatement of Lemma \ref{convex_main}]
\label{convex_appendix__} 
%Suppose that the batch size $n$ is sufficiently large, the sample space partition is fine, and the learning rate $\epsilon$ is sufficiently sma
Suppose that Assumptions  \ref{ass2a}-\ref{ass4}  hold. 
For any $\btheta \in \bTheta$, $\langle h(\btheta), \btheta - \btheta_{\star}\rangle \leq  -\phi\|\btheta - \btheta_{\star}\|^2+\mathcal{O}\left(\sup_{\bx} \Var(\xi_n(\bx))+\epsilon+\frac{1}{m}\right)$, where $\phi=\inf_{\btheta} Z_{\btheta}^{-1}>0$, $\theta_{\star}=(\int_{\MX_1}\pi(\bx)d\bx,\int_{\MX_2}\pi(\bx)d\bx,\ldots,\int_{\MX_m}\pi(\bx)d\bx)$, $\Var(\xi_n(\cdot))$ denotes the variance of the noise of the stochastic energy estimator $\xi_n(\cdot)$ of batch size $n$ and it 
decays to $0$ as $n\rightarrow N$.
\end{lemma}

\begin{proof}
 Let $\varpi_{\Psi_{\btheta}}(\bx)\propto\frac{\pi(\bx)}{\Psi^{\zeta}_{\btheta}(U(\bx))}$ denote a theoretical invariant measure of SGLD, where 
  $\Psi_{\btheta}(u)$ is a fixed
  piecewise continuous function given by 
\begin{equation}\label{new_design_appendix}
\Psi_{\btheta}(u)= \sum_{i=1}^m \left(\theta(i-1)e^{(\log\theta(i)-\log\theta(i-1)) \frac{u-u_{i-1}}{\Delta u}}\right) 1_{u_{i-1} < u \leq u_i},
\end{equation}
 the full data is used 
 in determining the indexes of subregions, and the learning rate converges to zero. 
 %Let  $\varpi_{\btheta}(\bx)$ denote the empirical 
 %measure of SGLD where the mini-batch of data is used 
 %in determining the indexes of subregions. Obviously, 
 %. 
% According to the convergence theory 
 %of SGLD, see e.g., \citep{Sato2014ApproximationAO}, \citep{DalalyanK2017}, \citep{SongLeSGLD2020} and \citep{bhatia2019bayesian},  
%the empirical measure ${\varpi}_{\btheta}(\bx)$
%of CSGLD converges to $\varpi_{\Psi_{\btheta}}(\bx)$ as the mini-batch  size $n$ approaches the full data size $N$ and the learning rate $\epsilon$ converges to zero. 
 In addition, we 
 define a piece-wise constant function 
 \[
 \widetilde{\Psi}_{\btheta}=\sum_{i=1}^m \theta(i) 1_{u_{i-1} < u \leq u_{i}},
 \]
 and a theoretical measure 
 $\varpi_{\widetilde{\Psi}_{\btheta}}(\bx) \propto \frac{\pi(\bx)}{\theta^{\zeta}(J(\bx))}$. 
 Obviously, as the sample space partition becomes 
 fine and fine, i.e., $u_1 \to u_{\min}$, $u_{m-1}\to u_{\max}$ and $m \to \infty$, we have  
 $\|\widetilde{\Psi}_{\btheta}-\Psi_{\btheta}\|\to 0$ and $\| 
 \varpi_{\widetilde{\Psi}_{\btheta}}(\bx)- 
 \varpi_{\Psi_{\btheta}}(\bx) \|\to 0$, where 
 $u_{\min}$ and $u_{\max}$ denote the minimum and maximum of $U(\bx)$, respectively. 
 Without loss of generality, we assume $u_{\max}<\infty$. Otherwise, $u_{\max}$ can be set to a value such that $\pi(\{\bx: U(\bx)>u_{\max}\})$ is sufficiently small.

%By adaptively simulating from $\varpi_{\Psi_{\btheta}}(\bx)\propto\frac{\pi(\bx)}{\Psi^{\zeta}_{\btheta}(U(\bx))}$, where \begin{equation}
%    \label{psi_define}
%    \Psi_{\btheta}(u)= \sum_{i=1}^m \left(\theta(i-1)e^{(\log\theta(i)-\log\theta(i-1)) \frac{u-u_{i-1}}{\Delta u}}\right) 1_{u_{i-1} < u \leq u_i},
%\end{equation}
%the random field $\widetilde H_i(\btheta,\bx)$ acts on an empirical measure $\varpi_{\btheta}(\bx)$ which asymptotically approximates the invariant measure $\varpi_{\Psi_{\btheta}}(\bx)$ as $\epsilon\rightarrow 0$ and $n\rightarrow N$. Meanwhile, $\varpi_{\btheta}(\bx):= \frac{1}{Z_{\btheta}} 
%\frac{\pi(\bx)}{\theta^{\zeta}(J(\bx))}$, where $Z_{\btheta}=\sum_{i=1}^m \frac{\int_{\bchi_i \pi(\bx)d\bx}}{\theta(i)^{\zeta}}$, is sufficiently close to $\varpi_{\Psi_{\btheta}}(\bx)$ by Lemma \ref{partition_order}. 

For each $i \in \{1,2,\ldots,m\}$, the random field $\widetilde H_i(\btheta,\bx)={\theta}^{\zeta}(J_{\widetilde U}(\bx))\left(1_{i\geq J_{\widetilde U}(\bx)}-{\theta}(i)\right)$ is a biased estimator of $ H_i(\btheta,\bx)={\theta}^{\zeta}( J(\bx))\left(1_{i\geq J(\bx)}-{\theta}(i)\right)$. By Lemma 4 of \citep{icsgld}, the bias caused by the mini-batch energy estimator follows $\E[\widetilde{H}(\btheta,\bx)-H(\btheta,\bx)]=O\left(\sup_{\bx} \Var(\xi_n(\bx))\right)$, where $\xi_n(\cdot)$ is the noisy energy estimator of batch size $n$ and $\Var(\cdot)$ denotes the variance.

First, let's compute the mean-field $h(\btheta)$ with respect to the empirical measure $\varpi_{\btheta}(\bx)$:
\begin{equation}
\small
\label{iiii}
\begin{split} 
        h_i(\btheta)&=\int_{\MX} \widetilde H_i(\btheta,\bx) 
         \varpi_{\btheta}(\bx) d\bx
         =\int_{\MX} H_i(\btheta,\bx) 
         \varpi_{\btheta}(\bx) d\bx+O\left(\sup_{\bx} \Var(\xi_n(\bx))\right)\\
         &=\ \int_{\MX} H_i(\btheta,\bx) \left( \underbrace{\varpi_{\widetilde{\Psi}_\btheta}(\bx)}_{\text{I}_1} \underbrace{-\varpi_{\widetilde{\Psi}_\btheta}(\bx)+\varpi_{\Psi_{\btheta}}(\bx)}_{\text{I}_2: \text{piece-wise approximation}}\underbrace{-\varpi_{\Psi_{\btheta}}(\bx)+\varpi_{\btheta}(\bx)}_{\text{I}_3: \text{discretization}}\right) d\bx+O\left(\sup_{\bx} \Var(\xi_n(\bx))\right).\\
\end{split}
\end{equation}

(i) For the term $\text{I}_1$, we have
\begin{equation}
\begin{split}
\label{i_1}
    \int_{\MX} H_i(\btheta,\bx) 
     \varpi_{\widetilde{\Psi}_\btheta}(\bx) d\bx&=\frac{1}{Z_{\btheta}} \int_{\MX} {\theta}^{\zeta}(J(\bx))\left(1_{i= J(\bx)}-{\theta}(i)\right) \frac{\pi(\bx)}{\theta^{\zeta}(J(\bx))} d\bx\\
    &=Z_{\btheta}^{-1}\left[\sum_{k=1}^m \int_{\MX_k} 
     \pi(\bx) 1_{k=i} d\bx -\theta(i)\sum_{k=1}^m\int_{\MX_k} \pi(\bx)d\bx \right] \\
    &=Z_{\btheta}^{-1} \left[\theta_{\star}(i)-\theta(i)\right],
\end{split}
\end{equation}
where $Z_{\btheta}=\sum_{i=1}^m \frac{\int_{\MX_i} \pi(\bx)d\bx}{\theta(i)^{\zeta}}$ denotes the normalizing constant 
of $\varpi_{\widetilde{\Psi}_\btheta}(\bx)$.

(ii) As to the integral $\text{I}_2$. By Lemma \ref{partition_order} and the boundedness of $H(\btheta,\bx)$, we have 
\begin{equation} \label{biasI2}
\int_{\MX} H_i(\btheta,\bx) (-\varpi_{\widetilde{\Psi}_{\btheta}}(\bx)+\varpi_{\Psi_{\btheta}}(\bx)) d\bx= \mathcal{O}\left(\frac{1}{m}\right).
\end{equation}
(iii) Regarding the term $\text{I}_3$, we have for any fixed $\btheta$,
\begin{equation}\label{iiii_2}
    \int_{\MX} H_i(\btheta,\bx) \left(-\varpi_{\Psi_{\btheta}}(\bx)+\varpi_{\btheta}(\bx)\right) d\bx=\mathcal{O}(\epsilon),
\end{equation}
where  the order of $\mathcal{O}(\epsilon)$ follows from Theorem 6 of \citet{Sato2014ApproximationAO}.

Plugging (\ref{i_1}), (\ref{biasI2}) and  \textcolor{black}{(\ref{iiii_2})} into (\ref{iiii}), we have
\begin{equation}\label{h_i_theta}
     h_i(\btheta)=Z_{\btheta}^{-1} \left[\varepsilon\beta_i(\btheta)+\theta_{\star}(i)-\theta(i)\right],
\end{equation}
where $\varepsilon=\mathcal{O}\left(\sup_{\bx} \Var(\xi_n(\bx))+\epsilon+\frac{1}{m}\right)$ and $\beta_i(\btheta)$ is a bounded term such that $Z_{\btheta}^{-1}\varepsilon\beta_i(\btheta)=\mathcal{O}\left(\sup_{\bx} \Var(\xi_n(\bx))+\epsilon+\frac{1}{m}\right)$.

To solve the ODE system with small disturbances, we consider standard techniques in perturbation theory.
%and write $\theta(i)=\theta^{(0)}(i)+\varepsilon\theta^{(1)}(i)+\varepsilon^2\theta^{(2)}(i)+\cdot\cdot\cdot$. Then, it suffices to solve 
%\begin{equation}
%    \varepsilon \beta_i\biggl(\left(\theta^{(0)}(i)+\varepsilon\theta^{(1)}(i)+\varepsilon^2\theta^{(2)}(i)+\cdots,\cdots\right)\biggr)+\theta_{\star}(i)-\biggl(\theta^{(0)}(i)+\varepsilon\theta^{(1)}(i)+\varepsilon^2\theta^{(2)}(i)+\cdot\cdot\cdot\biggr)=0
%\end{equation}
%According to the fundamental theorem of perturbation theory, we have $\theta^{(0)}(i)=\theta_{\star}(i)$ 
%for any $i$, given small enough learning rates, large enough batch sizes, and fine enough partitions.
%In what follows, we can show $\theta^{(1)}(i)=\beta_i(\btheta_{\star})$ also holds for any $i$. The solution is summarized as follows
According to the fundamental theorem of
perturbation theory \citep{Eric}, we can obtain the solution to 
the mean field equation $h(\btheta)=0$: 
\begin{equation}
    \theta(i)=\theta_{\star}(i)+\varepsilon\beta_i(\btheta_{\star}) +\mathcal{O}(\varepsilon^2), \quad i=1,2,\ldots,m,
\end{equation}
which is a stable point in a small neighbourhood of $\btheta_{\star}$.

Considering the positive definite function $\mathbb{V}(\btheta)=\frac{1}{2}\| \btheta_{\star}-\btheta\|^2$ for the mean-field system $h(\btheta)=Z_{\btheta}^{-1} (\varepsilon\beta_i(\btheta)+\btheta_{\star}-\btheta)=Z_{\btheta}^{-1} (\btheta_{\star}-\btheta)+\mathcal{O}(\varepsilon)$, we have
\begin{equation*}
\begin{split}
    \langle h(\btheta), \mathbb{V}(\btheta)\rangle&=\langle h(\btheta), \btheta - \btheta_{\star}\rangle \\
    &= -Z_{\btheta}^{-1}\|\btheta - \btheta_{\star}\|^2+\mathcal{O}(\varepsilon)\\
    &\leq -\phi\|\btheta - \btheta_{\star}\|^2+\mathcal{O}\left(\sup_{\bx} \Var(\xi_n(\bx))+\epsilon+\frac{1}{m}\right),
\end{split}
\end{equation*}
where $\phi=\inf_{\btheta} Z_{\btheta}^{-1}>0$ by
the compactness assumption \ref{ass2a}. This concludes the proof.
\end{proof}

%  For a function $V: \bX \to [1,\infty)$ and a function $q: \bX \to \mR^{m}$, 
%  define the norm
% \[
%  \|q\|_V=\sup_{\bx\in \bX} \frac{\|q(\bx)\|}{V(\bx)},
%  \]
%  where $m$ denotes the dimension of $\btheta$. 
%  Let $\mathcal{L}_V=\{q: \bX \to \mR^{m}, \sup_{x\in \bX} \|q\|_V <\infty\}$.

% The following lemma is a restatement of Lemma 1 in \citep{deng2019}, for which the required conditions are 
% clearly satisfied due to the compactness assumption \ref{ass2a}.
% \begin{lemma}[Uniform $L^2$ bounds]
% \label{lemma:1}
% Suppose Assumptions \ref{ass2a}-\textcolor{black}{\ref{ass4}} holds.  Given a small enough learning rate
%  $\ 0<\epsilon<\operatorname{Re}(\tfrac{\tilde{m}-\sqrt{\tilde{m}^2-3M^2}}{3M^2})\wedge 1$,  then 
% $\sup_{k\geq 1} \E[\|\bm{\xeta}_{k}\|^2] < \infty$.
% \end{lemma}

\textcolor{black}{The following is a restatement of Lemma 3.2 \citep{Maxim17}, which holds for any $\btheta$ in the compact space $\bTheta$. Similar results have been shown in \citet{deng2019}.
\begin{lemma}[Uniform $L^2$ bounds]
\label{lemma:1}
Suppose Assumptions \ref{ass2a}, \ref{ass3} and \ref{ass4} hold.  Given a small enough learning rate, then 
$\sup_{k\geq 1} \E[\|\bm{\xeta}_{k}\|^2] < \infty$.
\end{lemma}}

The ensure the perturbations of the stochastic approximation process decays sufficiently fast, we present the following result on the regularity properties.
\begin{lemma}[Solution of Poisson equation]
\label{lyapunov}
Suppose that Assumptions  \ref{ass2a}-\ref{ass4}  hold. 
There is a solution $\mu_{\btheta}(\cdot)$ on $\MX$ to the Poisson equation 
\begin{equation}
    \label{poisson_eqn}
    \mu_{\btheta}(\bm{x})-\mathrm{\Pi}_{\bm{\theta}}\mu_{\bm{\theta}}(\bm{x})=\widetilde H(\bm{\theta}, \bm{x})-h(\bm{\theta}).
\end{equation}
In addition, for all $\bm{\theta}, \bm{\theta}'\in \bm{\bTheta}$, %and a function  $V(\bx)=1+\|\bx\|^2$, 
there exists a constant $C$ such that
\begin{equation}
\begin{split}
\label{poisson_reg}
\E[\|\mathrm{\Pi}_{\bm{\theta}}\mu_{\btheta}(\bx)\|]&\leq C,\\
\E[\|\mathrm{\Pi}_{\bm{\theta}}\mu_{\bm{\theta}}(\bx)-\mathrm{\Pi}_{\bm{\theta}'}\mu_{\bm{\theta'}}(\bx)\|]&\leq C\|\bm{\theta}-\bm{\theta}'\|.\\
\end{split}
\end{equation}
\end{lemma}
\begin{proof} The proof hinges on verifying drift conditions proposed in Section 6 of \citep{andrieu05} and the details have been given in Lemma 6 of \citet{icsgld}.

% The lemma can be proved based on 
% Theorem 13 of  \citet{VollmerZW2016}, 
% whose conditions can be 
% easily verified for CSGLD given the assumptions A1-A4 and 
% Lemma \ref{lemma:1}. The details are omitted. 
\end{proof}

Now we are ready to prove the first main result on the 
convergence of $\btheta_k$.
The technique lemmas are listed 
in Section \ref{Lemmasection}.

\begin{assump}[Learning rate and step size]
\label{ass1}
The learning rate $\{\epsilon_k\}_{k \in \mathrm{N}}$ is a positive non-increasing sequence of real numbers satisfying the conditions 
\[
\lim_k \epsilon_k=0, \quad \sum_{k=1}^{\infty} \epsilon_k=\infty.
\]
%which includes a constant learning rate as a special case.
The step size $\{\omega_{k}\}_{k\in \mathrm{N}}$ is a positive decreasing sequence of real numbers such that
\begin{equation} \label{a1}
\omega_{k}\rightarrow 0, \ \ \sum_{k=1}^{\infty} \omega_{k}=+\infty,\ \  \lim_{k\rightarrow \infty} \inf 2\phi  \dfrac{\omega_{k}}{\omega_{k+1}}+\dfrac{\omega_{k+1}-\omega_{k}}{\omega^2_{k+1}}>0.
\end{equation}
According to \citet{Albert90}, we can choose $\omega_{k}:=\frac{A}{k^{\alpha}+B}$ for some $\alpha \in (\frac{1}{2}, 1]$ and some suitable constants 
 $A>0$ and $B>0$. 
 \end{assump}

\begin{theorem}[$L^2$ convergence rate, restatement of Theorem \ref{thm:1}]
\label{latent_convergence}
Suppose Assumptions $\ref{ass2a}$-$\ref{ass1}$ hold. For a sufficiently
large value of $m$, a sufficiently small learning rate sequence  $\{\epsilon_k\}_{k=1}^{\infty}$,  and a sufficiently small
 step  size sequence $\{\omega_k\}_{k=1}^{\infty}$, 
$\{\btheta_k\}_{k=0}^{\infty}$ converges to
 $\btheta_{\star}$ in $L_2$-norm  such that
\begin{equation*}
    \E\left[\|\bm{\theta}_{k}-\btheta_{\star}\|^2\right]=\mathcal{O}\left( \omega_{k}+
    \sup_{i\geq k_0}\epsilon_i+
\frac{1}{m} +\sup_{\bx} \Var(\xi_n(\bx))\right),
\end{equation*}
where $k_0$ is a sufficiently large constant, $\xi_n$ is the noisy energy estimator of batch size $n$, and $\Var(\cdot)$ denotes the variance.
\end{theorem}
%\paragraph{Proof of Theorem \ref{latent_convergence}} 
\begin{proof}
Consider the iterations 
\begin{equation*}
    \bm{\theta}_{k+1}=\bm{\theta}_{k}+\omega_{k+1} \left(\widetilde H(\bm{\theta}_{k}, \bm{x}_{k+1})+\omega_{k+1} \rho(\btheta_k, \bx_{k+1})\right).
\end{equation*}
Define $\bm{T}_{k}=\bm{\theta}_{k}-\btheta_{\star}$. By subtracting $\btheta_{\star}$ from both sides and taking the square and $L_2$ norm,  we have
\begin{equation*}
\small
\begin{split}
    \|\bT_{k+1}\|^2&=\|\bT_k\|^2 +\omega_{k+1}^2 \|\widetilde H(\btheta_k, \bx_{k+1}) + \omega_{k+1}\rho(\btheta_k, \bx_{k+1})\|^2+2\omega_{k+1}\underbrace{\langle \bT_k,  \widetilde H(\bx_{k+1})+\omega_{k+1}\rho(\btheta_k, \bx_{k+1})\rangle}_{\text{D}}.
\end{split}
\end{equation*}

First, by Lemma \ref{convex_property}, there exists a constant $G=4Q^2(1+Q^2)$ such that
\begin{equation}
\label{first_term}
    \| \widetilde H(\btheta_k, \bx_{k+1}) + \omega_{k+1}\rho(\btheta_k, \bx_{k+1})\|^2 \leq G (1+\|\bT_k\|^2).
\end{equation}

Next, by the Poisson equation (\ref{poisson_eqn}), we have
\begin{equation*}
\begin{split}
   \text{D}&=\langle \bT_k,  \widetilde H(\btheta_k, \bx_{k+1})+\omega_{k+1}\rho(\btheta_k, \bx_{k+1}) \rangle\\
   &=\langle \bT_k,  h(\btheta_k)+\mu_{\btheta_k}(\bm{x}_{k+1})-\mathrm{\Pi}_{\bm{\theta}_k}\mu_{\bm{\theta}_k}(\bm{x}_{k+1})+\omega_{k+1}\rho(\btheta_k, \bx_{k+1}) \rangle\\
   &=\underbrace{\langle \bT_k,  h(\btheta_k)\rangle}_{\text{D}_{1}} +\underbrace{\langle\bT_k, \mu_{\btheta_k}(\bm{x}_{k+1})-\mathrm{\Pi}_{\bm{\theta}_k}\mu_{\bm{\theta}_k}(\bm{x}_{k+1})\rangle}_{\text{D}_{2}}+\underbrace{\langle \bT_k, \omega_{k+1}\rho(\btheta_k, \bx_{k+1})\rangle}_{{\text{D}_{3}}}.
\end{split}
\end{equation*}

For the term $\text{D}_1$, by Lemma \ref{convex_appendix__}, we have
\begin{align*}
\E\left[\langle \bm{T}_{k}, h(\bm{\theta}_{k})\rangle\right] &\leq - \phi\E[\|\bm{T}_{k}\|^2]+\mathcal{O}(\sup_{\bx} \Var(\xi_n(\bx))+\epsilon_k+\frac{1}{m}).
\end{align*}
For convenience, in the following context, we denote $\mathcal{O}(\sup_{\bx} \Var(\xi_n(\bx))+\epsilon_k+\frac{1}{m})$ by $\Delta_k$.

To deal with the term $\text{D}_2$, we make the following decomposition 
\begin{equation*}
\begin{split}
\text{D}_2 &=\underbrace{\langle \bT_k, \mu_{\bm{\theta}_{k}}(\bm{\xeta}_{k+1})-\mathrm{\Pi}_{\bm{\theta}_{k}}\mu_{\bm{\theta}_{k}}(\bm{\bx}_{k})\rangle}_{\text{D}_{21}} \\
&+ \underbrace{\langle \bT_k,\mathrm{\Pi}_{\bm{\theta}_{k}}\mu_{\bm{\theta}_{k}}(\bm{x}_{k})- \mathrm{\Pi}_{\bm{\theta}_{k-1}}\mu_{\bm{\theta}_{k-1}}(\bm{x}_{k})\rangle}_{\text{D}_{22}}
+ \underbrace{\langle \bT_k,\mathrm{\Pi}_{\bm{\theta}_{k-1}}\mu_{\bm{\theta}_{k-1}}(\bm{x}_{k})- \mathrm{\Pi}_{\bm{\theta}_{k}}\mu_{\bm{\theta}_{k}}(\bm{\xeta}_{k+1})\rangle}_{\text{D}_{23}}.\\
\end{split}
\end{equation*}

(\text{i})  From the Markov property, $\mu_{\bm{\theta}_{k}}(\bm{\xeta}_{k+1})-\mathrm{\Pi}_{\bm{\theta}_{k}}\mu_{\bm{\theta}_{k}}(\bm{x}_{k})$ forms a martingale difference sequence 
$$\E\left[\langle \bT_k, \mu_{\bm{\theta}_{k}}(\bm{\xeta}_{k+1})-\mathrm{\Pi}_{\bm{\theta}_{k}}\mu_{\bm{\theta}_{k}}(\bm{x}_{k})\rangle |\mathcal{F}_{k}\right]=0, \eqno{(\text{D}_{21})}$$
where  $\mathcal{F}_k$ is a $\sigma$-filter formed by $\{\btheta_0, \bx_1, \btheta_1, \bx_2, \cdots, \bx_k,\btheta_k\}$.

(\text{ii})  By the regularity of the solution of Poisson equation in (\ref{poisson_reg}) and Lemma \ref{theta_lip}, we have 
\begin{equation}
\label{theta_delta}
\E[\|\mathrm{\Pi}_{\bm{\theta}_{k}}\mu_{\bm{\theta}_{k}}(\bm{x}_{k})- \mathrm{\Pi}_{\bm{\theta}_{k-1}}
 \mu_{\bm{\theta}_{k-1}}(\bm{x}_{k})\|]\leq C \|\btheta_k-\btheta_{k-1}\|\leq 2Q C\omega_k.
\end{equation}
Using Cauchy–Schwarz inequality, (\ref{theta_delta}) and the compactness of $\Theta$ in Assumption \ref{ass2a}, we have
$$\small{\E[\langle\bm{T}_{k},\mathrm{\Pi}_{\bm{\theta}_{k}}\mu_{\bm{\theta}_{k}}(\bm{x}_{k})- \mathrm{\Pi}_{\bm{\theta}_{k-1}}\mu_{\bm{\theta}_{k-1}}(\bm{x}_{k})\rangle]\leq \E[\|\bT_k\|]\cdot 2Q C\omega_k\leq 4Q^2 C\omega_{k}\leq 5Q^2 C\omega_{k+1}}   \eqno{(\text{D}_{22})},$$
where the last inequality follows from assumption \ref{ass1} and holds for a large enough $k$.

(\text{iii})  For the last term of $\text{D}_{2}$, 
\begin{equation*}
\begin{split}
\small
&\langle \bm{T}_{k},\mathrm{\Pi}_{\bm{\theta}_{k-1}}\mu_{\bm{\theta}_{k-1}}(\bm{x}_{k})- \mathrm{\Pi}_{\bm{\theta}_{k}}\mu_{\bm{\theta}_{k}}(\bm{\xeta}_{k+1})\rangle\\
=&\left(\langle \bm{T}_{k}, \mathrm{\Pi}_{\bm{\theta}_{k-1}}\mu_{\bm{\theta}_{k-1}}(\bm{x}_{k}) \rangle- \langle \bm{T}_{k+1}, \mathrm{\Pi}_{\bm{\theta}_{k}}\mu_{\bm{\theta}_{k}}(\bm{\xeta}_{k+1})\rangle\right)\\
&\ \ \ +\left(\langle \bm{T}_{k+1}, \mathrm{\Pi}_{\bm{\theta}_{k}}\mu_{\bm{\theta}_{k}}(\bm{\xeta}_{k+1})\rangle-\langle \bm{T}_{k}, \mathrm{\Pi}_{\bm{\theta}_{k}}\mu_{\bm{\theta}_{k}}(\bm{\xeta}_{k+1})\rangle\right)\\
=&{({z}_{k}-{z}_{k+1})}+{\langle \bm{T}_{k+1}-\bm{T}_{k}, \mathrm{\Pi}_{\bm{\theta}_{k}}\mu_{\bm{\theta}_{k}}(\bm{\xeta}_{k+1})\rangle},\\
\end{split}
\end{equation*}
where ${z}_{k}=\langle \bm{T}_{k}, \mathrm{\Pi}_{\bm{\theta}_{k-1}}\mu_{\bm{\theta}_{k-1}}(\bm{x}_{k})\rangle$. By the regularity assumption (\ref{poisson_reg}) and Lemma \ref{theta_lip}, 
$$\E\langle \bm{T}_{k+1}-\bm{T}_{k}, \mathrm{\Pi}_{\bm{\theta}_{k}}\mu_{\bm{\theta}_{k}}(\bm{\xeta}_{k+1})\rangle\leq   \E[\|\bm{\theta}_{k+1}-\bm{\theta}_{k}\|] \cdot \E[\|\mathrm{\Pi}_{\bm{\theta}_{k}}\mu_{\bm{\theta}_{k}}(\bm{\xeta}_{k+1})\|] \leq 2Q C \omega_{k+1}.\eqno{(\text{D}_{23})}$$

Regarding $\text{D}_3$, since $\rho(\btheta_k, \bx_{k+1})$ is bounded, applying Cauchy–Schwarz inequality gives
$${\E[\langle \bT_k, \omega_{k+1}\rho(\btheta_k, \bx_{k+1}))]\leq 2Q^2\omega_{k+1}} \eqno{(\text{D}_{3})}$$

Finally, adding (\ref{first_term}), $\text{D}_1$, $\text{D}_{21}$, $\text{D}_{22}$, $\text{D}_{23}$ and $\text{D}_3$ together, it follows that for a constant $C_0 = G+10Q^2C+4QC+4Q^2$,
\begin{equation}
\begin{split}
\label{key_eqn}
\E\left[\|\bm{T}_{k+1}\|^2\right]&\leq (1-2\omega_{k+1}\phi+G\omega^2_{k+1} )\E\left[\|\bm{T}_{k}\|^2\right]+C_0\omega^2_{k+1} +2\Delta_k\omega_{k+1} +2\E[z_{k}-z_{k+1}]\omega_{k+1}.
\end{split}
\end{equation}
Moreover, from (\ref{compactness}) and (\ref{poisson_reg}), $\E[|z_{k}|]$ is upper bounded by
\begin{equation}
\begin{split}
\label{condition:z}
\E[|z_{k}|]=\E[\langle \bm{T}_{k}, \mathrm{\Pi}_{\bm{\theta}_{k-1}}\mu_{\bm{\theta}_{k-1}}(\bm{x}_{k})\rangle]\leq \E[\|\bT_k\|]\E[\|\mathrm{\Pi}_{\bm{\theta}_{k-1}}\mu_{\bm{\theta}_{k-1}}(\bm{x}_{k})\|]\leq 2QC.
\end{split}
\end{equation}

According to Lemma $\ref{lemma:4}$, we can choose $\lambda_0$ and $k_0$ such that 
\begin{align*}
\E[\|\bm{T}_{k_0}\|^2]\leq \psi_{k_0}=\lambda_0 \omega_{k_0}+\frac{1}{\phi}\sup_{i\geq k_0}\Delta_{i},
\end{align*}
which satisfies the conditions ($\ref{lemma:3-a}$) and ($\ref{lemma:3-b}$) of Lemma $\ref{lemma:3-all}$. Applying Lemma $\ref{lemma:3-all}$ leads to
\begin{equation}
\begin{split}
\label{eqn:9}
\E\left[\|\bm{T}_{k}\|^2\right]\leq \psi_{k}+\E\left[\sum_{j=k_0+1}^{k}\Lambda_j^k \left(z_{j-1}-z_{j}\right)\right],
\end{split}
\end{equation}
where $\psi_{k}=\lambda_0 \omega_{k}+\frac{1}{\phi}\sup_{i\geq k_0}\Delta_{i}$ for all $k>k_0$. Based on ($\ref{condition:z}$) and the increasing condition of $\Lambda_{j}^k$ in Lemma $\ref{lemma:2}$, we have
\begin{equation}
\small
\begin{split}
\label{eqn:10}
&\E\left[\left|\sum_{j=k_0+1}^{k} \Lambda_j^k\left(z_{j-1}-z_{j}\right)\right|\right]
=\E\left[\left|\sum_{j=k_0+1}^{k-1}(\Lambda_{j+1}^k-\Lambda_j^k)z_j-2\omega_{k}z_{k}+\Lambda_{k_0+1}^k z_{k_0}\right|\right]\\
\leq& \sum_{j=k_0+1}^{k-1}2(\Lambda_{j+1}^k-\Lambda_j^k)QC+\E[|2\omega_{k} z_{k}|]+2\Lambda_k^k QC\\
\leq& 2(\Lambda_k^k-\Lambda_{k_0}^k)QC+2\Lambda_k^k QC+2\Lambda_k^k QC\\
\leq& 6\Lambda_k^k QC.
\end{split}
\end{equation}

Given $\psi_{k}=\lambda_0 \omega_{k}+\frac{1}{\phi}\sup_{i\geq k_0}\Delta_{i}$ which satisfies the conditions ($\ref{lemma:3-a}$) and ($\ref{lemma:3-b}$) of Lemma $\ref{lemma:3-all}$,
it follows from (\ref{eqn:9}) and (\ref{eqn:10}) 
that  the following inequality holds for any $k>k_0$,
 
\begin{equation*}
\E[\|\bm{T}_{k}\|^2]\leq \psi_{k}+6\Lambda_k^k QC=\left(\lambda_0+12QC\right)\omega_{k}+\frac{1}{\phi}\sup_{i\geq k_0}\Delta_{i}=\lambda \omega_{k}+\frac{1}{\phi}\sup_{i\geq k_0}\Delta_{i},
\end{equation*}
where $\lambda=\lambda_0+12QC$, $\lambda_0=\frac{2G\sup_{i\geq k_0} \Delta_i + 2C_0\phi}{C_1\phi}$, $\small{C_1=\lim \inf 2\phi \dfrac{\omega_{k}}{\omega_{k+1}}+\dfrac{\omega_{k+1}-\omega_{k}}{\omega^2_{k+1}}>0}$, $C_0=G+5Q^2C+2QC+2Q^2$ and $G=4 Q^2(1+Q^2)$.
\end{proof}

\subsection{Technical lemmas} \label{Lemmasection}

\begin{lemma}\label{partition_order}
Suppose Assumption \ref{ass2a} holds, and $u_1$ and 
$u_{m-1}$ are fixed such that $\Psi(u_1)>\nu$ and $\Psi(u_{m-1})>1-\nu$ for some small constant $\nu>0$. For any bounded function $f(\bx)$, we have 
\label{m_order}
\begin{equation}\label{i_2}
    \int_{\MX} f(\bx)\left(\varpi_{\Psi_{\btheta}}(\bx)-\varpi_{\widetilde\Psi_{\btheta}}(\bx)\right) d\bx=\mathcal{O}\left(\frac{1}{m}\right).
\end{equation}
\end{lemma}

\begin{proof}
Recall that $\varpi_{\widetilde\Psi_{\btheta}}(\bx)= \frac{1}{Z_{\btheta}} 
\frac{\pi(\bx)}{\theta^{\zeta}(J(\bx))}$ and $\varpi_{\Psi_{\btheta}}(\bx)=\frac{1}{Z_{\Psi_{\btheta}}}\frac{\pi(\bx)}{\Psi^{\zeta}_{\btheta}(U(\bx))}$. Since $f(\bx)$ is bounded, 
it suffices to show 
\begin{equation}
\begin{split}
    &\int_{\MX} \frac{1}{Z_{\btheta}} 
\frac{\pi(\bx)}{\theta^{\zeta}(J(\bx))}-\frac{1}{Z_{\Psi_{\btheta}}}\frac{\pi(\bx)}{\Psi^{\zeta}_{\btheta}(U(\bx))} d\bx\\
\leq &\int_{\MX} \left|\frac{1}{Z_{\btheta}} 
\frac{\pi(\bx)}{\theta^{\zeta}(J(\bx))}-\frac{1}{Z_{\btheta}}\frac{\pi(\bx)}{\Psi^{\zeta}_{\btheta}(U(\bx))}\right|d\bx+\int_{\MX}\left|\frac{1}{Z_{\btheta}}\frac{\pi(\bx)}{\Psi^{\zeta}_{\btheta}(U(\bx))}-\frac{1}{Z_{\Psi_{\btheta}}}\frac{\pi(\bx)}{\Psi^{\zeta}_{\btheta}(U(\bx))}\right| d\bx\\
=&\underbrace{\frac{1}{Z_{\btheta}}\sum_{i=1}^m \int_{\MX_i} \left| 
\frac{\pi(\bx)}{\theta^{\zeta}(i)}-\frac{\pi(\bx)}{\Psi^{\zeta}_{\btheta}(U(\bx))}\right|d\bx}_{\text{I}_1}+\underbrace{\sum_{i=1}^m\left|\frac{1}{Z_{\btheta}}-\frac{1}{Z_{\Psi_{\btheta}}}\right|\int_{\MX_i}\frac{\pi(\bx)}{\Psi^{\zeta}_{\btheta}(U(\bx))} d\bx}_{\text{I}_2}=\mathcal{O}\left(\frac{1}{m}\right),\\
\end{split}
\end{equation}
where $Z_{\btheta}=\sum_{i=1}^m \int_{\mX_i} \frac{\pi(\bx)}{\theta(i)^{\zeta}}d\bx$, $Z_{\Psi_{\btheta}}=\sum_{i=1}^{m}\int_{\mX_i} \frac{\pi(\bx)}{\Psi^{\zeta}_{\btheta}(U(\bx))}d\bx$, and $\Psi_{\btheta}(u)$ is a piecewise continuous function defined in (\ref{new_design_appendix}).
%such that $\theta(i)<\Psi_{\btheta}(u)\leq \theta(i+1)$ for any $u\in [u_1, \infty)$ and some $i\in\{1,2,..., m-1\}$ that satisfies $u_i< U(\bx)\leq u_{i+1}$. 

By Assumption \ref{ass2a}, $\inf_{\Theta} \theta(i)>0$ for any $i$. 
Further, by the mean-value theorem, which implies $|x^{\zeta}-y^{\zeta}|\lesssim |x-y| z^{\zeta}$ for any $\zeta>0, x\leq y$ and $z\in[x, y]\subset [u_1, \infty)$, we have 
\begin{equation*}
\small
\begin{split}
    \text{I}_1&=\frac{1}{Z_{\btheta}}\sum_{i=1}^m \int_{\MX_i} \left| 
\frac{\theta^{\zeta}(i)-\Psi^{\zeta}_{\btheta}(U(\bx))}{\theta^{\zeta}(i)\Psi^{\zeta}_{\btheta}(U(\bx))}\right|\pi(\bx)d\bx\lesssim \frac{1}{Z_{\btheta}}\sum_{i=1}^m \int_{\MX_i} 
\frac{|\Psi_{\btheta}(u_{i-1})-\Psi_{\btheta}(u_i)|}{\theta^{\zeta}(i)}\pi(\bx)d\bx\\
&\leq \max_i |\Psi_{\btheta}(u_{i}-\Delta u)-\Psi_{\btheta}(u_i)|  \frac{1}{Z_{\btheta}}\sum_{i=1}^m \int_{\MX_i} 
\frac{\pi(\bx)}{\theta^{\zeta}(i)}d\bx=\max_i |\Psi_{\btheta}(u_{i}-\Delta u)-\Psi_{\btheta}(u_i)|\lesssim \Delta u=\mathcal{O}\left(\frac{1}{m}\right),
\end{split}
\end{equation*}
where the last inequality follows by Taylor expansion, 
and the last equality follows as $u_1$ and $u_{m-1}$ 
are fixed. Similarly, we have % the same order for $\text{I}_2$.
\begin{equation*}
    \begin{split}
        \text{I}_2= \left|\frac{1}{Z_{\btheta}}-\frac{1}{Z_{\Psi_{\btheta}}}\right|Z_{\Psi_{\btheta}}=\frac{ |Z_{\Psi_{\btheta}}-Z_{\btheta}|}{Z_{\btheta}}\leq \frac{1}{Z_{\btheta}}\sum_{i=1}^m \int_{\MX_i} \left|\frac{\pi(\bx)}{\theta^{\zeta}(i)}-\frac{\pi(\bx)}{\Psi^{\zeta}_{\btheta}(U(\bx))}\right|d\bx=\text{I}_1=\mathcal{O}\left(\frac{1}{m}\right).
    \end{split}
\end{equation*}
The proof can then be concluded by combining the orders of $\text{I}_1$ and $\text{I}_2$. 
\end{proof}

\begin{lemma}
\label{convex_property}
Given $\sup\{\omega_k\}_{k=1}^{\infty}\leq 1$, there exists a constant $G=4 Q^2(1+Q^2)$ such that
\begin{equation} \label{bound2}
\| \widetilde H(\bm{\theta}_k, \bm{\xeta}_{k+1})+\omega_{k+1}\rho(\btheta_k, \bx_{k+1})\|^2 \leq G (1+\|\bm{\theta}_k-\btheta_{\star}\|^2). 
\end{equation}
\end{lemma}
\begin{proof}

According to the compactness condition in Assumption \ref{ass2a}, we have
\begin{equation}
\label{mid_1}
\|\widetilde H(\bm{\theta}_k, \bm{\xeta}_{k+1})\|^2\leq Q^2 (1+\|\bm{\theta}_k\|^2) = 
 Q^2 (1+\|\bm{\theta}_k-\btheta_{\star}+\btheta_{\star}\|^2)\leq Q^2 (1+2\|\bm{\theta}_k-\btheta_{\star}\|^2+2Q^2).
\end{equation}

Therefore, using (\ref{mid_1}), we can show that for a constant $G=4Q^2(1+Q^2)$
\begin{equation*}
\small
\begin{split}
    &\ \ \ \| \widetilde H(\bm{\theta}_k, \bm{\xeta}_{k+1})+\omega_{k+1}\rho(\btheta_k, \bx_{k+1})\|^2 \\
    &\leq 2\|\widetilde H(\bm{\theta}_k, \bm{\xeta}_{k+1})\|^2 + 2\omega_{k+1}^2 \|\rho(\btheta_k, \bx_{k+1})\|^2\\
    &\leq 2Q^2 (1+2\|\bm{\theta}_k-\btheta_{\star}\|^2+2Q^2) + 2Q^2\\
    &\leq 2Q^2 (2+2Q^2+(2+2Q^2)\|\bm{\theta}_k-\btheta_{\star}\|^2)\\
    &\leq G (1+\|\bm{\theta}_k-\btheta_{\star}\|^2).
\end{split}
\end{equation*}
\end{proof}

\begin{lemma}
\label{theta_lip}Given $\sup\{\omega_k\}_{k=1}^{\infty}\leq 1$, we have that
\begin{equation}
\label{lip_theta}
    \|\btheta_{k}-\btheta_{k-1}\|\leq 2\omega_{k} Q
\end{equation}
\end{lemma}

\begin{proof}
Following the update $\btheta_k-\btheta_{k-1}=\omega_k \widetilde H(\bm{\theta}_{k-1}, \bm{x}_{k})+\omega_{k}^2 \rho(\btheta_{k-1}, \bx_{k})$, we have that
$$\|\btheta_{k}-\btheta_{k-1}\|= \|\omega_k \widetilde H(\bm{\theta}_{k-1}, \bm{x}_{k})+\omega_{k}^2 \rho(\btheta_{k-1}, \bx_{k})\|\leq \omega_k\| \widetilde H(\bm{\theta}_{k-1},\bm{x}_{k})\|+\omega_{k}^2\| \rho(\btheta_{k-1}, \bx_{k})\|.$$
By the compactness condition in Assumption \ref{ass2a} and $\sup\{\omega_k\}_{k=1}^{\infty}\leq 1$, (\ref{lip_theta}) can be derived.
\end{proof}

\begin{lemma}
\label{lemma:4}
There exist constants $\lambda_0$ and $k_0$ such that $\forall \lambda\geq\lambda_0$ and $\forall k> k_0$, the sequence $\{\psi_{k}\}_{k=1}^{\infty}$, where $\psi_{k}=\lambda\omega_{k}+\frac{1}{\phi} \sup_{i\geq k_0}\Delta_i$, satisfies
\begin{equation}
\begin{split}
\label{key_ieq}
\psi_{k+1}\geq& (1-2\omega_{k+1}\phi+G\omega_{k+1}^2)\psi_{k}+C_0\omega_{k+1}^2  +2\Delta_k\omega_{k+1}.
\end{split}
\end{equation}
\begin{proof}
By replacing $\psi_{k}$ with $\lambda\omega_{k}+\frac{1}{\phi} \sup_{i\geq k_0}\Delta_i$ in ($\ref{key_ieq}$), it suffices to show
\begin{equation*}
\small
\begin{split}
\label{lemma:loss_control}
\lambda \omega_{k+1}+\frac{1}{\phi} \sup_{i\geq k_0}\Delta_i\geq& (1-2\omega_{k+1}\phi+G\omega_{k+1}^2)\left(\lambda \omega_{k}+\frac{1}{\phi} \sup_{i\geq k_0}\Delta_i\right)+C_0\omega_{k+1}^2 + 2\Delta_k\omega_{k+1}.
\end{split}
\end{equation*}

which is equivalent to proving
\begin{equation*}
\small
\begin{split}
&\lambda (\omega_{k+1}-\omega_k+2\omega_k\omega_{k+1}\phi-G\omega_k\omega_{k+1}^2)\geq  \frac{1}{\phi}\sup_{i\geq k_0}\Delta_i(-2\omega_{k+1}\phi+G\omega_{k+1}^2 )+C_0\omega_{k+1}^2+ 2\Delta_k\omega_{k+1}.
\end{split}
\end{equation*}

Given the step size condition in ($\ref{a1}$), we have $$\small{\omega_{k+1}-\omega_{k}+2 \omega_{k}\omega_{k+1}\phi \geq C_1 \omega_{k+1}^2},$$ 
where $\small{C_1=\lim \inf 2\phi  \dfrac{\omega_{k}}{\omega_{k+1}}+\dfrac{\omega_{k+1}-\omega_{k}}{\omega^2_{k+1}}>0}$. Combining $-\sup_{i\geq k_0}\Delta_i\leq \Delta_k$, it suffices to prove
\begin{equation}
\begin{split}
\label{loss_control-2}
\lambda \left(C_1-G\omega_{k}\right)\omega^2_{k+1}\geq  \left(\frac{G}{\phi} \sup_{i\geq k_0}\Delta_i+C_0\right)\omega^2_{k+1}.
\end{split}
\end{equation}

It is clear that for a large enough $k_0$ and $\lambda_0$ such that $\omega_{k_0}\leq \frac{C_1}{2G}$, $\lambda_0=\frac{2G\sup_{i\geq k_0} \Delta_i + 2C_0\phi}{C_1\phi}$, the desired conclusion ($\ref{loss_control-2}$) holds for all such $k\geq k_0$ and $\lambda\geq \lambda_0$.
\end{proof}
\end{lemma}

The following lemma is a restatement of Lemma 25 (page 247) from \citet{Albert90}.
\begin{lemma}
\label{lemma:2}
Suppose $k_0$ is an integer satisfying
$\inf_{k> k_0} \dfrac{\omega_{k+1}-\omega_{k}}{\omega_{k}\omega_{k+1}}+2\phi-G\omega_{k+1}>0$ 
for some constant $G$. 
Then for any $k>k_0$, the sequence $\{\Lambda_k^K\}_{k=k_0, \ldots, K}$ defined below is increasing and uppered bounded by $2\omega_{k}$
\begin{equation}  
\Lambda_k^K=\left\{  
             \begin{array}{lr}  
             2\omega_{k}\prod_{j=k}^{K-1}(1-2\omega_{j+1}\phi+G\omega_{j+1}^2) & \text{if $k<K$},   \\  
              & \\
             2\omega_{k} &  \text{if $k=K$}.
             \end{array}  
\right.  
\end{equation} 
\end{lemma}

\begin{lemma}
\label{lemma:3-all}
Let $\{\psi_{k}\}_{k> k_0}$ be a series that satisfies the following inequality for all $k> k_0$
\begin{equation}
\begin{split}
\label{lemma:3-a}
\psi_{k+1}\geq &\psi_{k}\left(1-2\omega_{k+1}\phi+G\omega^2_{k+1}\right)+C_0\omega^2_{k+1} + 2 \Delta_k\omega_{k+1},
\end{split}
\end{equation}
and assume there exists such $k_0$ that 
\begin{equation}
\begin{split}
\label{lemma:3-b}
\E\left[\|\bm{T}_{k_0}\|^2\right]\leq \psi_{k_0}.
\end{split}
\end{equation}
Then for all $k> k_0$, we have
\begin{equation}
\begin{split}
\label{result}
\E\left[\|\bm{T}_{k}\|^2\right]\leq \psi_{k}+\sum_{j=k_0+1}^{k}\Lambda_j^k (z_{j-1}-z_j).
\end{split}
\end{equation}
\end{lemma}

\begin{proof}
We prove by the induction method. Assuming (\ref{result}) is true and applying (\ref{key_eqn}), we have that 
\begin{equation*}
\small
\begin{split}
    \E\left[\|\bm{T}_{k+1}\|^2\right]&\leq (1-2\omega_{k+1}\phi+\omega^2_{k+1} G)(\psi_{k}+\sum_{j=k_0+1}^{k}\Lambda_j^k (z_{j-1}-z_j))\\
    &\ \ \ \ \ \ \ \ +C_0\omega^2_{k+1} +2 \Delta_k\omega_{k+1}+2\omega_{k+1}\E[z_{k}-z_{k+1}]\\
\end{split}
\end{equation*}

Combining (\ref{key_ieq}) and Lemma.\ref{lemma:2}, respectively, we have
\begin{equation*}
\small
\begin{split}
    \E\left[\|\bm{T}_{k+1}\|^2\right]&\leq  \psi_{k+1}+(1-2\omega_{k+1}\phi+\omega^2_{k+1} G)\sum_{j=k_0+1}^{k}\Lambda_j^k (z_{j-1}-z_j)+2\omega_{k+1}\E[z_{k}-z_{k+1}]\\
    & \leq \psi_{k+1}+\sum_{j=k_0+1}^{k}\Lambda_j^{k+1} (z_{j-1}-z_j)+\Lambda_{k+1}^{k+1}\E[z_{k}-z_{k+1}]\\
    & \leq \psi_{k+1}+\sum_{j=k_0+1}^{k+1}\Lambda_j^{k+1} (z_{j-1}-z_j).\\
\end{split}
\end{equation*}
\end{proof}

\section{Ergodicity and dynamic importance sampler}

\label{ergodicity}
Our interest is to analyze the deviation between the weighted averaging estimator $\frac{1}{k}\sum_{i=1}^k\theta_{i}^{\zeta}( J_{\widetilde U}(\bx_i)) f(\bx_i)$ and posterior expectation $\int_{\MX}f(\bx)\pi(d\bx)$ for a \textcolor{black}{bounded} function $f$. To accomplish this analysis, we first study the convergence of the 
posterior sample mean $\frac{1}{k}\sum_{i=1}^k f(\bx_i)$ 
to the posterior expectation $\bar{f}=\int_{\MX}f(\bx)\varpi_{\Psi_{\btheta_{\star}}}(\bx)(d\bx)$ and then extend it to $\int_{\MX}f(\bx)\varpi_{\widetilde{\Psi}_{\btheta_{\star}}}(\bx)(d\bx)$. The key tool for establishing the ergodic theory is still the Poisson equation 
which is used to characterize the fluctuation 
between $f(\bx)$ and $\bar f$: 
\begin{equation}
    \mathcal{L}g(\bx)=f(\bx)-\bar f,
\end{equation}
where $g(\bx)$ is the solution of the Poisson equation, and $\mathcal{L}$ is the infinitesimal generator of the Langevin diffusion 
\begin{equation*}
    \mathcal{L}g:=\langle\nabla g, \nabla L(\cdot, \btheta_{\star})\rangle+\textcolor{black}{\tau}\nabla^2g.
\end{equation*}

To control the perturbations of $\frac{1}{k}\sum_{i=1}^k f(\bx_i)-\bar f$ and enables convergence of the weighted averaging estimate, \citet{Chen15} impose the following regularity conditions on the solution $g(\bx)$:
\paragraph{Regularity} Given a sufficiently smooth function $g(\bx)$ and a function $\mathcal{V}(\bx)$ such that $\|D^k g\|\lesssim \mathcal{V}^{p_k}(\bx)$ for some 
constants $p_k>0$, where $k\in\{0,1,2,3\}$. In addition, $\mathcal{V}^p$ has a bounded expectation, i.e., $\sup_{\bx} \E[\mathcal{V}^p(\bx)]<\infty$; and $\mathcal{V}$ is smooth, i.e. $\sup_{s\in\{0, 1\}} \mathcal{V}^p(s\bx+(1-s)\by)\lesssim \mathcal{V}^p(\bx)+\mathcal{V}^p(\by)$ for all $\bx,\by\in\MX$ and $p\leq 2\max_k\{p_k\}$. 

For stronger but verifiable conditions, we refer readers to \citet{VollmerZW2016}. For further verifications, \citet{Mackey18} showed that the 0th, 1st and 2nd order of the regularity can be verified given standard smoothness and dissipative assumptions. In what follows, we present a lemma, which is majorly adapted from Theorem 2 of \citet{Chen15} with a fixed learning rate $\epsilon$. 

%\begin{assump}[Regularity condition on the target distribution]  \label{ass7} The target distribution $\pi(\bx)$ has finite second moments, i.e., $\int_{\MX} \|\bx\|^2 \pi(\bx) d\bx < \infty$. 
%\end{assump}

\begin{lemma}[Convergence of the Averaging Estimators, restatement of Lemma \ref{avg_converge}]
\label{avg_converge_appendix}
Suppose Assumptions \ref{ass2a}-\ref{ass1} hold.
%and that the full data is used in determining the indexes of subregions.  
For any bounded function $f$, 
\begin{equation*}
\small
\begin{split}
    \left|\E\left[\frac{\sum_{i=1}^k f(\bx_i)}{k}\right]-\int_{\MX}f(\bx)\varpi_{\widetilde{\Psi}_{\btheta_{\star}}}(\bx)d\bx\right|&=
    \mathcal{O}\left(\frac{1}{k\epsilon}+\sqrt{\epsilon}+\sqrt{\frac{\sum_{i=1}^k \omega_i}{k}}+ \frac{1}{\sqrt{m}}+
    \sqrt{\sup_{\bx} \Var(\xi_n(\bx))}\right), \\
\end{split}
\end{equation*}
where $k_0$ is a sufficiently large constant, $\varpi_{\widetilde{\Psi}_{\btheta_{\star}}}(\bx) \propto  
\frac{\pi(\bx)}{\theta_{\star}^{\zeta}(J(\bx))}$, and $\frac{\sum_{i=1}^k \omega_i}{k} =o(\frac{1}{\sqrt{k}})$ as implied by Assumption \ref{ass1}. 
%$Z_{\btheta_{\star}}=\sum_{i=1}^m \frac{\int_{\MX_i} \pi(\bx)d\bx}{\theta_{\star}^{\zeta}(i)}$}.
\end{lemma}

\begin{proof}
% In the case that the full data is used in determining the indexes of subregions, the CSGLD
% algorithm can be written as follows:
%can view the adaptive algorithm as a standard sampling algorithm with fixed latent variable $\btheta_{\star}$ as follows
We rewrite the CSGLD algorithm as follows:
\begin{equation*}
\begin{split}
    \bm{x}_{k+1}&=\bx_k- \epsilon_k\nabla_{\bx} \widetilde{L}(\bx_k, \btheta_k)+\mathcal{N}({0, 2\epsilon_k \tau\bm{I}})\\
    &=\bx_k- \epsilon_k\left(\nabla_{\bx} 
    \widehat{L}(\bx_k, \btheta_{\star})+{\Upsilon}(\bx_k, \btheta_k, \btheta_{\star})\right)+\mathcal{N}({0, 2\epsilon_k \tau\bm{I}}),
\end{split}
\end{equation*}
where    
$\nabla_{\bx} \widehat{L}(\bx,\btheta)= \frac{N}{n} \left[1+  \frac{\zeta\tau}{\Delta u}  \left(\textcolor{black}{\log \theta({J}(\bx))-\log\theta(({J}(\bx)-1)\vee 1)} \right) \right]  \nabla_{\bx} \widetilde U(\bx)$,  $\nabla_{\bx} \widetilde{L}(\bx,\btheta)$ is as 
defined in Section \ref{Alg:app},
 and the bias term is given by ${\Upsilon}(\bx_k,\btheta_k,\btheta_{\star})=\nabla_{\bx} \widetilde{L}(\bx_k,\btheta_k)-\nabla_{\bx} \widehat{L}(\bx_k,\btheta_{\star})$.
%and is an 
%unbiased estimator of 
%$\nabla_{\bx} {L}(\bx_k,\btheta_k)-\nabla_{\bx} {L}(\bx_k,\btheta_{\star})$.

%By the boundedness of $\Theta$ and   $\inf_{\Theta}\theta(i)>0$, 
% By Assumption \ref{ass4}, Lemma \ref{lemma:1},
% and Eve's law (i.e., the variance decomposition formula),
% it is easy to derive that  $\E\|\nabla_{\bx} \tilde{U}(\bx)\|^2<\infty$. 
% Then, by the triangle inequality and Jensen's inequality, we have

By Assumption \ref{ass2}, we have $\|\nabla_{\bx} U(\bx)\|=\|\nabla_{\bx} U(\bx)-\nabla_{\bx} U(\bx_{\star})\|\lesssim \|\bx-\bx_{\star}\|\leq \|\bx\|+\|\bx_{\star}\|$ for some optimum. Then the $L^2$ upper bound in Lemma \ref{lemma:1} implies that $\nabla_{\bx} U(\bx)$ has a bounded second moment. Combining Assumption \ref{ass4}, we have $\E\left[\|\nabla_{\bx} \widetilde U(\bx)\|^2\right]<\infty$. Further by Eve’s law (i.e., the
variance decomposition formula), it is easy to derive that $\E \left[\| \nabla_{\bx} \widetilde{U}(\bx) \|\right]<\infty$.
%[The previous one is not that straightforward, so I filled in the gaps.]} 
Then, by the triangle inequality and Jensen's inequality, 
\begin{equation}
\label{latent_bias}
\small
\begin{split}
    \|\E[\Upsilon(\bx_k,\btheta_k,\btheta_{\star})]\|&\leq 
    \E[\|\nabla_{\bx} \widetilde{L}(\bx_k, \btheta_k)-\nabla_{\bx} \widetilde{L}(\bx_k, \btheta_{\star})\|] + \E[\|\nabla_{\bx} \widetilde{L}(\bx_k, \btheta_{\star})-\nabla_{\bx} \widehat{L}(\bx_k, \btheta_{\star})\|] \\
    &\lesssim  \E[\|\btheta_k-\btheta_{\star}\|]+\mathcal{O}(\sup_{\bx} \Var(\xi_n(\bx)))\leq \sqrt{\E[\|\btheta_k-\btheta_{\star}\|^2]}+\mathcal{O}(\sup_{\bx} \Var(\xi_n(\bx)))\\
    &\leq \mathcal{O}\left( \sqrt{\omega_{k}+\epsilon+
\frac{1}{m} +\sup_{\bx} \Var(\xi_n(\bx)))}\right),
\end{split}
\end{equation}
where Assumption \ref{ass2a} and Theorem \ref{latent_convergence} are used to derive the smoothness of $\nabla_{\bx} \tilde{L}(\bx, \btheta)$ with respect to $\btheta$, and  $\sup_{\bx} \Var(\xi_n(\bx))$ is the bias caused by the mini-batch evaluation of $U(\bx)$.

The ergodic average based on biased gradients and a fixed learning rate is a direct result of Theorem 2 of  \citet{Chen15}. By simulating from $\varpi_{\Psi_{\btheta_{\star}}}(\bx)\propto\frac{\pi(\bx)}{\Psi^{\zeta}_{\btheta_{\star}}(U(\bx))}$ and combining (\ref{latent_bias}) and Theorem \ref{latent_convergence}, we have 
\begin{equation*}
\small
\begin{split}
    \left|\E\left[\frac{\sum_{i=1}^k f(\bx_i)}{k}\right]-\int_{\MX}f(\bx) \varpi_{\Psi_{\btheta_{\star}}}(\bx)d\bx\right|&\leq \mathcal{O}\left(\frac{1}{k\epsilon}+\epsilon+\frac{\sum_{i=1}^k \|\E[\Upsilon(\bx_k,\btheta_k,\btheta_{\star})]\|}{k}\right)\\
    &\lesssim \mathcal{O}\left(\frac{1}{k\epsilon}+\epsilon+\frac{\sum_{i=1}^k \sqrt{\omega_i+\epsilon+\frac{1}{m}+\sup_{\bx} \Var(\xi_n(\bx))}}{k}\right) \\
    &\leq \mathcal{O}\left(\frac{1}{k\epsilon}+\sqrt{\epsilon}+\sqrt{\frac{\sum_{i=1}^k \omega_i}{k}}+
    \frac{1}{\sqrt{m}}+\sqrt{\sup_{\bx} \Var(\xi_n(\bx))}\right),
\end{split}
\end{equation*}
where the last inequality follows by repeatedly applying the inequality $\sqrt{a+b}\leq \sqrt{a}+\sqrt{b}$ and 
the inequality $\sum_{i=1}^k \sqrt{\omega_i}\leq \sqrt{k\sum_{i=1}^k \omega_i}$.

For any a bounded function $f(\bx)$, we have $|\int_{\MX}f(\bx) \varpi_{\Psi_{\btheta_{\star}}}(\bx)d\bx -  \int_{\MX}f(\bx) \varpi_{\widetilde{\Psi}_{\btheta_{\star}}}(\bx)d\bx|= \mathcal{O}(\frac{1}{m})$ by Lemma \ref{partition_order}. By the triangle inequality, we have 
\begin{equation*}
\small
\begin{split}
    \left|\E\left[\frac{\sum_{i=1}^k f(\bx_i)}{k}\right]-\int_{\MX}f(\bx) \varpi_{\widetilde{\Psi}_{\btheta_{\star}}}(\bx)d\bx\right|\leq  \mathcal{O}\left(\frac{1}{k\epsilon}+\sqrt{\epsilon}+\sqrt{\frac{\sum_{i=1}^k \omega_i}{k}}+ \frac{1}{\sqrt{m}}+
    \sqrt{\sup_{\bx} \Var(\xi_n(\bx))}\right),
\end{split}
\end{equation*}
which concludes the proof. 
\end{proof}

Finally, we are ready to show the convergence of the weighted averaging estimator $\frac{\sum_{i=1}^k\theta_{i}
     ^{\zeta}(J_{\widetilde U}(\bx_i)) f(\bx_i)}{\sum_{i=1}^k\theta_{i}^{\zeta}( 
      J_{\widetilde U}(\bx_i))}$ to the posterior mean $\int_{\MX}f(\bx)\pi(d\bx)$.
\begin{theorem}[Convergence of the Weighted Averaging Estimators, restatement of Theorem \ref{w_avg_converge}] Assume Assumptions \ref{ass2a}-\ref{ass1} hold. For any bounded function $f$, we have that 
\label{w_avg_converge_appendix}
\begin{equation*}
\small
\begin{split}
    \left|\E\left[\frac{\sum_{i=1}^k\theta_{i}
     ^{\zeta}(J_{\widetilde U}(\bx_i)) f(\bx_i)}{\sum_{i=1}^k\theta_{i}^{\zeta}( 
      J_{\widetilde U}(\bx_i))}\right]-\int_{\MX}f(\bx)\pi(d\bx)\right|&= \mathcal{O}\left(\frac{1}{k\epsilon}+\sqrt{\epsilon}+\sqrt{\frac{\sum_{i=1}^k \omega_i}{k}}+\frac{1}{\sqrt{m}}+\sqrt{\sup_{\bx} \Var(\xi_n(\bx))}\right). \\
\end{split}
\end{equation*}
\end{theorem}

\begin{proof}

Applying triangle inequality and $|\E[x]|\leq \E[|x|]$, we have
\begin{equation*}
\footnotesize
    \begin{split}
        &\left|\E\left[\frac{\sum_{i=1}^k\theta_{i}
        ^{\zeta}( J_{\widetilde U}(\bx_i)) f(\bx_i)}{\sum_{i=1}^k\theta_{i}^{\zeta}(
         J_{\widetilde U}(\bx_i))}\right]-\int_{\MX}f(\bx)\pi(d\bx)\right|\\
        \leq &\underbrace{\E\left[\left|\frac{\sum_{i=1}^k\theta_{i}^{\zeta}
         (J_{\widetilde U}(\bx_i))f(\bx_i)}
         {\sum_{i=1}^k\theta_{i}^{\zeta}(J_{\widetilde U}(\bx_i)) }-\frac{\sum_{i=1}^k\theta_{i}^{\zeta}
         ({J}(\bx_i))f(\bx_i)}
         {\sum_{i=1}^k\theta_{i}^{\zeta}({J}(\bx_i)) }\right|\right]}_{\text{I}_1}\\
         &\ \ +\underbrace{\E\left[\left|\frac{\sum_{i=1}^k\theta_{i}^{\zeta}
         ({J}(\bx_i))f(\bx_i)}
         {\sum_{i=1}^k\theta_{i}^{\zeta}({J}(\bx_i)) }-\frac{Z_{\btheta_{\star}}\sum_{i=1}^k\theta_{i}^{\zeta} ({J}(\bx_i)) f(\bx_i)}{k}\right|\right]}_{\text{I}_2}\\
        \ \ + &\underbrace{\E\left[\frac{Z_{\btheta_{\star}}}{k}\sum_{i=1}^k\left|\theta_i^{\zeta} ({J}(\bx_i))-\theta_{\star}^{\zeta}
      ({J}(\bx_i))  \right| \cdot |f(\bx_i)|\right]}_{\text{I}_3} +\underbrace{\left|\E\left[\frac{Z_{\btheta_{\star}}}{k}\sum_{i=1}^k\theta_{\star}^{\zeta}
     ({J}(\bx_i)) f(\bx_i)\right]-\int_{\MX}f(\bx)\pi(d\bx)\right|}_{\text{I}_4}.
    \end{split}
\end{equation*}

For the term $\text{I}_1$, consider the bias $\E[\widetilde H(\btheta, \bx)- H(\btheta, \bx)]=O\left(\sup_{\bx} \Var(\xi_n(\bx))\right)$ as defined in the proof of Lemma 
\ref{convex_appendix__}. By applying mean-value theorem, we have 
\begin{equation}
\footnotesize
\begin{split}
    \text{I}_1&=\E\left[\left|\frac{\left(\sum_{i=1}^k\theta_{i}^{\zeta}(
         J_{\widetilde U}(\bx_i))f(\bx_i)\right)\left(\sum_{i=1}^k\theta_{i}^{\zeta}(
         {J}(\bx_i))\right)-\left(\sum_{i=1}^k\theta_{i}^{\zeta}(
         {J}(\bx_i))f(\bx_i)\right)\left(\sum_{i=1}^k\theta_{i}^{\zeta}(
         J_{\widetilde U}(\bx_i))\right)}
         {\left(\sum_{i=1}^k\theta_{i}^{\zeta}(
         J_{\widetilde U}(\bx_i))\right)\left(\sum_{i=1}^k\theta_{i}^{\zeta}(
         {J}(\bx_i))\right)}\right|\right]\\
         %%%%%
         &\lesssim O\left(\sup_{\bx} \Var(\xi_n(\bx))\right) \E\left[\frac{\left(\sum_{i=1}^k\theta_{i}
     ^{\zeta}({J}(\bx_i)) f(\bx_i) \left(\sum_{i=1}^k\theta_{i}^{\zeta}(
         {J}(\bx_i))\right)\right)}{\left(\sum_{i=1}^k\theta_{i}^{\zeta}(
         {J}(\bx_i))\right)\left(\sum_{i=1}^k\theta_{i}^{\zeta}(
         {J}(\bx_i))\right)}\right]
         =O\left(\sup_{\bx} \Var(\xi_n(\bx))\right).
\end{split}
\end{equation}

For the term $\text{I}_2$, 
by the boundedness of $\bTheta$ and $f$ and the assumption  $\inf_{\Theta}\theta^{\zeta}(i)>0$, we have
\begin{equation*}
\small
\begin{split}
    \text{I}_2=&\E\left[\left|\frac{\sum_{i=1}^k\theta_{i}^{\zeta}({J}(\bx_i))  f(\bx_i)}{\sum_{i=1}^k\theta_{i}^{\zeta}
    ({J}(\bx_i))
    }\left(1-\sum_{i=1}^k\frac{\theta_i^{\zeta}
    ({J}(\bx_i))
    }{k}Z_{\btheta_{\star}}\right)\right|\right]\\
    %%%%%%
    \lesssim & \E\left[\left|Z_{\btheta_{\star}}\frac{{\sum_{i=1}^k\theta_{i}^{\zeta}
    ({J}(\bx_i))
    }}{k}-1\right|\right]\\
    =&\E\left[\left|Z_{\btheta_{\star}}\sum_{i=1}^m \frac{\sum_{j=1}^k\left( \theta_j^{\zeta}(i)-\theta_{\star}^{\zeta}(i)+\theta_{\star}^{\zeta}(i)\right)1_{
    {J}(\bx_j)=i}}{k}-1\right|\right]\\
    %%%%%%%%%%%
    \leq & \underbrace{\E\left[Z_{\btheta_{\star}}\sum_{i=1}^m \frac{\sum_{j=1}^k\left| \theta_j^{\zeta}(i)-\theta_{\star}^{\zeta}(i)\right| 1_{{J}(\bx_j)=i}}{k} \right]}_{\text{I}_{21}} + \underbrace{\E\left[\left| Z_{\btheta_{\star}}\sum_{i=1}^m \frac{\theta_{\star}^{\zeta}(i)\sum_{j=1}^k  1_{{J}(\bx_j)=i}}{k}-1\right|
    \right]}_{\text{I}_{22}}.\\
\end{split}
\end{equation*}

For $\text{I}_{21}$, by first applying the inequality $|x^{\zeta}-y^{\zeta}|\leq \zeta |x-y| z^{\zeta-1}$ for any $\zeta>0, x\leq y$ and $z\in[x, y]$ based on the mean-value theorem and then applying the Cauchy–Schwarz inequality, we have 
\begin{equation}\label{ii_21}
    \text{I}_{21}\lesssim \frac{1}{k}\E\left[ \sum_{j=1}^k\sum_{i=1}^m\left| \theta_j^{\zeta}(i)-\theta_{\star}^{\zeta}(i)\right| \right]\lesssim  \frac{1}{k}\E\left[ \sum_{j=1}^k\sum_{i=1}^m\left| \theta_j(i)-\theta_{\star}(i)\right| \right]\lesssim  \frac{1}{k}\sqrt{\sum_{j=1}^k\E\left[\left\| \btheta_j-\btheta_{\star}\right\|^2\right]},
\end{equation}
where the compactness of $\Theta$ has been 
used in deriving the second inequality. 

For $\text{I}_{22}$, considering the following relation $$
    1=\sum_{i=1}^m\int_{\MX_i} \pi(\bx)d\bx=\sum_{i=1}^m\int_{\MX_i} \theta_{\star}^{\zeta}(i) \frac{\pi(\bx)}{\theta_{\star}^{\zeta}(i)}d\bx
    =Z_{\btheta_{\star}}\int_{\MX} \sum_{i=1}^m \theta_{\star}^{\zeta}(i) 1_{{J}(\bx)=i}\varpi_{\widetilde{\Psi}_{\btheta_{\star}}}(\bx)d\bx,$$ then we have
% \begin{equation}
% \begin{split}
%     \text{I}_{22}&=\E\left[\left|Z_{\btheta_{\star}}\sum_{i=1}^m \theta_{\star}^{\zeta}(i)\left(\frac{\sum_{j=1}^k 1_{
%     {J}(\bx_j)=i}}{k}-\int_{\bX}1_{
%     {J}(\bx)=i}\varpi_{\widetilde{\Psi}_{\btheta_{\star}}}(\bx)d\bx\right)\right|\right]\\
%     %%%%%
%     &\lesssim \sum_{i=1}^m\E\left[\left|\frac{\sum_{j=1}^k 1_{ {J}(\bx_j)=i}
%     }{k}-\int_{\bX} 1_{{J}(\bx)=i }\varpi_{\widetilde{\Psi}_{\btheta_{\star}}}(\bx)d\bx\right|\right]\\
%     &= \mathcal{O}\left(\frac{1}{k\epsilon}+\epsilon^{0.5}+\sqrt{\frac{\sum_{i=1}^k \omega_k}{k}}+\frac{1}{m^{0.5}}\right),
% \end{split}
% \end{equation}
\begin{equation}
\begin{split}
    \text{I}_{22}&=\E\left[\left| Z_{\btheta_{\star}}\sum_{i=1}^m \frac{\theta_{\star}^{\zeta}(i)\sum_{j=1}^k  1_{{J}(\bx_j)=i}}{k}-Z_{\btheta_{\star}}
    \int_{\MX} \sum_{i=1}^m \theta_{\star}^{\zeta}(i) 1_{{J}(\bx)=i}\varpi_{\widetilde{\Psi}_{\btheta_{\star}}}(\bx)d\bx\right|\right]\\
    &=Z_{\btheta_{\star}} \E\left[\left| \frac{1}{k}\sum_{j=1}^k \left(\sum_{i=1}^m\theta_{\star}^{\zeta}(i)  1_{{J}(\bx_j)=i}\right)-\int_{\MX} \left(\sum_{i=1}^m \theta_{\star}^{\zeta}(i) 1_{{J}(\bx)=i}\right)\varpi_{\widetilde{\Psi}_{\btheta_{\star}}}(\bx)d\bx\right|\right]\\
    &= \mathcal{O}\left(\frac{1}{k\epsilon}+\sqrt{\epsilon}+\sqrt{\frac{\sum_{i=1}^k \omega_i}{k}}+\frac{1}{\sqrt{m}}+  \sqrt{\sup_{\bx} \Var(\xi_n(\bx))} \right),
\end{split}
\end{equation}
where the last equality follows  from Lemma \ref{avg_converge_appendix} as the \textcolor{black}{step 
function $\sum_{i=1}^m \theta_{\star}^{\zeta}(i) 1_{{J}(\bx)=i}$} is bounded.

For $\text{I}_3$, by the boundedness of $f$,  
the mean value theorem and Cauchy-Schwarz inequality, 
we have 
\begin{equation}\label{ii_3}
\small
    \begin{split}
        \text{I}_3&\lesssim \E\left[\frac{1}{k}\sum_{i=1}^k\left|\theta_{i}
        ^{\zeta}({J}(\bx_i)) -\theta_{\star}^{\zeta}(
        {J}(\bx_i))\right|\right]\lesssim  \frac{1}{k}\E\biggl[ \sum_{j=1}^k\sum_{i=1}^m\bigl| \theta_j(i)-\theta_{\star}(i)\bigr| \biggr]\lesssim  \frac{1}{k}\sqrt{\sum_{j=1}^k\E\left[\left\| \btheta_j-\btheta_{\star}\right\|^2\right]}.\\
    \end{split}
\end{equation}

For the last term $\text{I}_4$, we first decompose $\int_{\MX} f(\bx) \pi(d\bx)$ into $m$ disjoint regions to facilitate the analysis
\begin{equation}
\label{split_posterior}
\footnotesize
\begin{split}
      \int_{\MX} f(\bx) \pi(d\bx)&=\int_{\cup_{j=1}^m \MX_j}  f(\bx) \pi(d\bx)=\sum_{j=1}^m\int_{\MX_j}\theta_{\star}^{\zeta}(j)  f(\bx) \frac{\pi(d\bx)}{\theta_{\star}^{\zeta}(j)}\\
      &=Z_{\btheta_{\star}}\int_{\MX} \sum_{j=1}^m \theta_{\star}(j)^{\zeta}f(\bx) 1_{
        {J}(\bx_i)=j 
        }\varpi_{\widetilde{\Psi}_{\btheta_{\star}}}(\bx)(d\bx).\\
\end{split}
\end{equation}

Plugging (\ref{split_posterior}) into the last term $\text{I}_4$, we have
% \begin{equation}. %%%%%%%%% m elements of order 1/m doesn't mean order 1/m
% \label{final_i2}
% \small
%     \begin{split}
%         \text{I}_4&=\left|\E\left[\frac{Z_{\btheta_{\star}}}{k}\sum_{i=1}^k\sum_{j=1}^m\theta_{\star}(j)^{\zeta} f(\bx_i)1_{  {J}(\bx_i)=j
%         }\right]-\int_{\bX}f(\bx)\pi(d\bx)\right|\\
%         %%%%
%         &=Z_{\btheta_{\star}}\left|\sum_{j=1}^m\theta_{\star}^{\zeta}(j)\E\left[\frac{1}{k}\sum_{i=1}^k f(\bx_i)
%         1_{ {J}(\bx_i)=j 
%         }\right]-\sum_{j=1}^m\theta_{\star}^{\zeta}(j)\int_{\bX_j}  f(\bx) \varpi_{\widetilde{\Psi}_{\btheta_{\star}}}(\bx)(d\bx)\right|\\
%         %%%%%%
%         &\leq Z_{\btheta_{\star}}\sum_{j=1}^m\theta_{\star}^{\zeta}(j)\left|\E\left[\frac{1}{k}\sum_{i=1}^k f(\bx_i)1_{
%         {J}(\bx_i)=j 
%         }\right]-\int_{\bX_j}  f(\bx) \varpi_{\widetilde{\Psi}_{\btheta_{\star}}}(\bx)(d\bx)\right|\\
%     \end{split}
% \end{equation}
\begin{equation}
\label{final_i2}
\small
    \begin{split}
        \text{I}_4&=\left|\E\left[\frac{Z_{\btheta_{\star}}}{k}\sum_{i=1}^k\sum_{j=1}^m\theta_{\star}(j)^{\zeta} f(\bx_i)1_{  {J}(\bx_i)=j
        }\right]-\int_{\MX}f(\bx)\pi(d\bx)\right|\\
        % &=Z_{\btheta_{\star}}\left|\sum_{j=1}^m\theta_{\star}^{\zeta}(j)\E\left[\frac{1}{k}\sum_{i=1}^k f(\bx_i)
        % 1_{ {J}(\bx_i)=j 
        % }\right]-\sum_{j=1}^m\theta_{\star}^{\zeta}(j)\int_{\bX_j}  f(\bx) \varpi_{\widetilde{\Psi}_{\btheta_{\star}}}(\bx)(d\bx)\right|\\
        % &= Z_{\btheta_{\star}}\left|\sum_{j=1}^m\theta_{\star}^{\zeta}(j)\left(\E\left[\frac{1}{k}\sum_{i=1}^k f(\bx_i)1_{
        % {J}(\bx_i)=j 
        % }\right]-\int_{\bX_j}  f(\bx) \varpi_{\widetilde{\Psi}_{\btheta_{\star}}}(\bx)(d\bx)\right)\right|\\
        &= Z_{\btheta_{\star}}\left|\E\left[\frac{1}{k}\sum_{i=1}^k \left(\sum_{j=1}^m\theta_{\star}^{\zeta}(j) f(\bx_i)1_{
        {J}(\bx_i)=j 
        }\right)\right]-\int_{\MX}  \left(\sum_{j=1}^m\theta_{\star}^{\zeta}(j) f(\bx_i)1_{
        {J}(\bx_i)=j 
        }\right) \varpi_{\widetilde{\Psi}_{\btheta_{\star}}}(\bx)(d\bx)\right|\\
    \end{split}
\end{equation}

Applying the function \textcolor{black}{$\sum_{j=1}^m\theta_{\star}^{\zeta}(j) f(\bx_i)1_{
        {J}(\bx_i)=j 
        }$ }
to Lemma \ref{avg_converge_appendix} yields
\begin{equation}
\label{almost_i2}
\small
\begin{split}
      \left|\E\left[\frac{1}{k}\sum_{i=1}^k f(\bx_i)\right]-\int_{\MX}  f(\bx) \varpi_{\widetilde{\Psi}_{\btheta_{\star}}}(\bx)(d\bx)\right| = \mathcal{O}\left(\frac{1}{k\epsilon}+\sqrt{\epsilon}+\sqrt{\frac{\sum_{i=1}^k \omega_i}{k}}+\frac{1}{\sqrt{m}}+\sqrt{\sup_{\bx} \Var(\xi_n(\bx))} \right).\\
\end{split}
\end{equation}

Plugging (\ref{almost_i2}) into (\ref{final_i2}) and combining $\text{I}_{1}$, $\text{I}_{21}$, $\text{I}_{22}$, $\text{I}_3$ and Theorem \ref{latent_convergence}, we have
\begin{equation*}
\small
\begin{split}
      \left|\E\left[\frac{\sum_{i=1}^k\theta_{i}
     ^{\zeta}(J_{\widetilde U}(\bx_i)) f(\bx_i)}{\sum_{i=1}^k\theta_{i}^{\zeta}( 
      J_{\widetilde U}(\bx_i))}\right]-\int_{\MX}f(\bx)\pi(d\bx)\right| = \mathcal{O}\left(\frac{1}{k\epsilon}+\sqrt{\epsilon}+\sqrt{\frac{\sum_{i=1}^k \omega_i}{k}}+\frac{1}{\sqrt{m}}+\sqrt{\sup_{\bx} \Var(\xi_n(\bx))} \right),\\
\end{split}
\end{equation*}
which concludes the proof of the theorem.

\end{proof}

\section{More discussions on the algorithm}
\label{ext}

\subsection{An alternative numerical scheme}
\label{alternative_scheme}
In addition to the numerical scheme used in (6) and (8) in the main body, we can also consider the following numerical scheme 
  \begin{equation*} \label{alternative_SGLDeq6}
  \footnotesize
 \begin{split}
  \small{\bx_{k+1}=\bx_k - \epsilon_{k+1} \frac{N}{n} \left[1+ 
   \zeta\tau\frac{\log {\theta}_{k}\big(J_{\widetilde U}(\bx_k) \wedge m\big) - \log{\theta}_{k}\big(J_{\widetilde U}(\bx_k)\big)}{\Delta u}  \right]  
    \nabla_{\bx} \widetilde U(\bx_k) +\sqrt{2 \tau \epsilon_{k+1}} \bw_{k+1}}.
 \end{split}
  \end{equation*}

Such a scheme leads to a similar theoretical result and a better treatment of $\Psi_{\btheta}(\cdot)$ for the subregions that contains stationary points.

\subsection{Bizarre peaks in the Gaussian mixture distribution}

A bizarre peak always indicates that there is a stationary point of the same energy in somewhere of the sample space, as the sample space is partitioned according to the energy function in CSGLD. For example, we study a mixture distribution with asymmetric modes $\pi(x)=1/6 N(-6,1)+5/6 N(4,1)$. Figure \ref{bizzae} shows a bizarre peak at $x$. Although $x$ is not a local minimum, it has the same energy as ``-6'' which is a local minimum. Note that in CSGLD, $x$ and ``-6'' belongs to the same subregion.

\begin{figure}[ht]
\vspace{-0.05in}
  \centering
  \includegraphics[scale=0.27]{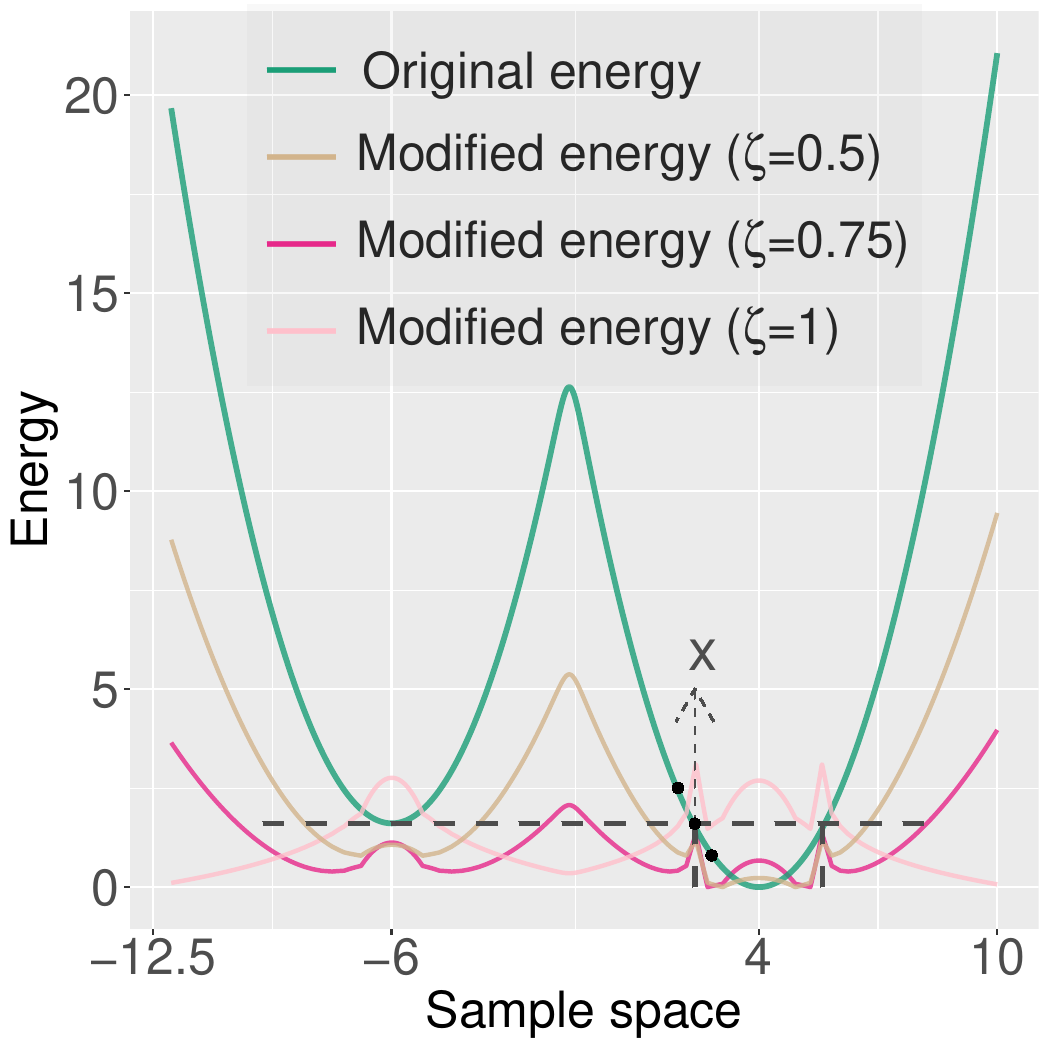}
  \caption{Explanation of bizarre peaks.}
  \label{bizzae}
  \vspace{-0.05in}
\end{figure}

\subsection{Simulations of multi-modal distributions}

We run all the algorithms with 200,000 iterations and assume the energy and gradient follow the Gaussian distribution with a variance of 0.1. We include an additional quadratic regularizer $(\|\bx\|^2-7)1_{\|\bx\|^2>7}$ to limit the samples to the center region. We use a constant learning rate 0.001 for SGLD, reSGLD, and CSGLD; We adopt the cyclic cosine learning rates with initial learning rate 0.005 and 20 cycles for cycSGLD. The temperature is fixed at 1 for all the algorithms, excluding the high-temperature process of reSGLD, which employs a temperature of 3.  In particular for CSGLD, we choose the step size $\omega_k=\min\{0.003, 10/(k^{0.8}+100)\}$ for learning the latent vector. We fix 100 partitions and each energy bandwidth is set to $0.25$. We choose $\zeta=0.75$.

\subsection{Extension to the scenarios with high-$\zeta$}
\label{more_scalable_target}

% In some complex experiments (e.g. computer vision) with a high-loss function, the fixed point $\btheta_{\star}$ is close to the vector $(1, 0, ..., 0)$, i.e., the first subregion may dominate the probability mass. As a result, estimating $\theta(i)$'s for the high energy subregions can be quite difficult due to the limitation of floating points. 

In some complex big data experiments, the first subregion may dominate the probability mass and estimating $\theta(i)$'s for the high energy subregions can be quite difficult due to the limitation of floating points. If  a small value of $\zeta$ is used, the gradient multiplier $1+ 
   \zeta\tau\frac{\log {\theta}_{\star}(i) - \log{\theta}_{\star}((i-1)\vee 1)}{\Delta u}$ is close to $1$ for any $i$ and the algorithm will perform similarly to SGLD, except with different weights. When a large value of $\zeta$ is used, the stochastic approximation update in Eq.(\ref{def_tilde_H}) is slow \textcolor{red}{since $\theta^{\zeta}$ is close to 0}. To tackle this issue, one solution is to include a high-order bias item in the stochastic approximation:
\begin{equation} \label{regularizer_cslgd}
{\theta}_{k+1}(i)={\theta}_{k}(i)+\omega_{k+1}\left({\theta}_{k}^{\textcolor{red}{\zeta}}(J_{\widetilde U}(\bx_{k+1})+\omega_{k+1} 1_{i\geq J_{\widetilde U}(\bx_{k+1})}\rho)\right)\left(1_{i= J_{\widetilde U}(\bx_{k+1})}-{\theta}_{k}(i)\right), 
\end{equation} 
for $i=1,2,\ldots,m$, where $\rho$ is a constant. As shown early, our convergence theory allows inclusion of such a high-order bias term.
In simulations, the high-order bias term $\omega_{k+1}^2 1_{i\geq J_{\widetilde U}(\bx_{k+1})}\rho$ penalized more on the higher energy regions, and thus accelerates the convergence % of $\btheta_k$  toward the pattern  $(1,0,0,\ldots,0)$ especially 
in the early period.

In the computer vision examples, we set the momentum coefficient  to 0.9 and the weight decay to $25$. For CSGHMC and saCSGHMC, we set $\omega_k=\frac{10}{k^{0.75}+1000}$ and $\rho=1$ in (\ref{regularizer_cslgd}) for both CIFAR10 and CIFAR100,  and set $\zeta=1\times 10^6$ for CIFAR10 and $3\times 10^6$ for CIFAR100. 

\paragraph{Can we make the algorithm more scalable:} The additional term $\rho$ makes the algorithm less appealing in practice. To tackle that issue and make it more scalable to big datasets, we can adopt a more scalable and elegant stochastic approximation scheme proposed in \citet{icsgld}
\begin{equation*} \label{novel_cslgd}
{\theta}_{k+1}(i)={\theta}_{k}(i)+\omega_{k+1} {\theta}_{k}(J_{\widetilde U}(\bx_{k+1})\left(1_{i= J_{\widetilde U}(\bx_{k+1})}-{\theta}_{k}(i)\right).
\end{equation*}

Compared to Eq.(\ref{regularizer_cslgd}), the dependence of $\rho$ is no longer required. Moreover, since $\theta(i)<1$ for any $i\in\{1,2,\cdots, m\}$, $\theta(i)^{\zeta}\approx 0$ when $\zeta\gg 1$, which makes Eq.(\ref{regularizer_cslgd}) hard to update. By contrast, the new scheme doesn't have this issue. Theoretically, a local stability shows that $\theta_k(i)$ converges to a much smoother estimate $(\int_{\MX_i}\pi(\bx)d\bx)^{\frac{1}{\zeta}}$ instead of the original $\int_{\MX_i}\pi(\bx)d\bx$ for $i\in\{1,2,\cdots, m\}$ and $\zeta>1$. Notably, the new target is proven to be easier to estimate for high-energy regions.

\subsection{Number of partitions}
A fine partition will lead to a smaller discretization error, but it may increase the risk in stability. In particular, it leads to large bouncy jumps around optima (a large negative learning rate, i.e., $\frac{\log\theta(2)-\log\theta(1)}{\Delta u}\ll 0$ in formula (8) may be caused there). Empirically, we suggest to partition the sample space into a moderate number of subregions, e.g. 10-1000, to balance between stability and discretization error.

\end{document}

\maketitle

In this supplementary material, we review the related methodologies in $\S$\ref{review}, show the convergence in $\S$\ref{convergence} and prove the technical lemmas in $\S$\ref{technique}.
\section{Background on Stochastic Approximation and Poisson Equation}
\label{review}

\subsection{Robbins–Monro algorithm}
Robbins–Monro algorithm \citep{Robbins51} aims to solve the root finding problem. Given a random field of $H(\bm{\btheta}, \bm{\bx})$ with respect to $\bm{\bx}$, our goal is to find the equilibrium $\btheta$ for the mean field function $h(\btheta)$ such that
\begin{equation*}
\begin{split}
\label{sa00}
h(\btheta)&=\int H(\bm{\theta}, \bm{\bx})\varpi_{\bm{\theta}}(d\bm{\bx})=0,
\end{split}
\end{equation*}
where $\bx\in \bX \subset \mathbb{R}^d$ and $\btheta\in\bTheta \subset \mathbb{R}^{d_{\btheta}}$.
% page 244 eq 1.10.3
The algorithm is implemented as follows:
\begin{itemize}
\item[(1)] Sample $\bm{x}_{k+1}$ from the invariant distribution $\varpi_{\bm{\theta}_{k}}(\bm{x})$,

\item[(2)] Update $\bm{\theta}_{k+1}=\bm{\theta}_{k}+\omega_{k+1} H(\bm{\theta}_{k}, \bm{x}_{k+1}).$
\end{itemize}

\subsection{General stochastic approximation}
The stochastic approximation algorithm \citep{Albert90} is a generalization of the Robbins–Monro algorithm, which consists of the following steps:
\begin{itemize}
\item[(1)] Sample $\bm{x}_{k+1}$ from the transition kernel  $\Pi_{\bm{\theta_{k}}}(\bm{x}_{k}, \cdot)$, which admits $\varpi_{\bm{\theta}_{k}}(\bm{x})$ as
the invariant distribution,

\item[(2)] Update $\bm{\theta}_{k+1}=\bm{\theta}_{k}+\omega_{k+1} H(\bm{\theta}_{k}, \bm{x}_{k+1})+\omega_{k+1}^2 \rho(\bm{\theta}_{k}, \bm{x}_{k+1})$.
\end{itemize}

In contrast to the Robbins-Monro algorithm, stochastic approximation samples $\bx$ from a transition kernel $\Pi_{\bm{\theta_{k}}}(\cdot, \cdot)$ instead of a distribution $\varpi_{\bm{\theta}_{k}}(\cdot)$, which leads to a Markov state-dependent noise $H(\btheta_k, \bx_{k+1})-h(\btheta_k)$. In addition, it allows small oscillations without affecting the convergence. 

\subsection{Poisson equation}

The coupled process $\{(\bx_k, \btheta_k)\}_{k=1}^{\infty}$ forms a nonhomogeneous Markov chain. Let $\Pi_{\bm{\theta}}(\bm{x}, A)$ be the transition kernel for any Borel subset $A\subset \bX$ and let a function $\mu_{\btheta}(\cdot)$ on $\bX$ solve the following Poisson equation 
\begin{equation*}
    \mu_{\btheta}(\bm{x})-\mathrm{\Pi}_{\bm{\theta}}\mu_{\bm{\theta}}(\bm{x})=H(\bm{\theta}, \bm{x})-h(\bm{\theta}).
\end{equation*}
The solution of Poisson equation can be formulated as follows when the series converge.
\begin{equation*}
    \mu_{\btheta}(\bx):=\sum_k \Pi_{\btheta}^k (H(\btheta, \bx)-h(\btheta)).
\end{equation*}
By imposing regularity conditions on $\mu_{\btheta}(\cdot)$, we can control the perturbations over time from $\int H(\btheta_k, \bx)\Pi_{\btheta_k}(\bx_k, d\bx)$ to $h(\btheta)$ and guarantee the consistency of the estimator $\btheta$. In particular, \citet{Albert90} has reduced the study of individual algorithms to the verification of the following regularity assumption on $\mu_{\btheta}(\cdot)$ and bounded moment of certain degree on $V(\bx)$.

\textbf{Assumption}
There exist a function  $V: \mX \to [1,\infty)$, and a constant $C$ such that for all $\bm{\theta}, \bm{\theta}'\in \bm{\bTheta}$, we have
\begin{equation*}
\begin{split}
\|\mathrm{\Pi}_{\bm{\theta}}\mu_{\btheta}(\bx)\|&\leq C V(\bx),\ \|\mathrm{\Pi}_{\bm{\theta}}\mu_{\bm{\theta}}(\bx)-\mathrm{\Pi}_{\bm{\theta'}}\mu_{\bm{\theta'}}(\bx)\|\leq C\|\bm{\theta}-\bm{\theta}'\| V(\bx), \E[V(\bx)]\leq \infty.\\
\end{split}
\end{equation*}

Notably, only 1st-order smoothness is required for the convergence of the adaptive algorithms \citep{Albert90, andrieu06}, which is much weaker than the 4th-order smoothness used in the ergodicity theory \citep{mattingly10, VollmerZW2016}.

% Poisson equation has been widely used in ergodic theory and adaptive algorithms to prove the desired limit of a time-average. Consider the infinitesimal generator $\mathcal{L}$ of the overdamped Langevin diffusion and let $\phi$ solve the Poisson equation
% \begin{equation}
%     \mathcal{L}\phi(\bx):=g-\bar g, 
% \end{equation}
% where $g$ is a test function and $\bar g$ is the expectation of $g$ over the Gibbs measure, defined as $\bar g=\int g(\bx) \varpi(d\bx)$. It is known that in a d-dimensional torus $\mathbb{T}^d$ and under elliptic settings, there is a unique solution for the Poisson equation, which is at least k+2-order smooth given a k-order smooth test function $g$ \citep{mattingly10}. To extend the ergodic average from $\mathbb{T}^d$ to $\mathbb{R}^d$, \citet{VollmerZW2016} established the required assumptions to establish the existence of smooth solutions of Poisson equation for stochastic gradient Langevin dynamics.

\section{Convergence Analysis}\label{convergence}

\subsection{Our algorithm}

Our algorithm falls into the class of the stochastic approximation algorithm, which follows
\begin{itemize}
\item[(1)] Sample $\bm{x}_{k+1}=\bx_k- \epsilon_k\nabla_{\bx} \widetilde L(\bx_k, \btheta_k)+\mathcal{N}({0, 2\epsilon_k \tau\bm{I}})$,

\item[(2)] Update $\bm{\theta}_{k+1}=\bm{\theta}_{k}+\omega_{k+1} \widetilde H(\bm{\theta}_{k}, \bm{x}_{k+1})+\omega_{k+1}^2 \rho(\bm{\theta}_{k}, \bm{x}_{k+1})$.
\end{itemize}
where $\epsilon_k$ is the learning rate, $\omega_k$ is the step size, $\nabla_{\bx} \widetilde L(\bx_k, \btheta_k)$ is the stochastic gradient for $\nabla_{\bx} L(\bx_k, \btheta_k)$. The stochastic random field $\widetilde H(\bm{\theta}_{k}, \bm{x}_{k+1})$ is an estimator of the random field $ H(\bm{\theta}_{k}, \bm{x}_{k+1})$ caused by the mini-batch evaluations. Without loss of generality, we assume 
\begin{equation}
    \label{tildeh}
    \widetilde H(\bm{\theta}_{k}, \bm{x}_{k+1})= H(\bm{\theta}_{k}, \bm{x}_{k+1})+\delta(\btheta_{k}, \bx_{k+1}),
\end{equation}
where the bias $\delta(\btheta_{k}, \bx_{k+1})$ is a random vector, and is usually biased due to the non-linearity of $H(\bm{\theta}_{k}, \bm{x}_{k+1})$. For example, when $H(\bm{\theta}_{k}, \bm{x}_{k+1})=L^2(\bx_{k+1}, \btheta_k)$ and $\widetilde H(\bm{\theta}_{k}, \bm{x}_{k+1})=(L(\bx_{k+1}, \btheta_k)+\mathcal{N}(0,\sigma^2))^2$, it is clear that $\E[\widetilde H(\bm{\theta}_{k}, \bm{x}_{k+1})]=L^2(\bx_{k+1}, \btheta_k)+\sigma^2\neq H(\bm{\theta}_{k}, \bm{x}_{k+1})$, which induced a fixed bias term $\sigma^2$. As such, we know that the bias becomes smaller given a larger batch size and vanishes when the full dataset is used. %We can also apply a bias correction step to reduce the bias similar to \citet{Matias19}, but that goes beyond the scope of our paper. 

\subsection{Convergence analysis of stochastic approximation}

In order to show the convergence analysis, we first lay out the following assumptions:

\begin{assump}[Step size]
\label{ass1}
$\{\omega_{k}\}_{k\in \mathrm{N}}$ is a positive decreasing sequence of real numbers such that
\begin{equation} \label{a1}
\omega_{k}\rightarrow 0, \ \ \sum_{k=1}^{\infty} \omega_{k}=+\infty,\ \  \lim_{k\rightarrow \infty} \inf 2  \dfrac{\omega_{k}}{\omega_{k+1}}+\dfrac{\omega_{k+1}-\omega_{k}}{\omega^2_{k+1}}>0.
\end{equation}
According to \citet{Albert90}, we can choose $\omega_{k}:=\frac{A}{k^{\alpha}+B},\ \ \text{where} \ \alpha \in (0, 1]\ \text{and}\ A>\frac{\alpha}{2}.$

\end{assump}

\begin{assump}[Compactness] \label{ass2a} 
We study $\btheta, H(\btheta, \bx)$, $\widetilde H(\btheta, \bx)$, $\rho(\btheta, \bx)$ and $\delta(\btheta, \bx)$ in a compact space, and there exists a constant $Q>0$ such that for $\forall \btheta\in \bTheta$, $\bx \in \bX$ and $k \in \mathbb{N}$.
\begin{equation}
\label{compactness}
     \|\btheta\|\leq Q, \|\rho(\btheta, \bx)\|\leq Q, \|\delta(\btheta, \bx)\|\leq Q, \|H(\btheta, \bx)\|\leq Q,\|\widetilde H(\btheta, \bx)\|\leq Q.
\end{equation}
\end{assump}

\begin{assump}[Smoothness]
\label{ass2}
$L(\bm{\xeta}, \bm{\theta})$ is $M$-smooth, namely, for any $\bx, \bx'\in \mX$, $\bm{\theta}, \bm{\theta}'\in \bTheta$.
\begin{equation}
\begin{split}
\label{ass_2_1_eq}
\|\nabla_{\bx} L(\bx, \btheta)-\nabla_{\bx} L(\bm{\bx}', \btheta')\|\leq M\|\bx-\bx'\|+M\|\btheta-\btheta'\|. 
\end{split}
\end{equation}
\end{assump}

\begin{assump}[Dissipativity]
\label{ass3}
 There exist constants $m>0, b\geq 0$, such that for any $\bx \in \mX$ and $\btheta \in \bTheta$, 
\label{ass_dissipative}
\begin{equation}
\label{eq:01}
\langle \nabla_{\bx} L(\bx, \btheta), \bx\rangle\leq b-m\|\bx\|^2.
\end{equation}
\end{assump}
This assumption has been widely used in proving the geometric ergodicity of dynamical systems \citep{mattingly02, Maxim17, Xu18}. It ensures a particle to move towards the origin regardless of the starting position.

\begin{assump}[Gradient noise] 
\label{ass4}
The stochastic noise follows
\begin{equation*}
\E[\nabla_{\bx}\widetilde L(\bx_{k},
 \btheta_{k})-\nabla_{\bx} L(\bx_{k}, \btheta_{k})]=0.
\end{equation*}
There exists some constants $M, B>0$ such that the second moment of the noise is bounded by
\begin{equation*} 
\E [\|\nabla_{\bx}\widetilde L(\bx_{k},
 \btheta_{k})-\nabla_{\bx} L(\bx_{k}, \btheta_{k})\|^2]\leq M^2 \|\bx\|^2+B^2. 
\end{equation*}

\end{assump}

\begin{lemma}[Stability]
\label{convex}
The mean field function $h(\btheta)$ satisfies that $\forall \btheta \in \bTheta$, $\langle h(\btheta) , \btheta -\btheta_{\star}\rangle \leq  -\|\btheta - \btheta_{\star}\|^2$. The mean field system $\dot{\btheta}=-h(\btheta)$ is globally asymptotically stable and $\btheta_{\star}$ is the globally asymptotically stable equilibrium.
\end{lemma}

\begin{proof}

Given the random field $H(\btheta, \bx) = \btheta_{ \mathcal{I}(\bx)}^{\zeta}\bm{1}_{\mathcal{I}(\bx)} - \btheta$, the mean field function $h(\btheta)$ under probability $\varpi_{\bm{\theta}}(\bx)$ follows
\begin{equation*}
\small
    h(\btheta)=\int H(\btheta, \bx) \varpi_{\bm{\theta}}(d\bx)=\int \left(\btheta_{\mathcal{I}(\bx)}^{\zeta}\bm{1}_{\mathcal{I}(\bx)} - \btheta\right) \dfrac{\pi(\bx)}{\btheta_{\mathcal{I}(\bx)}^{\zeta}} d\bx=\int \bm{1}_{\mathcal{I}(\bx)} \pi(\bx)d\bx-\btheta=\btheta_{\star}-\btheta.
\end{equation*}

Thus,
\begin{equation*}
    \langle h(\btheta), \btheta -\btheta_{\star}\rangle = -\|\btheta - \btheta_{\star}\|^2\leq  -\|\btheta - \btheta_{\star}\|^2.
\end{equation*}

Consider a positive definite Lyapunov function $L(\btheta)=\frac{1}{2}\left(\btheta_{\star}-\btheta\right)^2$ for the mean field system $\dot{\btheta}=-\left(\btheta_{\star}-\btheta\right)$. It is clear that $\dot{L}=\frac{\partial L(\btheta)}{\partial \btheta} \dot{\btheta}=-\left(\btheta_{\star}-\btheta\right)^2<0$ for $\forall \btheta \neq \btheta_{\star}$. This shows that the mean field system is globally asymptotically stable and $\btheta_{\star}$ is the globally asymptotically stable equilibrium.
\end{proof}

%  For a function $V: \mX \to [1,\infty)$ and a function $q: \mX \to \mR^{d_{\btheta}}$, 
%  define the norm
% \[
%  \|q\|_V=\sup_{\bx\in \mX} \frac{\|q(\bx)\|}{V(\bx)},
%  \]
%  where $d_{\btheta}$ denotes the dimension of $\btheta$. 
%  Let $\mathcal{L}_V=\{q: \mX \to \mR^{d_{\btheta}}, \sup_{x\in \mX} \|q\|_V <\infty\}$.  

The following lemma is a restatement of Lemma 1 in \citet{deng2019}.
\begin{lemma}[Uniform $L^2$ bounds]
\label{lemma:1}
Suppose Assumptions \ref{ass1}-\ref{ass4} holds.  Given a small enough learning rate
 $\ 0<\epsilon<\operatorname{Re}(\tfrac{m-\sqrt{m^2-3M^2}}{3M^2})\wedge 1$,  then 
$\sup_{k\geq 1} \E[\|\bm{\xeta}_{k}\|^2] \leq \E [\|\bm{\xeta}_0\|^2]+ \tfrac{1}{2m}(2b+3 B^2+2\tau d)$.
\end{lemma}

\begin{lemma}[Solution of Poisson equation]
\label{lyapunov}
There is a solution $\mu_{\btheta}(\cdot)$ on $\bX$ to the Poisson equation 
\begin{equation}
    \label{poisson_eqn}
    \mu_{\btheta}(\bm{x})-\mathrm{\Pi}_{\bm{\theta}}\mu_{\bm{\theta}}(\bm{x})=H(\bm{\theta}, \bm{x})-h(\bm{\theta}).
\end{equation}
such that for all $\bm{\theta}, \bm{\theta}'\in \bm{\bTheta}$ and a function  $V(\bx)=1+\|\bx\|^2$, there exists a constant $C$ such that
\begin{equation}
\begin{split}
\label{poisson_reg}
\E[\|\mathrm{\Pi}_{\bm{\theta}}\mu_{\btheta}(\bx)\|]&\leq C,\\
\E[\|\mathrm{\Pi}_{\bm{\theta}}\mu_{\bm{\theta}}(\bx)-\mathrm{\Pi}_{\bm{\theta}'}\mu_{\bm{\theta'}}(\bx)\|]&\leq C\|\bm{\theta}-\bm{\theta}'\|.\\
\end{split}
\end{equation}
\end{lemma}

\begin{proof}
According to Assumption 12 in \citet{VollmerZW2016}, we can easily show that A.7 and A.8 in \citet{VollmerZW2016} is satisfied given a Lyapunov function $V(\bx)=1+\|\bx\|^2$, the dissipitivity condition (\ref{ass3}) and the gradient noise condition (\ref{ass4})\footnote{Adaptive algorithms only require 1st-order smoothness of the solution of Poisson equation instead of 4th-order smoothness as used in proving the ergodicity average \citep{mattingly10, VollmerZW2016}. Therefore, as shown in Lemma.15 \citep{VollmerZW2016}, a much weaker bound such as (\ref{ass4}) can be applied. Further improvement on the results of \citet{Pardoux01} goes beyond of this paper.}. In what follows, Theorem 13 (\citet{Pardoux01}) in \citet{VollmerZW2016} holds, which shows that for any function $H(\btheta, \bx)\lesssim V(\bx)$, we have $\mu_{\btheta}\lesssim V(\bx)$ and $\nabla \mu_{\btheta}\lesssim V(\bx)$. This implies that there exists a constant $\overline{C}$ such that 
\begin{equation}
\begin{split}
\label{first_reg_poisson}
\|\mu_{\btheta}(\bx)\|&\leq \overline{C} V(\bx),\\
\|\mu_{\bm{\theta}}(\bx)-\mu_{\bm{\theta'}}(\bx)\|&\leq \overline{C}\|\bm{\theta}-\bm{\theta}'\| V(\bx).\\
\end{split}
\end{equation}

Together with the triangle inequality, it follows that
\begin{equation*}
    \|\mathrm{\Pi}_{\bm{\theta}}\mu_{\bm{\theta}}(\bx)-\mathrm{\Pi}_{\bm{\theta}'}\mu_{\bm{\theta'}}(\bx)\|\leq \|\mathrm{\Pi}_{\bm{\theta}}\mu_{\bm{\theta}}(\bx)-\mathrm{\Pi}_{\bm{\theta}}\mu_{\bm{\theta'}}(\bx)\|+\|\mathrm{\Pi}_{\bm{\theta}}\mu_{\bm{\theta'}}(\bx)-\mathrm{\Pi}_{\bm{\theta}'}\mu_{\bm{\theta'}}(\bx)\|.
\end{equation*}

Combining Lemma.\ref{lemma:1} to ensure the bounded Lyapunov function  $V(\bx)$ in expectation, we can obtain the desired result
\begin{equation*}
\begin{split}
\E[\|\mathrm{\Pi}_{\bm{\theta}}\mu_{\btheta}(\bx)\|]&\leq C,\\
\E[\|\mathrm{\Pi}_{\bm{\theta}}\mu_{\bm{\theta}}(\bx)-\mathrm{\Pi}_{\bm{\theta}'}\mu_{\bm{\theta'}}(\bx)\|]&\leq C \|\bm{\theta}-\bm{\theta}'\|.\\
\end{split}
\end{equation*}
by applying (\ref{first_reg_poisson}) and the smoothness condition in \ref{ass2} for some constant $C$.
\end{proof}

Now we are ready to present our first main result, where the technique lemmas are shown in $\S$\ref{technique}.

% \textbf{Theorem 1} ($L^2$ convergence rate).
\begin{theorem}[$L^2$ convergence rate]
\label{latent_converge}
Assume Assumptions $\ref{ass1}$-$\ref{ass4}$ hold. For a large enough $k_0$, small enough learning rates $\{\epsilon_k\}_{k=1}^{\infty}$ and step sizes $\{\omega_k\}_{k=1}^{\infty}$, there exists a globally asymptotically equilibrium $\btheta_{\star}$ such that
\begin{equation*}
    \E\left[\|\bm{\theta}_{k}-\bm{\theta}_{\star}\|^2\right]=\mathcal{O}( \omega_{k})+Q\sup_{n\geq k_0}\E[\|\delta(\btheta_{n}, \bx_{n+1})\|].
\end{equation*}
\end{theorem}
\begin{proof}

Consider the following iterates 
\begin{equation*}
    \bm{\theta}_{k+1}=\bm{\theta}_{k}+\omega_{k+1} \left(\widetilde H(\bm{\theta}_{k}, \bm{x}_{k+1})+\omega_{k+1} \rho(\btheta_k, \bx_{k+1})\right).
\end{equation*}

Denote $\bm{T}_{k}=\bm{\theta}_{k}-\bm{\theta}_{\star}$. By subtracting $\btheta_{\star}$ on both sides and taking the square and $L_2$ norm,  we have
\begin{equation*}
\begin{split}
    \|\bT_{k+1}^2\|&=\|\bT_k^2\| +\omega_{k+1}^2 \|\widetilde H(\btheta_k, \bx_{k+1}) + \omega_{k+1}\rho(\btheta_k, \bx_{k+1})\|^2\\
    &\ \ \ +2\omega_{k+1}\underbrace{\langle \bT_k, \widetilde H(\btheta_k, \bx_{k+1})+\omega_{k+1}\rho(\btheta_k, \bx_{k+1})\rangle}_{\text{D}}.
\end{split}
\end{equation*}

First, using Lemma.\ref{convex_property} in $\S$\ref{technique}, there exists a constant $G=9 Q^2(1+Q^2)$ such that
\begin{equation}
\label{first_term}
    \|\widetilde H(\btheta_k, \bx_{k+1}) + \omega_{k+1}\rho(\btheta_k, \bx_{k+1})\|^2 \leq G (1+\|\bT_k\|^2).
\end{equation}

Next, according to the definition of $\widetilde H(\btheta_k, \bx_{k+1})$ in (\ref{tildeh}) and the Poisson equation (\ref{poisson_eqn}), we have
\begin{equation*}
\begin{split}
   \text{D}&=\langle \bT_k,  H(\btheta_k, \bx_{k+1})+\delta(\btheta_{k}, \bx_{k+1})+\omega_{k+1}\rho(\btheta_k, \bx_{k+1}) \rangle\\
   &=\langle \bT_k,  h(\btheta_k)+\mu_{\btheta_k}(\bm{x}_{k+1})-\mathrm{\Pi}_{\bm{\theta}_k}\mu_{\bm{\theta}_k}(\bm{x}_{k+1})+\delta(\btheta_{k}, \bx_{k+1})+\omega_{k+1}\rho(\btheta_k, \bx_{k+1}) \rangle\\
   &=\underbrace{\langle \bT_k,  h(\btheta_k)\rangle}_{\text{D}_{1}} +\underbrace{\langle\bT_k, \mu_{\btheta_k}(\bm{x}_{k+1})-\mathrm{\Pi}_{\bm{\theta}_k}\mu_{\bm{\theta}_k}(\bm{x}_{k+1})\rangle}_{\text{D}_{2}}+\underbrace{\langle \bT_k, \delta(\btheta_{k}, \bx_{k+1})+\omega_{k+1}\rho(\btheta_k, \bx_{k+1})\rangle}_{{\text{D}_{3}}}.
\end{split}
\end{equation*}

Using the stability property of the equilibrium in Lemma \ref{convex}, we have 
\begin{align*}
\langle \bm{T}_{k}, h(\bm{\theta}_{k})\rangle &\leq - \|\bm{T}_{k}\|^2. \tag{$\text{D}_1$}
\end{align*}
To deal with the error $\text{D}_2$, we make the following decomposition 
\begin{equation*}
\begin{split}
\text{D} &=\underbrace{\langle \bT_k, \mu_{\bm{\theta}_{k}}(\bm{\xeta}_{k+1})-\mathrm{\Pi}_{\bm{\theta}_{k}}\mu_{\bm{\theta}_{k}}(\bm{\bx}_{k})\rangle}_{\text{D}_{21}} \\
&+ \underbrace{\langle \bT_k,\mathrm{\Pi}_{\bm{\theta}_{k}}\mu_{\bm{\theta}_{k}}(\bm{x}_{k})- \mathrm{\Pi}_{\bm{\theta}_{k-1}}\mu_{\bm{\theta}_{k-1}}(\bm{x}_{k})\rangle}_{\text{D}_{22}}
+ \underbrace{\langle \bT_k,\mathrm{\Pi}_{\bm{\theta}_{k-1}}\mu_{\bm{\theta}_{k-1}}(\bm{x}_{k})- \mathrm{\Pi}_{\bm{\theta}_{k}}\mu_{\bm{\theta}_{k}}(\bm{\xeta}_{k+1})\rangle}_{\text{D}_{23}}.\\
\end{split}
\end{equation*}

(i) From the Markov property, $\mu_{\bm{\theta}_{k}}(\bm{\xeta}_{k+1})-\mathrm{\Pi}_{\bm{\theta}_{k}}\mu_{\bm{\theta}_{k}}(\bm{x}_{k})$ forms a martingale difference sequence 
$$\E\left[\langle \bT_k, \mu_{\bm{\theta}_{k}}(\bm{\xeta}_{k+1})-\mathrm{\Pi}_{\bm{\theta}_{k}}\mu_{\bm{\theta}_{k}}(\bm{x}_{k})\rangle |\mathcal{F}_{k}\right]=0. \eqno{(\text{D}_{21})}$$

(ii) By the regularity of the solution of Poisson equation in (\ref{poisson_reg}) and Lemma.\ref{theta_lip} in $\S$\ref{technique}, it leads to
\begin{equation}
\label{theta_delta}
\E[\|\mathrm{\Pi}_{\bm{\theta}_{k}}\mu_{\bm{\theta}_{k}}(\bm{x}_{k})- \mathrm{\Pi}_{\bm{\theta}_{k-1}}
 \mu_{\bm{\theta}_{k-1}}(\bm{x}_{k})\|]\leq C \|\btheta_k-\btheta_{k-1}\|\leq 2Q C\omega_k.
\end{equation}
Using Cauchy–Schwarz inequality, (\ref{theta_delta}) and the compactness of $\btheta$ in Assumption \ref{ass2a}, it follows that
$$\small{\E[\langle\bm{T}_{k},\mathrm{\Pi}_{\bm{\theta}_{k}}\mu_{\bm{\theta}_{k}}(\bm{x}_{k})- \mathrm{\Pi}_{\bm{\theta}_{k-1}}\mu_{\bm{\theta}_{k-1}}(\bm{x}_{k})\rangle]\leq \E[\|\bT_k\|]\cdot 2Q C\omega_k\leq 4Q^2 C\omega_{k}\leq 5Q^2 C\omega_{k+1}}   \eqno{(\text{D}_{22})},$$
where the last inequality follows from the step size assumption \ref{ass1} and holds given a large enough $k$.

(iii) \begin{equation*}
\begin{split}
\small
\langle \bm{T}_{k},\mathrm{\Pi}_{\bm{\theta}_{k-1}}\mu_{\bm{\theta}_{k-1}}(\bm{x}_{k})- \mathrm{\Pi}_{\bm{\theta}_{k}}\mu_{\bm{\theta}_{k}}(\bm{\xeta}_{k+1})\rangle
&=\left(\langle \bm{T}_{k}, \mathrm{\Pi}_{\bm{\theta}_{k-1}}\mu_{\bm{\theta}_{k-1}}(\bm{x}_{k}) \rangle- \langle \bm{T}_{k+1}, \mathrm{\Pi}_{\bm{\theta}_{k}}\mu_{\bm{\theta}_{k}}(\bm{\xeta}_{k+1})\rangle\right)\\
&\ \ \ +\left(\langle \bm{T}_{k+1}, \mathrm{\Pi}_{\bm{\theta}_{k}}\mu_{\bm{\theta}_{k}}(\bm{\xeta}_{k+1})\rangle-\langle \bm{T}_{k}, \mathrm{\Pi}_{\bm{\theta}_{k}}\mu_{\bm{\theta}_{k}}(\bm{\xeta}_{k+1})\rangle\right)\\
&={({z}_{k}-{z}_{k+1})}+{\langle \bm{T}_{k+1}-\bm{T}_{k}, \mathrm{\Pi}_{\bm{\theta}_{k}}\mu_{\bm{\theta}_{k}}(\bm{\xeta}_{k+1})\rangle},\\
\end{split}
\end{equation*}
where ${z}_{k}=\langle \bm{T}_{k}, \mathrm{\Pi}_{\bm{\theta}_{k-1}}\mu_{\bm{\theta}_{k-1}}(\bm{x}_{k})\rangle$. Using the regularity assumption in (\ref{poisson_reg}) and Lemma.\ref{theta_lip}, we have
$$\E\langle \bm{T}_{k+1}-\bm{T}_{k}, \mathrm{\Pi}_{\bm{\theta}_{k}}\mu_{\bm{\theta}_{k}}(\bm{\xeta}_{k+1})\rangle\leq   \E[\|\bm{\theta}_{k+1}-\bm{\theta}_{k}\|] \cdot \E[\|\mathrm{\Pi}_{\bm{\theta}_{k}}\mu_{\bm{\theta}_{k}}(\bm{\xeta}_{k+1})\|] \leq 2Q C \omega_{k+1}.\eqno{(\text{D}_{23})}$$

Furthermore, denote $\E[\|\delta(\btheta_{k}, \bx_{k+1})\|]$ by $\Delta_k$. Since $\rho(\btheta_k, \bx_{k+1})$ is compact, it follows that
\begin{equation*}
    \E[\|\delta(\btheta_{k}, \bx_{k+1})+\omega_{k+1}\rho(\btheta_k, \bx_{k+1}))\|]\leq \Delta_k + \omega_{k+1} Q.
\end{equation*}
Applying Cauchy–Schwarz inequality gives
$${\E[\langle \bT_k, \delta(\btheta_{k}, \bx_{k+1})+\omega_{k+1}\rho(\btheta_k, \bx_{k+1}))]\leq 2Q(\Delta_k+\omega_{k+1} Q)} \eqno{(\text{D}_{3})}$$

Finally, adding (\ref{first_term}), $\text{D}_1$, $\text{D}_{2}$ and $\text{D}_3$ together, it follows that for a constant 
     $$C_0 = G+5Q^2C+2QC+2Q^2,$$
we have
\begin{equation}
\begin{split}
\label{key_eqn}
\E\left[\|\bm{T}_{k+1}\|^2\right]&\leq (1-2\omega_{k+1}+G\omega^2_{k+1} )\E\left[\|\bm{T}_{k}\|^2\right]+C_0\omega^2_{k+1} +2Q\Delta_k\omega_{k+1} +2\E[z_{k}-z_{k+1}]\omega_{k+1}.
\end{split}
\end{equation}
Moreover, from (\ref{compactness}) and (\ref{poisson_reg}), $\E[|z_{k}|]$ is upper bounded by
\begin{equation}
\begin{split}
\label{condition:z}
\E[|z_{k}|]=\E[\langle \bm{T}_{k}, \mathrm{\Pi}_{\bm{\theta}_{k-1}}\mu_{\bm{\theta}_{k-1}}(\bm{x}_{k})\rangle]\leq \E[\|\bT_k\|]\E[\|\mathrm{\Pi}_{\bm{\theta}_{k-1}}\mu_{\bm{\theta}_{k-1}}(\bm{x}_{k})\|]\leq 2QC.
\end{split}
\end{equation}

According to Lemma $\ref{lemma:4}$ in $\S$\ref{technique}, we can choose $\lambda_0$ and $k_0$ such that 
\begin{align*}
\E[\|\bm{T}_{k_0}\|^2]\leq \psi_{k_0}=\lambda_0 \omega_{k_0}+Q\sup_{n\geq k_0}\Delta_{n},
\end{align*}
which satisfies the conditions ($\ref{lemma:3-a}$) and ($\ref{lemma:3-b}$) of Lemma $\ref{lemma:3-all}$ in $\S$\ref{technique}. Applying Lemma $\ref{lemma:3-all}$ leads to
\begin{equation}
\begin{split}
\label{eqn:9}
\E\left[\|\bm{T}_{k}\|^2\right]\leq \psi_{k}+\E\left[\sum_{j=k_0+1}^{k}\Lambda_j^k \left(z_{j-1}-z_{j}\right)\right],
\end{split}
\end{equation}
where $\psi_{k}=\lambda_0 \omega_{k}+Q\sup_{n\geq k_0}\Delta_{n}$ for all $k>k_0$. Based on ($\ref{condition:z}$) and the increasing condition of $\Lambda_{j}^k$ in Lemma $\ref{lemma:2}$ in $\S$\ref{technique}, we have
\begin{equation}
\small
\begin{split}
\label{eqn:10}
&\E\left[\left|\sum_{j=k_0+1}^{k} \Lambda_j^k\left(z_{j-1}-z_{j}\right)\right|\right]
=\E\left[\left|\sum_{j=k_0+1}^{k-1}(\Lambda_{j+1}^k-\Lambda_j^k)z_j-2\omega_{k}z_{k}+\Lambda_{k_0+1}^k z_{k_0}\right|\right]\\
\leq& \sum_{j=k_0+1}^{k-1}2(\Lambda_{j+1}^k-\Lambda_j^k)QC+\E[|2\omega_{k} z_{k}|]+2\Lambda_k^k QC\\
\leq& 2(\Lambda_k^k-\Lambda_{k_0}^k)QC+2\Lambda_k^k QC+2\Lambda_k^k QC\\
\leq& 6\Lambda_k^k QC.
\end{split}
\end{equation}

Therefore, given $\psi_{k}=\lambda_0 \omega_{k}+Q\sup_{n\geq k_0}\Delta_{n}$ that satisfies the conditions ($\ref{lemma:3-a}$), ($\ref{lemma:3-b}$) of Lemma $\ref{lemma:3-all}$, for any $k>k_0$, from ($\ref{eqn:9}$) and ($\ref{eqn:10}$), we have
\begin{equation*}
\E[\|\bm{T}_{k}\|^2]\leq \psi_{k}+6\Lambda_k^k QC=\left(\lambda_0+12QC\right)\omega_{k}+Q\sup_{n\geq k_0}\Delta_{n}=\lambda \omega_{k}+Q\sup_{n\geq k_0}\Delta_{n},
\end{equation*}
where $\lambda=\lambda_0+12QC$, $\lambda_0=\frac{2GQ\sup_{n\geq k_0} \Delta_n + 2C_0}{C_1}$, $\small{C_1=\lim \inf 2  \dfrac{\omega_{k}}{\omega_{k+1}}+\dfrac{\omega_{k+1}-\omega_{k}}{\omega^2_{k+1}}>0}$, $C_0=G+5Q^2C+2QC+2Q^2$ and $G=9 Q^2(1+Q^2)$.
\end{proof}

\subsection{Ergodicity and weighted averaging estimators}

Our interest is to analyze the deviation between the weighted averaging estimator $\frac{1}{k}\sum_{i=1}^k\btheta_{i, \mathcal{I}(\bx_i)}^{\zeta} f(\bx_i)$ and posterior average $\int f(\bx)\pi(d\bx)$ for a test function $f$. To obtain the desired error analysis, we first need to study the convergence of the empirical mean $\frac{1}{k}\sum_{i=1}^k f(\bx_i)$ to the posterior average $\bar f=\int f(\bx)\varpi_{\btheta_{\star}}(d\bx)$. The key tool for ergodic theory is still the Poisson equation to characterize the fluctuations between $f(\bx)$ and $\bar f$, which is defined as follows:
\begin{equation}
    \mathcal{L}\phi(\bx)=f(\bx)-\bar f,
\end{equation}
where $\phi(\bx)$ is the solution of the Poisson equation, and $\mathcal{L}$ is the infinitesimal generator of the Langevin diffusion 
\begin{equation*}
    \mathcal{L}\phi:=\nabla \phi \nabla L(\cdot, \btheta_{\star})+\tau\nabla^2\phi.
\end{equation*}

By imposing regularity conditions on the function $\phi(\bx)$, we can control the perturbations of $\frac{1}{k}\sum_{i=1}^k f(\bx_i)-\bar f$. Now, we present the first result, which is majorly adapted from Theorem 2 \citep{Chen15} with a fixed learning rate $\epsilon$. Similar conclusions have also been achieved in \citet{Teh16, VollmerZW2016, Dalalyan18}.
\begin{lemma}[Convergence of the Averaging Estimators]
\label{avg_converge}
Assume Assumptions $\ref{ass1}$-$\ref{ass4}$ hold. Given a sufficiently smooth function $\phi(\bx)$ and a function $\mathcal{V}(\bx)$, such that $\|D^k \phi\|\lesssim \mathcal{V}^{p_k}(\bx)$ and $p_k>0$ for $k\in\{0,1,2,3\}$. In addition, $\mathcal{V}^p$ has a bounded expectation: $\sup_{\bx} \E[\mathcal{V}^p(\bx)]<\infty$ and $\mathcal{V}$ is smooth, i.e. $\sup_{s\in\{0, 1\}} \mathcal{V}^p(s\bx+(1-s)\by)\lesssim \mathcal{V}^p(\bx)+\mathcal{V}^p(\by)$ for all $\bx,\by\in\bX$ and $p\leq 2\max_k\{p_k\}$. For any integrable function $f^2$, we have
\begin{equation*}
\begin{split}
    \left|\E\left[\frac{\sum_{i=1}^k f(\bx_i)}{k}\right]-\int f(\bx)\varpi_{\bm{\theta}_{\star}}(d\bx)\right|&= \mathcal{O}\left(\frac{1}{k\epsilon}+\epsilon+\sqrt{\frac{\sum_{i=1}^k \omega_k}{k}}+\sup_{n\geq k_0}\E[\|\delta(\btheta_{n}, \bx_{n+1})\|]^{0.5}\right). \\
\end{split}
\end{equation*}
\end{lemma}

\begin{proof}

To study the ergodic average, we can view the adaptive algorithm as a standard sampling algorithm with fixed latent variable $\btheta_{\star}$ and biased gradients, where the bias term is denoted by $$\bbeta_k=\nabla_{\bx} L(\bx, \btheta_k)-\nabla_{\bx}  L(\bx, \btheta_{\star}).$$

According to the smoothness assumption \ref{ass2} and Jensen's inequality, we have
\begin{equation}
\label{latent_bias}
\small
    \|\E[\bbeta_k]\|\leq \E[\|\bbeta_k\|]=\E[\|\nabla_{\bx} L(\bx, \btheta_k)-\nabla_{\bx}  L(\bx, \btheta_{\star})\|] \leq M\E[\|\btheta_k-\btheta_{\star}\|] \leq M\sqrt{\E[\|\btheta_k-\btheta_{\star}\|^2]}.
\end{equation}

As a result, we can reformulate the original adaptive algorithm as
\begin{equation*}
\begin{split}
    \bm{x}_{k+1}&=\bx_k- \epsilon_k\nabla_{\bx} \widetilde L(\bx_k, \btheta_k)+\mathcal{N}({0, 2\epsilon_k \tau\bm{I}})\\
    &=\bx_k- \epsilon_k\left(\nabla_{\bx} \widetilde L(\bx_k, \btheta_{\star})+\bbeta_k\right)+\mathcal{N}({0, 2\epsilon_k \tau\bm{I}}).
\end{split}
\end{equation*}
The ergodic average based on biased gradients and a fixed learning rate is a direct result of Theorem 2 \citep{Chen15}. Given regularity conditions on the solution of the Poisson equation, (\ref{latent_bias}) and Theorem \ref{latent_converge}, we know that 
\begin{equation*}
\small
\begin{split}
    \left|\E\left[\frac{\sum_{i=1}^k f(\bx_i)}{k}\right]-\int f(\bx)\varpi_{\bm{\theta}_{\star}}(d\bx)\right|&\leq \mathcal{O}\left(\frac{1}{k\epsilon}+\epsilon+\frac{\sum_{i=1}^k \|\E[\bbeta_k]\|}{k}\right)\\
    &\lesssim \mathcal{O}\left(\frac{1}{k\epsilon}+\epsilon+\frac{\sum_{i=1}^k \left(\omega_k+\sup_{n\geq k_0}\E[\|\delta(\btheta_{n}, \bx_{n+1})\|]\right)^{0.5}}{k}\right), \\
    % &\leq \mathcal{O}\left(\frac{1}{k\epsilon}+\epsilon+\frac{\sum_{i=1}^k \omega_k^{0.5}}{k}+\sup_{n\geq k_0}\E[\|\delta(\btheta_{n}, \bx_{n+1})\|]^{0.5}\right) \\
    &\leq \mathcal{O}\left(\frac{1}{k\epsilon}+\epsilon+\sqrt{\frac{\sum_{i=1}^k \omega_k}{k}}+\sup_{n\geq k_0}\E[\|\delta(\btheta_{n}, \bx_{n+1})\|]^{0.5}\right)
\end{split}
\end{equation*}
where the last inequality follows from $\small{(\omega_k+\Delta)^{0.5}\leq \omega_k^{0.5}+\Delta^{0.5}}$ and $\small{\sum_{i=1}^k \omega_i^{0.5}\leq \sqrt{k\sum_{i=1}^k \omega_i}}$.
\end{proof}

Now we are ready to show the convergence of the weighted averaging estimator $\frac{1}{k}\sum_{i=1}^k\btheta_{i, \mathcal{I}(\bx_i)}^{\zeta} f(\bx_i)$ to the posterior average $\int f(\bx)\pi(d\bx)$.
\begin{theorem}[Convergence of the Weighted Averaging Estimators] Assume Assumptions $\ref{ass1}$-$\ref{ass4}$ hold. Given a sufficiently smooth function $\phi(\bx)$ and a function $\mathcal{V}(\bx)$, such that $\|D^k \phi\|\lesssim \mathcal{V}^{p_k}(\bx)$ and $p_k>0$ for $k\in\{0,1,2,3\}$. In addition, $\mathcal{V}^p$ has a bounded expectation: $\sup_{\bx} \E[\mathcal{V}^p(\bx)]<\infty$ and $\mathcal{V}$ is smooth, i.e. $\sup_{s\in\{0, 1\}} \mathcal{V}^p(s\bx+(1-s)\by)\lesssim \mathcal{V}^p(\bx)+\mathcal{V}^p(\by)$ for all $\bx,\by\in\bX$ and $p\leq 2\max_k\{p_k\}$. For any integrable function $f^2$, we have that 
\label{w_avg_converge}
\begin{equation*}
\begin{split}
    \left|\E\left[\frac{1}{k}\sum_{i=1}^k\btheta_{i, \mathcal{I}(\bx_i)}^{\zeta} f(\bx_i)\right]-\int f(\bx)\pi(d\bx)\right|&= \mathcal{O}\left(\frac{1}{k\epsilon}+\epsilon+\sqrt{\frac{\sum_{i=1}^k \omega_k}{k}}+\sup_{n\geq k_0}\E[\|\delta(\btheta_{n}, \bx_{n+1})\|]^{0.5}\right). \\
\end{split}
\end{equation*}
\end{theorem}

\begin{proof}

Applying triangle inequality and $|\E[x]|\leq \E[|x|]$, we have
\begin{equation*}
\small
    \begin{split}
        &\left|\E\left[\frac{1}{k}\sum_{i=1}^k\btheta_{i, \mathcal{I}(\bx_i)}^{\zeta} f(\bx_i)\right]-\int f(\bx)\pi(d\bx)\right|\\
        \leq &\underbrace{\E\left[\frac{1}{k}\sum_{i=1}^k\left|\btheta_{i, \mathcal{I}(\bx_i)}^{\zeta}-\btheta_{\star, \mathcal{I}(\bx_i)}^{\zeta}\right| \cdot |f(\bx_i)|\right]}_{\text{I}_1} +\underbrace{\left|\E\left[\frac{1}{k}\sum_{i=1}^k\btheta_{\star, \mathcal{I}(\bx_i)}^{\zeta} f(\bx_i)\right]-\int f(\bx)\pi(d\bx)\right|}_{\text{I}_2}.
    \end{split}
\end{equation*}

For the first term $\text{I}_1$, consider the Mean value theorem for $t(x)=x^{\zeta}$
\begin{equation}
    \begin{split}
    \label{mvt}
        |\btheta_{i, \mathcal{I}(\bx)}^{\zeta}(\mathbb{A})-\btheta_{\star, \mathcal{I}(\bx)}^{\zeta}|\leq |\btheta_{i, \mathcal{I}(\bx)}(\mathbb{A})-\btheta_{\star, \mathcal{I}(\bx)}| \cdot \widetilde\btheta^{\zeta}\lesssim |\btheta_{i, \mathcal{I}(\bx)}(\mathbb{A})-\btheta_{\star, \mathcal{I}(\bx)}|,
    \end{split}
\end{equation}
where the first inequality holds for any $\bx\in\bX$, any $i\in\{1,2,...,m\}$, any $\sigma$-algebra $\mathbb{A}$ for the stochastic variable $\btheta_{i, \mathcal{I}(\bx)}$ and some $\widetilde\btheta \leq  \btheta_{i, \mathcal{I}(\bx)}\vee \btheta_{\star, \mathcal{I}(\bx)}$; the last inequality follows because we only consider $\btheta$ in a compact set. By Cauchy-Schwarz inequality, (\ref{mvt}) and Theorem \ref{latent_converge}, it follows that
\begin{equation*}
\small
    \begin{split}
        \text{I}_1&\lesssim \sqrt{\E\left[\frac{\sum_{i=1}^k\left(\btheta_{i, \mathcal{I}(\bx_i)}-\btheta_{\star, \mathcal{I}(\bx_i)}\right)^2}{k} \right]\E\left[\sum_{i=1}^k\frac{f^2(\bx_i)}{k}\right]}\\
        &\lesssim \sqrt{\sum_{i=1}^k\dfrac{\E[\|\btheta_i-\btheta_{\star}\|^2]}{k}}\lesssim \sqrt{\frac{\sum_{i=1}^k\omega_{i}}{k}+\sup_{n\geq k_0}\E[\|\delta(\btheta_{n}, \bx_{n+1})\|]}\lesssim \sqrt{\frac{\sum_{i=1}^k\omega_{i}}{k}}+\sup_{n\geq k_0}\E[\|\delta(\btheta_{n}, \bx_{n+1})\|]^{0.5},
    \end{split}
\end{equation*}
where the second inequality holds because of the integrability of $f^2(\bx)$ and the last inequality follows from $\sqrt{x+y}\leq \sqrt{x}+\sqrt{y}$.

Before we study $\text{I}_2$, we first decompose $\int  f(\bx) \pi(d\bx)$ into $m$ disjoint regions to facilitate the analysis
\begin{equation}
\label{split_posterior}
\small
\begin{split}
      \int  f(\bx) \pi(d\bx)=\int_{\cup_{j=1}^m E_j}  f(\bx) \pi(d\bx)=\sum_{j=1}^m\btheta_{\star, j}^{\zeta}\int_{E_j}  f(\bx) \frac{\pi(d\bx)}{\btheta_{\star, j}^{\zeta}}=\sum_{j=1}^m\btheta_{\star, j}^{\zeta}\int_{E_j}  f(\bx) \varpi_{\bm{\theta}_{\star}}(d\bx).\\
\end{split}
\end{equation}

Plugging (\ref{split_posterior}) into the second term $\text{I}_2$, we have
\begin{equation}
\label{final_i2}
\small
    \begin{split}
        \text{I}_2&=\left|\E\left[\frac{1}{k}\sum_{i=1}^k\sum_{j=1}^m\btheta_{\star, j}^{\zeta} f(\bx_i)1_{\bx_i\in E_j}\right]-\int f(\bx)\pi(d\bx)\right|\\
        &=\left|\sum_{j=1}^m\btheta_{\star, j}^{\zeta}\E\left[\frac{1}{k}\sum_{i=1}^k f(\bx_i)1_{\bx_i\in E_j}\right]-\sum_{j=1}^m\btheta_{\star, j}^{\zeta}\int_{E_j}  f(\bx) \varpi_{\bm{\theta}_{\star}}(d\bx)\right|\\
        &\leq \sum_{j=1}^m\btheta_{\star, j}^{\zeta}\left|\E\left[\frac{1}{k}\sum_{i=1}^k f(\bx_i)1_{\bx_i\in E_j}\right]-\int_{E_j}  f(\bx) \varpi_{\bm{\theta}_{\star}}(d\bx)\right|.\\
    \end{split}
\end{equation}

Given any $j\in \{1,2,...,m\}$, applying the function $f(\bx)1_{\bx\in E_j}$ to Theorem \ref{avg_converge} leads to
\begin{equation}
\label{almost_i2}
\small
\begin{split}
      \left|\E\left[\frac{1}{k}\sum_{i=1}^k f(\bx_i)1_{\bx_i\in E_j}\right]-\int_{E_j}  f(\bx) \varpi_{\bm{\theta}_{\star}}(d\bx)\right|\leq \mathcal{O}\left(\frac{1}{k\epsilon}+\epsilon+\sqrt{\frac{\sum_{i=1}^k \omega_k}{k}}+\sup_{n\geq k_0}\E[\|\delta(\btheta_{n}, \bx_{n+1})\|]^{0.5}\right).\\
\end{split}
\end{equation}

Plugging (\ref{almost_i2}) into (\ref{final_i2}) and combining $\text{I}_1$, we have
\begin{equation}
\small
\begin{split}
      \left|\E\left[\frac{1}{k}\sum_{i=1}^k\btheta_{i, \mathcal{I}(\bx_i)}^{\zeta} f(\bx_i)\right]-\int f(\bx)\pi(d\bx)\right|\leq \mathcal{O}\left(\frac{1}{k\epsilon}+\epsilon+\sqrt{\frac{\sum_{i=1}^k \omega_k}{k}}+\sup_{n\geq k_0}\E[\|\delta(\btheta_{n}, \bx_{n+1})\|]^{0.5}\right).\\
\end{split}
\end{equation}

    % \btheta_{\star, i}\left[\frac{1}{n}\sum_{k=1}^n \btheta_{\star, \mathcal{I}(\bx_k)}^{\zeta} f(\bx_k) 1_{\bx_k\in E_i}\right]

\end{proof}

\section{Technical Lemmas}
\label{technique}

\begin{lemma}
\label{convex_property}
Given $\sup\{\omega_k\}_{k=1}^{\infty}\leq 1$, there exists a constant $G=9 Q^2(1+Q^2)$ such that
\begin{equation} \label{bound2}
\| \widetilde H(\bm{\theta}_k, \bm{\xeta}_{k+1})+\omega_{k+1}\rho(\btheta_k, \bx_{k+1})\|^2 \leq G (1+\|\bm{\theta}_k-\bm{\theta}_*\|^2). 
\end{equation}
\end{lemma}
\begin{proof}

According to the compactness condition in Assumption \ref{ass2a}, we have
\begin{equation*}
\|H(\bm{\theta}_k, \bm{\xeta}_{k+1})\|^2\leq Q^2 (1+\|\bm{\theta}_k\|^2) = 
 Q^2 (1+\|\bm{\theta}_k-\bm{\theta}_*+\bm{\theta}_*\|^2)\leq Q^2 (1+2\|\bm{\theta}_k-\bm{\theta}_*\|^2+2Q^2).
\end{equation*}

Therefore, we can show that for a constant $G=9Q^2(1+Q^2)$
\begin{equation*}
\small
\begin{split}
    \|\widetilde H(\bm{\theta}, \bm{\xeta})+\omega_{k+1}\rho(\btheta_k, \bx_{k+1})\|^2 &\leq 3\|H(\bm{\theta}_k, \bm{\xeta}_{k+1})\|^2 + 3\|\bdelta(\theta_k)\|^2+3\omega_{k+1}^2 \|\rho(\btheta_k, \bx_{k+1})\|\\
    &\leq 3Q^2 (1+2\|\bm{\theta}_k-\bm{\theta}_*\|^2+2Q^2) + 6Q^2\\
    &\leq 3Q^2 (3+3Q^2+(3+3Q^2)\|\bm{\theta}_k-\bm{\theta}_*\|^2)\\
    &\leq G (1+\|\bm{\theta}_k-\bm{\theta}_*\|^2).
\end{split}
\end{equation*}
\end{proof}

\begin{lemma}
\label{theta_lip}Given $\sup\{\omega_k\}_{k=1}^{\infty}\leq 1$, we have that
\begin{equation}
\label{lip_theta}
    \|\btheta_{k}-\btheta_{k-1}\|\leq 2\omega_{k} Q
\end{equation}
\end{lemma}

\begin{proof}
Following the update $\btheta_k-\btheta_{k-1}=\omega_k \widetilde H(\bm{\theta}_{k-1}, \bm{x}_{k})+\omega_{k}^2 \brho_{k}$, we have that
$$\|\btheta_{k}-\btheta_{k-1}\|= \|\omega_k \widetilde H(\bm{\theta}_{k-1}, \bm{x}_{k})+\omega_{k}^2 \brho_{k}\|\leq \omega_k\| \widetilde H(\bm{\theta}_{k-1},\bm{x}_{k})\|+\omega_{k}^2\| \brho_{k}\|.$$
By the compactness condition in Assumption \ref{ass2a} and $\sup\{\omega_k\}_{k=1}^{\infty}\leq 1$, (\ref{lip_theta}) can be derived.
\end{proof}

\begin{lemma}
\label{lemma:4}
There exist constants $\lambda_0$ and $k_0$ such that $\forall \lambda\geq\lambda_0$ and $\forall k> k_0$, the sequence $\{\psi_{k}\}_{k=1}^{\infty}$, where $\psi_{k}=\lambda\omega_{k}+Q \sup_{n\geq k_0}\Delta_n$, satisfies
\begin{equation}
\begin{split}
\label{key_ieq}
\psi_{k+1}\geq& (1-2\omega_{k+1}+G\omega_{k+1}^2)\psi_{k}+C_0\omega_{k+1}^2  +2Q \Delta_k\omega_{k+1}.
\end{split}
\end{equation}
\begin{proof}
By replacing $\psi_{k}$ with $\lambda\omega_{k}+Q \sup_{n\geq k_0}\Delta_n$ in ($\ref{key_ieq}$), it suffices to show
\begin{equation*}
\small
\begin{split}
\label{lemma:loss_control}
\lambda \omega_{k+1}+Q \sup_{n\geq k_0}\Delta_n\geq& (1-2\omega_{k+1}+G\omega_{k+1}^2)\left(\lambda \omega_{k}+Q \sup_{n\geq k_0}\Delta_n\right)+C_0\omega_{k+1}^2 + 2Q\Delta_k\omega_{k+1}.
\end{split}
\end{equation*}

which is equivalent to proving
\begin{equation*}
\small
\begin{split}
&\lambda (\omega_{k+1}-\omega_k+2\omega_k\omega_{k+1}-G\omega_k\omega_{k+1}^2)\geq  Q\sup_{n\geq k_0}\Delta_n(-2\omega_{k+1}+G\omega_{k+1}^2 )+C_0\omega_{k+1}^2+ 2Q\Delta_k\omega_{k+1}.
\end{split}
\end{equation*}

Given the step size condition in ($\ref{a1}$), we have $\small{\omega_{k+1}-\omega_{k}+2 \omega_{k}\omega_{k+1} \geq C_1 \omega_{k+1}^2}$, where $\small{C_1=\lim \inf 2  \dfrac{\omega_{k}}{\omega_{k+1}}+\dfrac{\omega_{k+1}-\omega_{k}}{\omega^2_{k+1}}>0}$. Together with the fact that $-\sup_{n\geq k_0}\Delta_n\leq \Delta_k$, it suffices to prove
\begin{equation}
\begin{split}
\label{loss_control-2}
\lambda \left(C_1-G\omega_{k}\right)\omega^2_{k+1}\geq  \left(GQ \sup_{n\geq k_0}\Delta_n+C_0\right)\omega^2_{k+1}.
\end{split}
\end{equation}

It is clear that for a large enough $k_0$ and $\lambda_0$ such that $\omega_{k_0}\leq \frac{C_1}{2G}$, $\lambda_0=\frac{2GQ\sup_{n\geq k_0} \Delta_n + 2C_0}{C_1}$, the desired conclusion ($\ref{loss_control-2}$) holds for all such $k\geq k_0$ and $\lambda\geq \lambda_0$.
\end{proof}
\end{lemma}

\begin{lemma}
\label{lemma:3-all}
Let $\{\psi_{k}\}_{k> k_0}$ be a series that satisfies the following inequality for all $k> k_0$
\begin{equation}
\begin{split}
\label{lemma:3-a}
\psi_{k+1}\geq &\psi_{k}\left(1-2\omega_{k+1}+G\omega^2_{k+1}\right)+C_0\omega^2_{k+1} + 2Q \Delta_k\omega_{k+1},
\end{split}
\end{equation}
and assume there exists such $k_0$ that 
\begin{equation}
\begin{split}
\label{lemma:3-b}
\E\left[\|\bm{T}_{k_0}\|^2\right]\leq \psi_{k_0}.
\end{split}
\end{equation}
Then for all $k> k_0$, we have
\begin{equation}
\begin{split}
\label{result}
\E\left[\|\bm{T}_{k}\|^2\right]\leq \psi_{k}+\sum_{j=k_0+1}^{k}\Lambda_j^k (z_{j-1}-z_j).
\end{split}
\end{equation}
\end{lemma}

\begin{proof}
We prove by the induction method. Assuming (\ref{result}) is true and combining (\ref{key_eqn}), (\ref{lemma:3-a}) and Lemma.\ref{lemma:2}, we have that 
\begin{equation*}
\small
\begin{split}
    \E\left[\|\bm{T}_{k+1}\|^2\right]&\leq (1-2\omega_{k+1}+\omega^2_{k+1} G)(\psi_{k}+\sum_{j=k_0+1}^{k}\Lambda_j^k (z_{j-1}-z_j))\\
    &\ \ +\omega^2_{k+1} C_0+2Q \delta(\btheta_{k}, \bx_{k+1})\omega_{k+1}+2\omega_{k+1}\E[z_{k}-z_{k+1}]\\
    & \leq \psi_{k+1}+(1-2\omega_{k+1}+\omega^2_{k+1} G)\sum_{j=k_0+1}^{k}\Lambda_j^k (z_{j-1}-z_j)+2\omega_{k+1}\E[z_{k}-z_{k+1}]\\
    & \leq \psi_{k+1}+\sum_{j=k_0+1}^{k}\Lambda_j^{k+1} (z_{j-1}-z_j)+\Lambda_{k+1}^{k+1}\E[z_{k}-z_{k+1}]\\
    & \leq \psi_{k+1}+\sum_{j=k_0+1}^{k+1}\Lambda_j^{k+1} (z_{j-1}-z_j)\\
\end{split}
\end{equation*}
\end{proof}

The following lemma is a restatement of Lemma 25 (page 247) from \citet{Albert90}.
\begin{lemma}
\label{lemma:2}
Suppose $k_0$ is an integer satisfying
$\inf_{k> k_0} \dfrac{\omega_{k+1}-\omega_{k}}{\omega_{k}\omega_{k+1}}+2-G\omega_{k+1}>0$ 
for some constant $G$. 
Then for any $k>k_0$, the sequence $\{\Lambda_k^K\}_{k=k_0, \ldots, K}$ defined below is increasing and uppered bounded by $2\omega_{k}$
\begin{equation}  
\Lambda_k^K=\left\{  
             \begin{array}{lr}  
             2\omega_{k}\prod_{j=k}^{K-1}(1-2\omega_{j+1}+G\omega_{j+1}^2) & \text{if $k<K$},   \\  
              & \\
             2\omega_{k} &  \text{if $k=K$}.
             \end{array}  
\right.  
\end{equation} 
\end{lemma}

\bibliography{mybib}
\bibliographystyle{plainnat}

\maketitle

In this supplementary material, we review the related methodologies in $\S$\ref{review}, show the convergence of latent variable in $\S$\ref{convergence} and prove the ergodicity and weighted averaging estimators in $\S$\ref{ergodic}.
\section{Background on Stochastic Approximation and Poisson Equation}
\label{review}

\subsection{Stochastic approximation}
Stochastic approximation \citep{RobbinsM1951, Albert90} is a standard framework for formulating adaptive algorithms. Given a random field of $H(\bm{\btheta}, \bm{\bx})$ with respect to $\bm{\bx}$, the goal is to find the equilibrium $\btheta$ for the mean-field function $h(\btheta)$ such that
\begin{equation*}
\begin{split}
\label{sa00}
h(\btheta)&=\int_{\bchi} H(\bm{\theta}, \bm{\bx})\varpi_{\bm{\theta}}(d\bm{\bx})=0,
\end{split}
\end{equation*}
where $\bx\in \bchi \subset \mathbb{R}^d$, $\btheta\in\bTheta \subset \mathbb{R}^{m}$, 
$\varpi_{\btheta}(\bx)$ denotes a distribution parametrized by 
$\btheta$, and $H(\btheta,\bx)$ denotes a random field 
function. The algorithm works by iterating 
between the following two steps:
\begin{itemize}
\item[(1)] Simulate $\bm{x}_{k+1}$ from the transition kernel  $\Pi_{\bm{\theta_{k}}}(\bm{x}_{k}, \cdot)$, which admits $\varpi_{\bm{\theta}_{k}}(\bm{x})$ as
the invariant distribution,

\item[(2)] Update $\btheta_k$ by setting $\bm{\theta}_{k+1}=\bm{\theta}_{k}+\omega_{k+1} H(\bm{\theta}_{k}, \bm{x}_{k+1})+\omega_{k+1}^2 \rho(\bm{\theta}_{k}, \bm{x}_{k+1}),$
where $\rho(\cdot,\cdot)$ denotes a bias or oscillation term. 
\end{itemize}

The algorithm differs from the Robbins–Monro algorithm in that we simulate $\bx$ from a transition kernel $\Pi_{\bm{\theta_{k}}}(\cdot, \cdot)$ instead of a distribution $\varpi_{\bm{\theta}_{k}}(\cdot)$. As a result, a Markov state-dependent noise $H(\btheta_k, \bx_{k+1})-h(\btheta_k)$ is generated, which requires some regularity conditions to control the fluctuations $\sum_k \Pi_{\btheta}^k (H(\btheta, \bx)-h(\btheta))$. Moreover, it supports a more general form where bounded oscillations $\rho(\cdot,\cdot)$ are allowed without affecting the theoretical properties. 

\subsection{Poisson equation}

Stochastic approximation generates a nonhomogeneous Markov chain $\{(\bx_k, \btheta_k)\}_{k=1}^{\infty}$, of which the convergence theory can be studied by the Poisson equation 
\begin{equation*}
    \mu_{\btheta}(\bm{x})-\mathrm{\Pi}_{\bm{\theta}}\mu_{\bm{\theta}}(\bm{x})=H(\bm{\theta}, \bm{x})-h(\bm{\theta}).
\end{equation*}
where $\Pi_{\bm{\theta}}(\bm{x}, A)$ is the transition kernel for any Borel subset $A\subset \bchi$ and $\mu_{\btheta}(\cdot)$ is a function on $\bchi$.
The existence of the solution of Poisson equation can be identified when the series converges.
\begin{equation*}
    \mu_{\btheta}(\bx):=\sum_k \Pi_{\btheta}^k (H(\btheta, \bx)-h(\btheta)).
\end{equation*}
In other words, the consistency of the estimator $\btheta$ can be established by controlling the perturbations of $\sum_k \Pi_{\btheta}^k (H(\btheta, \bx)-h(\btheta))$ via some regularity conditions on $\mu_{\btheta}(\cdot)$. To avoid studying the individual algorithms, \citet{Albert90} has simplified the work to the justification of the regularity conditions on $\mu_{\btheta}(\cdot)$:

There exist a function  $V: \mX \to [1,\infty)$, and a constant $C$ such that for all $\bm{\theta}, \bm{\theta}'\in \bm{\bTheta}$, we have
\begin{equation*}
\begin{split}
\|\mathrm{\Pi}_{\bm{\theta}}\mu_{\btheta}(\bx)\|&\leq C V(\bx),\quad
\|\mathrm{\Pi}_{\bm{\theta}}\mu_{\bm{\theta}}(\bx)-\mathrm{\Pi}_{\bm{\theta'}}\mu_{\bm{\theta'}}(\bx)\|\leq C\|\bm{\theta}-\bm{\theta}'\| V(\bx),  \quad 
\E[V(\bx)]\leq \infty.\\
\end{split}
\end{equation*}

In particular, only the first order smoothness is required for the convergence of the adaptive algorithms \citep{Albert90}. By contrast, we see that the ergodicity theory \citep{mattingly10, VollmerZW2016} relies on a much stronger 4th-order smoothness.

% Poisson equation has been widely used in ergodic theory and adaptive algorithms to prove the desired limit of a time-average. Consider the infinitesimal generator $\mathcal{L}$ of the overdamped Langevin diffusion and let $\phi$ solve the Poisson equation
% \begin{equation}
%     \mathcal{L}\phi(\bx):=g-\bar g, 
% \end{equation}
% where $g$ is a test function and $\bar g$ is the expectation of $g$ over the Gibbs measure, defined as $\bar g=\int_{\mX}g(\bx) \varpi(d\bx)$. It is known that in a d-dimensional torus $\mathbb{T}^d$ and under elliptic settings, there is a unique solution for the Poisson equation, which is at least k+2-order smooth given a k-order smooth test function $g$ \citep{mattingly10}. To extend the ergodic average from $\mathbb{T}^d$ to $\mathbb{R}^d$, \citet{VollmerZW2016} established the required assumptions to establish the existence of smooth solutions of Poisson equation for stochastic gradient Langevin dynamics.

\section{Convergence Analysis for CSGLD} \label{convergence}

For notational simplicity, we define  
$\nabla_{\bx} \tilde{L}(\bx,\btheta)= \frac{N}{n} \left[1+ 
   \frac{\zeta\tau}{\Delta u}  \left( \frac{\theta((\tilde{J}(\bx)+1)\wedge m)}{\theta(\tilde{J}(\bx))} -1 \right) \right]  
    \nabla_{\bx} \widetilde U(\bx)$. To make 
 our theory more general, 
 we generalize Algorithm a little bit by 
 including a higher order oscillation term in 
 parameter updating. The resulting algorithm works as follows:
\begin{itemize}
\item[(1)] Sample $\bm{x}_{k+1}=\bx_k- \epsilon_k\nabla_{\bx} \widetilde L(\bx_k, \btheta_k)+\mathcal{N}({0, 2\epsilon_k \tau\bm{I}}), \ \ \ \ \ \ \ \ \ \ \ \ \ \ \ \ \ \ \ \ \ \ \ \ \ \ \ \ \ \ \ \ \ \ \ \ \ \ \ \ \ \ \ \ \ \ \ \ \ \ \ \ \ \ \ \ \ \ \ \ (\text{S}_1)$

\item[(2)] Update $\bm{\theta}_{k+1}=\bm{\theta}_{k}+\omega_{k+1} \widetilde H(m,n, \epsilon_k, \bm{\theta}_{k}, \bm{x}_{k+1})
+\omega_{k+1}^2 \rho(\bm{\theta}_{k}, \bm{x}_{k+1}),
\ \ \ \ \ \ \ \ \ \ \ \ \ \ \ \ \ \ \ \ \ \ \ \  (\text{S}_2)$
\end{itemize}
where $\epsilon_k$ is the learning rate, $\omega_{k+1}$ is the step size, $\nabla_{\bx} \widetilde L(\bx_k, \btheta_k)$ is an unbiased estimate of $\nabla_{\bx} L(\bx_k, \btheta_k)= \left[1+
   \frac{\zeta\tau}{\Delta u}  \left( \frac{\theta_k((\tilde{J}(\bx_k)+1)\wedge m)}{\theta_k(\tilde{J}(\bx_k))} -1 \right) \right]  
    \nabla_{\bx} U(\bx_k)$ for any given $\btheta_k$, 
    and $\rho(\btheta_k,\bx_{k+1})$ is an oscillation term. 
    The stochastic random field $\widetilde H(m, n, \epsilon_k, \bm{\theta}_{k}, \bm{x}_{k+1})$ is an estimator of the random field $ H(\bm{\theta}_{k}, \bm{x}_{k+1})$  
    and it is subject to 
    discretization and mini-batch approximation 
    errors. Formally, we can decompose 
    $\widetilde H(m,n, \epsilon_k, \bm{\theta}_{k}, \bm{x}_{k+1})$ as follows:
\begin{equation}
    \label{tildeh}
    \widetilde H(m,n, \epsilon_k, \bm{\theta}_{k}, \bm{x}_{k+1})= H(\bm{\theta}_{k}, \bm{x}_{k+1})+\delta(m,n, \epsilon_k, \btheta_{k}, \bx_{k+1}),
\end{equation}
where the bias $\delta(m, n, \epsilon_k, \btheta_{k}, \bx_{k+1})$ is a random vector generated from two aspects: the discretization error $\delta_1(\frac{1}{m}, \epsilon_k)$ of SGLD ($\text{S}_1$) due to approximating the underlying Langevin diffusion \citep{Issei14, Maxim17}, which depends on $\epsilon_k$ and $m$; the estimation error $\delta_2(n, \btheta_k, \bx_{k+1})$ due to  approximating $H(\bm{\theta}_{k}, \bm{x}_{k+1})$ using a mini-batch of data, which mainly depends on the batch size $n$. Notably, the bias $\delta_2$ becomes smaller given a larger batch size and vanishes when the full data is used.

  %Take an example for the latter, when $H(\bm{\theta}_{k}, \bm{x}_{k+1})=L^2(\bx_{k+1}, \btheta_k)$ and $\widetilde H(\epsilon_k, \bm{\theta}_{k}, \bm{x}_{k+1})=(L(\bx_{k+1}, \btheta_k)+\mathcal{N}(0,\sigma^2))^2$, it is clear that $\E[\widetilde H(\epsilon_k, \bm{\theta}_{k}, \bm{x}_{k+1})]=L^2(\bx_{k+1}, \btheta_k)+\sigma^2\neq H(\bm{\theta}_{k}, \bm{x}_{k+1})$, which induced a fixed bias term $\sigma^2$. 

\subsection{Convergence of parameter estimation} 

The convergence analysis rests on the following assumptions:

\begin{assump}[Compactness] \label{ass2a} 
The space $\Theta$ is compact, $\inf \theta(i) >0$ for any $i\in \{1,2,\ldots,m\}$, 
%$\btheta, H(\btheta, \bx)$ and $\widetilde H(m, n, \epsilon, \btheta, \bx)$ in a compact space $\bTheta$ such that %$\inf_{\btheta} \btheta(1)>0$ and $\btheta_{\star}\in\bTheta$. Moreover, 
and there exists a constant $Q>0$ such that for any  $\btheta\in \bTheta$ and $\bx \in \mX$, 
\begin{equation}
\label{compactness}
     \|\btheta\|\leq Q, \quad  %\|\rho(\btheta, \bx)\|\leq Q, \|\delta(\btheta, \bx)\|\leq Q,
     \|H(\btheta, \bx)\|\leq Q, \quad 
     \|\widetilde H(m,n, \epsilon, \btheta, \bx)\|\leq Q, \quad 
     \|\rho(\btheta, \bx)\|\leq Q.
\end{equation}
\end{assump}

\begin{assump}[Smoothness]
\label{ass2}
$L(\bm{\xeta}, \bm{\theta})$ is $M$-smooth, namely, for any $\bx, \bx'\in \mX$, $\bm{\theta}, \bm{\theta}'\in \bTheta$,
\begin{equation}
\begin{split}
\label{ass_2_1_eq}
\|\nabla_{\bx} L(\bx, \btheta)-\nabla_{\bx} L(\bm{\bx}', \btheta')\|\leq M\|\bx-\bx'\|+M\|\btheta-\btheta'\|. 
\end{split}
\end{equation}
\end{assump}

\begin{assump}[Dissipativity]
\label{ass3}
 There exist constants $\tilde{m}>0$ and $\tilde{b}\geq 0$ such that for any $\bx \in \mX$ and $\btheta \in \bTheta$, 
\label{ass_dissipative}
\begin{equation}
\label{eq:01}
\langle \nabla_{\bx} L(\bx, \btheta), \bx\rangle\leq \tilde{b}-\tilde{m}\|\bx\|^2.
\end{equation}
\end{assump}
This assumption has been widely used in proving the geometric ergodicity of dynamical systems \citep{mattingly02, Maxim17, Xu18}. It ensures the sampler to move towards the origin regardless of the starting position.

\begin{assump}[Gradient noise] 
\label{ass4}
The stochastic noise follows
\begin{equation*}
\E[\nabla_{\bx}\widetilde L(\bx_{k},
 \btheta_{k})-\nabla_{\bx} L(\bx_{k}, \btheta_{k})]=0.
\end{equation*}
There exists some constants $M, B>0$ such that the second moment of the noise is bounded by
\begin{equation*} 
\E [\|\nabla_{\bx}\widetilde L(\bx_{k},
 \btheta_{k})-\nabla_{\bx} L(\bx_{k}, \btheta_{k})\|^2]\leq M^2 \|\bx\|^2+B^2. 
\end{equation*}

\end{assump}

We proceed by proving that $\mathbb{V}(\btheta)=\frac{1}{2}\|\btheta_{\star}-\btheta\|^2$ is a Lyapunov function for the mean-field $h(\btheta)$.

\begin{lemma}[Stability] \label{convex_appendix}
The mean-field $h(\btheta)$ satisfies $\forall \btheta \in \bTheta$, $\langle h(\btheta), \btheta - \btheta_{\star}\rangle \leq  -\phi\|\btheta - \btheta_{\star}\|^2$ for a constant $\phi>0$, where $\theta_*=(\Psi(u_1),\Psi(u_2),\ldots,\Psi(u_m))$.
%$\btheta_{\star}$ is a stable point he globally asymptotically stable equilibrium and
In addition, $\{\btheta\in\bTheta:\langle h(\btheta), \nabla \mathbb{V}(\btheta)\rangle=0\}=\{\btheta_{\star}\}$, i.e., 
$\btheta_*$ is the only stable point of the mean-field system.
\end{lemma}

\begin{proof}

Given the random field $H_i(\btheta,\bx)={\theta}^{\zeta}(J(\bx))\left(1_{i\geq J(\bx)}-{\theta}(i)\right)$ for $i\in\{1,2,....m\}$, the mean-field function $h(\btheta)$ under the density 
$\varpi_{\btheta}(\bx)$ given in XX  follows
\begin{equation}
\begin{split}
    h_i(\btheta)&=\int_{\mX} H_i(\btheta,\bx) 
     \varpi_{\btheta}(\bx) d\bx
    =\int_{\mX} {\theta}^{\zeta}(J(\bx))\left(1_{i\geq J(\bx)}-{\theta}(i)\right) \varpi_{\btheta}(\bx) d\bx\\
    &=Z_{\btheta}^{-1}\left[ \sum_{k=1}^m \int_{\mX_k} 
     \pi(\bx) 1_{k\leq i} d\bx -\theta(i)\sum_{k=1}^m\int_{\mX_k} \pi(\bx)d\bx \right] \\
    &=  Z_{\btheta}^{-1}[\Psi(u_i)-\theta(i)] =
    Z_{\btheta}^{-1} \left[ \theta_{\star}(i)-\theta(i) \right].
\end{split}
\end{equation}
Considering the positive definite Lyapunov function $\mathbb{V}(\btheta)=\frac{1}{2}\|\btheta_{\star}-\btheta\|^2$ for the mean-field system $h(\btheta)=Z_{\btheta}^{-1} (\btheta_{\star}-\btheta)$, we have
\begin{equation*}
    \langle h(\btheta), \mathbb{V}(\btheta)\rangle=\langle h(\btheta), \btheta -\btheta_{\star}\rangle = -Z_{\btheta}^{-1}\|\btheta - \btheta_{\star}\|^2\leq  -\phi\|\btheta - \btheta_{\star}\|^2,
\end{equation*}
where $\phi=\inf_{\btheta} Z_{\btheta}^{-1}>0$ by
the compactness assumption \ref{ass2a}.  Therefore, $\langle h(\btheta), \mathbb{V}(\btheta)\rangle=\langle h(\btheta), \btheta -\btheta_{\star}\rangle\leq-\phi\left(\hat\btheta_{\star}-\btheta\right)^2<0$ holds for any $\btheta \neq \hat\btheta_{\star}$. This shows that the mean-field system is 
% globally asymptotically 
stable and $\btheta_{\star}$ is the only stable point.
\end{proof}

%  For a function $V: \mX \to [1,\infty)$ and a function $q: \mX \to \mR^{m}$, 
%  define the norm
% \[
%  \|q\|_V=\sup_{\bx\in \mX} \frac{\|q(\bx)\|}{V(\bx)},
%  \]
%  where $m$ denotes the dimension of $\btheta$. 
%  Let $\mathcal{L}_V=\{q: \mX \to \mR^{m}, \sup_{x\in \mX} \|q\|_V <\infty\}$.

\begin{assump}[Step size]
\label{ass1}
$\{\omega_{k}\}_{k\in \mathrm{N}}$ is a positive decreasing sequence of real numbers such that
\begin{equation} \label{a1}
\omega_{k}\rightarrow 0, \ \ \sum_{k=1}^{\infty} \omega_{k}=+\infty,\ \  \lim_{k\rightarrow \infty} \inf 2\phi  \dfrac{\omega_{k}}{\omega_{k+1}}+\dfrac{\omega_{k+1}-\omega_{k}}{\omega^2_{k+1}}>0.
\end{equation}
According to \citet{Albert90}, we can choose $\omega_{k}:=\frac{A}{k^{\alpha}+B}$ for some $\alpha \in (\frac{1}{2}, 1]$ and some suitable constants 
 $A>0$ and $B>0$. 
 \end{assump}

The following lemma is a restatement of Lemma 1 in \citet{deng2019}, for which the required conditions are 
clearly satisfied due to the compactness assumption \ref{ass2a}.
\begin{lemma}[Uniform $L^2$ bounds]
\label{lemma:1}
Suppose Assumptions \ref{ass2a}-\ref{ass1} holds.  Given a small enough learning rate
 $\ 0<\epsilon<\operatorname{Re}(\tfrac{\tilde{m}-\sqrt{\tilde{m}^2-3M^2}}{3M^2})\wedge 1$,  then 
$\sup_{k\geq 1} \E[\|\bm{\xeta}_{k}\|^2] < \infty$.
\end{lemma}

\begin{lemma}[Solution of Poisson equation]
\label{lyapunov}
Suppose that Assumptions A2-A4 hold. 
There is a solution $\mu_{\btheta}(\cdot)$ on $\mX$ to the Poisson equation 
\begin{equation}
    \label{poisson_eqn}
    \mu_{\btheta}(\bm{x})-\mathrm{\Pi}_{\bm{\theta}}\mu_{\bm{\theta}}(\bm{x})=H(\bm{\theta}, \bm{x})-h(\bm{\theta}).
\end{equation}
such that for all $\bm{\theta}, \bm{\theta}'\in \bm{\bTheta}$ and a function  $V(\bx)=1+\|\bx\|^2$, there exists a constant $C$ such that
\begin{equation}
\begin{split}
\label{poisson_reg}
\E[\|\mathrm{\Pi}_{\bm{\theta}}\mu_{\btheta}(\bx)\|]&\leq C,\\
\E[\|\mathrm{\Pi}_{\bm{\theta}}\mu_{\bm{\theta}}(\bx)-\mathrm{\Pi}_{\bm{\theta}'}\mu_{\bm{\theta'}}(\bx)\|]&\leq C\|\bm{\theta}-\bm{\theta}'\|.\\
\end{split}
\end{equation}
\end{lemma}

\begin{proof}
The conditions of Theorem 13 of  \citet{VollmerZW2016} can be 
easily verified for AWSGLD given the assumptions A1-A4 and 
Lemma \ref{lemma:1}. The details are omitted. 
\end{proof}

Now we are ready to prove the first main result on the 
convergence of $\btheta_k$.
The technique lemmas are listed 
in Section \ref{Lemmasection}. 

% \textbf{Theorem 1} ($L^2$ convergence rate).
%\begin{theorem}[$L^2$ convergence rate]
%\label{latent_converge}
%Suppose Assumptions $\ref{ass2a}$-$\ref{ass1}$ hold. For a 
%sufficiently
%large value of $k_0$, a sufficiently small learning rate sequence  $\{\epsilon_k\}_{k=1}^{\infty}$,  and a sufficiently small
% step  size sequence $\{\omega_k\}_{k=1}^{\infty}$, 
% $\{\btheta_k\}_{k=1}^{\infty}$ converges to
% $\btheta_{\star}$ in $L_2$-norm 
%  such that
%\begin{equation*}
 %   \E\left[\|\bm{\theta}_{k}-\btheta_{\star}\|^2\right]=\mathcal{O}( \omega_{k}+\sup_{i\geq k_0}\E[\|\delta(n, \epsilon_i, \btheta_{i}, \bx_{i+1})\|]).
%\end{equation*}
%\end{theorem}

\begin{theorem}[$L^2$ convergence rate]
\label{latent_convergence}
Suppose Assumptions $\ref{ass2a}$-$\ref{ass1}$ (given in the Appendix) hold. For a sufficiently
large value of $k_0$, a sufficiently small learning rate sequence  $\{\epsilon_k\}_{k=1}^{\infty}$,  and a sufficiently small
 step  size sequence $\{\omega_k\}_{k=1}^{\infty}$, 
 $\{\btheta_k\}_{k=1}^{\infty}$ converges to
 $\btheta_{\star}$ in $L_2$-norm  such that
\begin{equation*}
    \E\left[\|\bm{\theta}_{k}-\btheta_{\star}\|^2\right]=\mathcal{O}( \omega_{k}+\sup_{i\geq k_0}\E[\|\delta(m,n, \epsilon_i, \btheta_{i}, \bx_{i+1})\|]),
\end{equation*}
where $\delta(m,n, \epsilon_i, \btheta_{i},\bx_{i+1})$ denotes a bias term.
\end{theorem}
\begin{proof}
Consider the iterates 
\begin{equation*}
    \bm{\theta}_{k+1}=\bm{\theta}_{k}+\omega_{k+1} \left(\widetilde H(m,n, \epsilon_k, \bm{\theta}_{k}, \bm{x}_{k+1})+\omega_{k+1} \rho(\btheta_k, \bx_{k+1})\right).
\end{equation*}
Define $\bm{T}_{k}=\bm{\theta}_{k}-\btheta_{\star}$. By subtracting $\btheta_{\star}$ from both sides and taking the square and $L_2$ norm,  we have
\begin{equation*}
\small
\begin{split}
    \|\bT_{k+1}^2\|&=\|\bT_k^2\| +\omega_{k+1}^2 \|\widetilde H(m,n,\epsilon_k, \btheta_k, \bx_{k+1}) + \omega_{k+1}\rho(\btheta_k, \bx_{k+1})\|^2\\
    & +2\omega_{k+1}\underbrace{\langle \bT_k, \widetilde H(m,n,\epsilon_k, \btheta_k, \bx_{k+1})+\omega_{k+1}\rho(\btheta_k, \bx_{k+1})\rangle}_{\text{D}}.
\end{split}
\end{equation*}

First, by Lemma \ref{convex_property}, there exists a constant $G=9 Q^2(1+Q^2)$ such that
\begin{equation}
\label{first_term}
    \|\widetilde H(m,n,\epsilon_k, \btheta_k, \bx_{k+1}) + \omega_{k+1}\rho(\btheta_k, \bx_{k+1})\|^2 \leq G (1+\|\bT_k\|^2).
\end{equation}

Next, by the definition of $\widetilde H(m,n, \epsilon_k, \btheta_k, \bx_{k+1})$ in (\ref{tildeh}) and the Poisson equation (\ref{poisson_eqn}), we have
\begin{equation*}
\begin{split}
   \text{D}&=\langle \bT_k,  H(\btheta_k, \bx_{k+1})+\delta(m, n,\epsilon_k, \btheta_{k}, \bx_{k+1})+\omega_{k+1}\rho(\btheta_k, \bx_{k+1}) \rangle\\
   &=\langle \bT_k,  h(\btheta_k)+\mu_{\btheta_k}(\bm{x}_{k+1})-\mathrm{\Pi}_{\bm{\theta}_k}\mu_{\bm{\theta}_k}(\bm{x}_{k+1})+\delta(m,n,\epsilon_k, \btheta_{k}, \bx_{k+1})+\omega_{k+1}\rho(\btheta_k, \bx_{k+1}) \rangle\\
   &=\underbrace{\langle \bT_k,  h(\btheta_k)\rangle}_{\text{D}_{1}} +\underbrace{\langle\bT_k, \mu_{\btheta_k}(\bm{x}_{k+1})-\mathrm{\Pi}_{\bm{\theta}_k}\mu_{\bm{\theta}_k}(\bm{x}_{k+1})\rangle}_{\text{D}_{2}}+\underbrace{\langle \bT_k, \delta(m,n,\epsilon_k, \btheta_{k}, \bx_{k+1})+\omega_{k+1}\rho(\btheta_k, \bx_{k+1})\rangle}_{{\text{D}_{3}}}.
\end{split}
\end{equation*}

Using the stability property of the equilibrium in Lemma \ref{convex_appendix}, we have 
\begin{align*}
\langle \bm{T}_{k}, h(\bm{\theta}_{k})\rangle &\leq - \phi\|\bm{T}_{k}\|^2. \tag{$\text{D}_1$}
\end{align*}
To deal with the error $\text{D}_2$, we make the following decomposition 
\begin{equation*}
\begin{split}
\text{D}_2 &=\underbrace{\langle \bT_k, \mu_{\bm{\theta}_{k}}(\bm{\xeta}_{k+1})-\mathrm{\Pi}_{\bm{\theta}_{k}}\mu_{\bm{\theta}_{k}}(\bm{\bx}_{k})\rangle}_{\text{D}_{21}} \\
&+ \underbrace{\langle \bT_k,\mathrm{\Pi}_{\bm{\theta}_{k}}\mu_{\bm{\theta}_{k}}(\bm{x}_{k})- \mathrm{\Pi}_{\bm{\theta}_{k-1}}\mu_{\bm{\theta}_{k-1}}(\bm{x}_{k})\rangle}_{\text{D}_{22}}
+ \underbrace{\langle \bT_k,\mathrm{\Pi}_{\bm{\theta}_{k-1}}\mu_{\bm{\theta}_{k-1}}(\bm{x}_{k})- \mathrm{\Pi}_{\bm{\theta}_{k}}\mu_{\bm{\theta}_{k}}(\bm{\xeta}_{k+1})\rangle}_{\text{D}_{23}}.\\
\end{split}
\end{equation*}

\begin{itemize}
\item[(i)] From the Markov property, $\mu_{\bm{\theta}_{k}}(\bm{\xeta}_{k+1})-\mathrm{\Pi}_{\bm{\theta}_{k}}\mu_{\bm{\theta}_{k}}(\bm{x}_{k})$ forms a martingale difference sequence 
$$\E\left[\langle \bT_k, \mu_{\bm{\theta}_{k}}(\bm{\xeta}_{k+1})-\mathrm{\Pi}_{\bm{\theta}_{k}}\mu_{\bm{\theta}_{k}}(\bm{x}_{k})\rangle |\mathcal{F}_{k}\right]=0, \eqno{(\text{D}_{21})}$$
where  $\mF_k$ is a $\sigma$-filter formed by 
$\{\btheta_0,\bx_1, \btheta_1, \bx_2, \cdots, \bx_k,\btheta_k\}$.

\item[(ii)] By the regularity of the solution of Poisson equation in (\ref{poisson_reg}) and Lemma.\ref{theta_lip}, it leads to
\begin{equation}
\label{theta_delta}
\E[\|\mathrm{\Pi}_{\bm{\theta}_{k}}\mu_{\bm{\theta}_{k}}(\bm{x}_{k})- \mathrm{\Pi}_{\bm{\theta}_{k-1}}
 \mu_{\bm{\theta}_{k-1}}(\bm{x}_{k})\|]\leq C \|\btheta_k-\btheta_{k-1}\|\leq 2Q C\omega_k.
\end{equation}
Using Cauchy–Schwarz inequality, (\ref{theta_delta}) and the compactness of $\btheta$ in Assumption \ref{ass2a}, it follows that
$$\small{\E[\langle\bm{T}_{k},\mathrm{\Pi}_{\bm{\theta}_{k}}\mu_{\bm{\theta}_{k}}(\bm{x}_{k})- \mathrm{\Pi}_{\bm{\theta}_{k-1}}\mu_{\bm{\theta}_{k-1}}(\bm{x}_{k})\rangle]\leq \E[\|\bT_k\|]\cdot 2Q C\omega_k\leq 4Q^2 C\omega_{k}\leq 5Q^2 C\omega_{k+1}}   \eqno{(\text{D}_{22})},$$
where the last inequality follows from the step size assumption \ref{ass1} and holds for a sufficiently large value of $k$.

\item[(iii)] For the last term of $D_2$, 
\begin{equation*}
\begin{split}
\small
\langle \bm{T}_{k},\mathrm{\Pi}_{\bm{\theta}_{k-1}}\mu_{\bm{\theta}_{k-1}}(\bm{x}_{k})- \mathrm{\Pi}_{\bm{\theta}_{k}}\mu_{\bm{\theta}_{k}}(\bm{\xeta}_{k+1})\rangle
&=\left(\langle \bm{T}_{k}, \mathrm{\Pi}_{\bm{\theta}_{k-1}}\mu_{\bm{\theta}_{k-1}}(\bm{x}_{k}) \rangle- \langle \bm{T}_{k+1}, \mathrm{\Pi}_{\bm{\theta}_{k}}\mu_{\bm{\theta}_{k}}(\bm{\xeta}_{k+1})\rangle\right)\\
&\ \ \ +\left(\langle \bm{T}_{k+1}, \mathrm{\Pi}_{\bm{\theta}_{k}}\mu_{\bm{\theta}_{k}}(\bm{\xeta}_{k+1})\rangle-\langle \bm{T}_{k}, \mathrm{\Pi}_{\bm{\theta}_{k}}\mu_{\bm{\theta}_{k}}(\bm{\xeta}_{k+1})\rangle\right)\\
&={({z}_{k}-{z}_{k+1})}+{\langle \bm{T}_{k+1}-\bm{T}_{k}, \mathrm{\Pi}_{\bm{\theta}_{k}}\mu_{\bm{\theta}_{k}}(\bm{\xeta}_{k+1})\rangle},\\
\end{split}
\end{equation*}
where ${z}_{k}=\langle \bm{T}_{k}, \mathrm{\Pi}_{\bm{\theta}_{k-1}}\mu_{\bm{\theta}_{k-1}}(\bm{x}_{k})\rangle$. By the regularity assumption (\ref{poisson_reg}) and Lemma \ref{theta_lip}, 
$$\E\langle \bm{T}_{k+1}-\bm{T}_{k}, \mathrm{\Pi}_{\bm{\theta}_{k}}\mu_{\bm{\theta}_{k}}(\bm{\xeta}_{k+1})\rangle\leq   \E[\|\bm{\theta}_{k+1}-\bm{\theta}_{k}\|] \cdot \E[\|\mathrm{\Pi}_{\bm{\theta}_{k}}\mu_{\bm{\theta}_{k}}(\bm{\xeta}_{k+1})\|] \leq 2Q C \omega_{k+1}.\eqno{(\text{D}_{23})}$$
\end{itemize}

For convenience, we denote $\E[\|\delta(m,n,\epsilon_k, \btheta_{k}, \bx_{k+1})\|]$ by $\Delta_k$. Since $\rho(\btheta_k, \bx_{k+1})$ is bounded, we have
\begin{equation*}
    \E[\|\delta(m,n,\epsilon_k, \btheta_{k}, \bx_{k+1})+\omega_{k+1}\rho(\btheta_k, \bx_{k+1}))\|]\leq \Delta_k + \omega_{k+1} Q.
\end{equation*}
Applying Cauchy–Schwarz inequality gives
$${\E[\langle \bT_k, \delta(m,n,\epsilon_k, \btheta_{k}, \bx_{k+1})+\omega_{k+1}\rho(\btheta_k, \bx_{k+1}))]\leq 2Q(\Delta_k+\omega_{k+1} Q)} \eqno{(\text{D}_{3})}$$

Finally, adding (\ref{first_term}), $\text{D}_1$, $\text{D}_{2}$ and $\text{D}_3$ together, it follows that for a constant $C_0 = G+10Q^2C+4QC+4Q^2$,
\begin{equation}
\begin{split}
\label{key_eqn}
\E\left[\|\bm{T}_{k+1}\|^2\right]&\leq (1-2\omega_{k+1}\phi+G\omega^2_{k+1} )\E\left[\|\bm{T}_{k}\|^2\right]+C_0\omega^2_{k+1} +4Q\Delta_k\omega_{k+1} +2\E[z_{k}-z_{k+1}]\omega_{k+1}.
\end{split}
\end{equation}
Moreover, from (\ref{compactness}) and (\ref{poisson_reg}), $\E[|z_{k}|]$ is upper bounded by
\begin{equation}
\begin{split}
\label{condition:z}
\E[|z_{k}|]=\E[\langle \bm{T}_{k}, \mathrm{\Pi}_{\bm{\theta}_{k-1}}\mu_{\bm{\theta}_{k-1}}(\bm{x}_{k})\rangle]\leq \E[\|\bT_k\|]\E[\|\mathrm{\Pi}_{\bm{\theta}_{k-1}}\mu_{\bm{\theta}_{k-1}}(\bm{x}_{k})\|]\leq 2QC.
\end{split}
\end{equation}

According to Lemma $\ref{lemma:4}$, we can choose $\lambda_0$ and $k_0$ such that 
\begin{align*}
\E[\|\bm{T}_{k_0}\|^2]\leq \psi_{k_0}=\lambda_0 \omega_{k_0}+\frac{2Q}{\phi}\sup_{i\geq k_0}\Delta_{i},
\end{align*}
which satisfies the conditions ($\ref{lemma:3-a}$) and ($\ref{lemma:3-b}$) of Lemma $\ref{lemma:3-all}$. Applying Lemma $\ref{lemma:3-all}$ leads to
\begin{equation}
\begin{split}
\label{eqn:9}
\E\left[\|\bm{T}_{k}\|^2\right]\leq \psi_{k}+\E\left[\sum_{j=k_0+1}^{k}\Lambda_j^k \left(z_{j-1}-z_{j}\right)\right],
\end{split}
\end{equation}
where $\psi_{k}=\lambda_0 \omega_{k}+\frac{2Q}{\phi}\sup_{i\geq k_0}\Delta_{i}$ for all $k>k_0$. Based on ($\ref{condition:z}$) and the increasing condition of $\Lambda_{j}^k$ in Lemma $\ref{lemma:2}$, we have
\begin{equation}
\small
\begin{split}
\label{eqn:10}
&\E\left[\left|\sum_{j=k_0+1}^{k} \Lambda_j^k\left(z_{j-1}-z_{j}\right)\right|\right]
=\E\left[\left|\sum_{j=k_0+1}^{k-1}(\Lambda_{j+1}^k-\Lambda_j^k)z_j-2\omega_{k}z_{k}+\Lambda_{k_0+1}^k z_{k_0}\right|\right]\\
\leq& \sum_{j=k_0+1}^{k-1}2(\Lambda_{j+1}^k-\Lambda_j^k)QC+\E[|2\omega_{k} z_{k}|]+2\Lambda_k^k QC\\
\leq& 2(\Lambda_k^k-\Lambda_{k_0}^k)QC+2\Lambda_k^k QC+2\Lambda_k^k QC\\
\leq& 6\Lambda_k^k QC.
\end{split}
\end{equation}

Therefore, given $\psi_{k}=\lambda_0 \omega_{k}+\frac{2Q}{\phi}\sup_{i\geq k_0}\Delta_{i}$ that satisfies the conditions ($\ref{lemma:3-a}$), ($\ref{lemma:3-b}$) of Lemma $\ref{lemma:3-all}$, for any $k>k_0$, from ($\ref{eqn:9}$) and ($\ref{eqn:10}$), we have
\begin{equation*}
\E[\|\bm{T}_{k}\|^2]\leq \psi_{k}+6\Lambda_k^k QC=\left(\lambda_0+12QC\right)\omega_{k}+\frac{2Q}{\phi}\sup_{i\geq k_0}\Delta_{i}=\lambda \omega_{k}+\frac{2Q}{\phi}\sup_{i\geq k_0}\Delta_{i},
\end{equation*}
where $\lambda=\lambda_0+12QC$, $\lambda_0=\frac{4GQ\sup_{i\geq k_0} \Delta_i + 2C_0\phi}{C_1\phi}$, $\small{C_1=\lim \inf 2\phi \dfrac{\omega_{k}}{\omega_{k+1}}+\dfrac{\omega_{k+1}-\omega_{k}}{\omega^2_{k+1}}>0}$, $C_0=G+5Q^2C+2QC+2Q^2$ and $G=9 Q^2(1+Q^2)$.
\end{proof}

\subsection{Technical Lemmas} \label{Lemmasection}

\begin{lemma}
\label{convex_property}
Given $\sup\{\omega_k\}_{k=1}^{\infty}\leq 1$, there exists a constant $G=9 Q^2(1+Q^2)$ such that
\begin{equation} \label{bound2}
\| \widetilde H(m, n, \epsilon_k, \bm{\theta}_k, \bm{\xeta}_{k+1})+\omega_{k+1}\rho(\btheta_k, \bx_{k+1})\|^2 \leq G (1+\|\bm{\theta}_k-\btheta_{\star}\|^2). 
\end{equation}
\end{lemma}
\begin{proof}

According to the compactness condition in Assumption \ref{ass2a}, we have
\begin{equation}
\label{mid_1}
\|H(\bm{\theta}_k, \bm{\xeta}_{k+1})\|^2\leq Q^2 (1+\|\bm{\theta}_k\|^2) = 
 Q^2 (1+\|\bm{\theta}_k-\btheta_{\star}+\btheta_{\star}\|^2)\leq Q^2 (1+2\|\bm{\theta}_k-\btheta_{\star}\|^2+2Q^2).
\end{equation}

Therefore, using (\ref{mid_1}), we can show that for a constant $G=9Q^2(1+Q^2)$
\begin{equation*}
\small
\begin{split}
    &\ \ \ \|\widetilde H(m, n, \epsilon_k, \bm{\theta}_k, \bm{\xeta}_{k+1})+\omega_{k+1}\rho(\btheta_k, \bx_{k+1})\|^2 \\
    &\leq 3\|H(\bm{\theta}_k, \bm{\xeta}_{k+1})\|^2 + 3\|\delta(m, n, \epsilon_k, \btheta_{k}, \bx_{k+1})\|^2+3\omega_{k+1}^2 \|\rho(\btheta_k, \bx_{k+1})\|\\
    &\leq 3Q^2 (1+2\|\bm{\theta}_k-\btheta_{\star}\|^2+2Q^2) + 6Q^2\\
    &\leq 3Q^2 (3+3Q^2+(3+3Q^2)\|\bm{\theta}_k-\btheta_{\star}\|^2)\\
    &\leq G (1+\|\bm{\theta}_k-\btheta_{\star}\|^2).
\end{split}
\end{equation*}
\end{proof}

\begin{lemma}
\label{theta_lip}Given $\sup\{\omega_k\}_{k=1}^{\infty}\leq 1$, we have that
\begin{equation}
\label{lip_theta}
    \|\btheta_{k}-\btheta_{k-1}\|\leq 2\omega_{k} Q
\end{equation}
\end{lemma}

\begin{proof}
Following the update $\btheta_k-\btheta_{k-1}=\omega_k \widetilde H(m, n, \epsilon_{k-1},\bm{\theta}_{k-1}, \bm{x}_{k})+\omega_{k}^2 \brho_{k}$, we have that
$$\|\btheta_{k}-\btheta_{k-1}\|= \|\omega_k \widetilde H(m, n, \epsilon_{k-1}, \bm{\theta}_{k-1}, \bm{x}_{k})+\omega_{k}^2 \brho_{k}\|\leq \omega_k\| \widetilde H(m, n, \epsilon_k, \bm{\theta}_{k-1},\bm{x}_{k})\|+\omega_{k}^2\| \brho_{k}\|.$$
By the compactness condition in Assumption \ref{ass2a} and $\sup\{\omega_k\}_{k=1}^{\infty}\leq 1$, (\ref{lip_theta}) can be derived.
\end{proof}

\begin{lemma}
\label{lemma:4}
There exist constants $\lambda_0$ and $k_0$ such that $\forall \lambda\geq\lambda_0$ and $\forall k> k_0$, the sequence $\{\psi_{k}\}_{k=1}^{\infty}$, where $\psi_{k}=\lambda\omega_{k}+\frac{2Q}{\phi} \sup_{i\geq k_0}\Delta_i$, satisfies
\begin{equation}
\begin{split}
\label{key_ieq}
\psi_{k+1}\geq& (1-2\omega_{k+1}\phi+G\omega_{k+1}^2)\psi_{k}+C_0\omega_{k+1}^2  +4Q \Delta_k\omega_{k+1}.
\end{split}
\end{equation}
\begin{proof}
By replacing $\psi_{k}$ with $\lambda\omega_{k}+\frac{2Q}{\phi} \sup_{i\geq k_0}\Delta_i$ in ($\ref{key_ieq}$), it suffices to show
\begin{equation*}
\small
\begin{split}
\label{lemma:loss_control}
\lambda \omega_{k+1}+\frac{2Q}{\phi} \sup_{i\geq k_0}\Delta_i\geq& (1-2\omega_{k+1}\phi+G\omega_{k+1}^2)\left(\lambda \omega_{k}+\frac{2Q}{\phi} \sup_{i\geq k_0}\Delta_i\right)+C_0\omega_{k+1}^2 + 4Q\Delta_k\omega_{k+1}.
\end{split}
\end{equation*}

which is equivalent to proving
\begin{equation*}
\small
\begin{split}
&\lambda (\omega_{k+1}-\omega_k+2\omega_k\omega_{k+1}\phi-G\omega_k\omega_{k+1}^2)\geq  \frac{2Q}{\phi}\sup_{i\geq k_0}\Delta_i(-2\omega_{k+1}\phi+G\omega_{k+1}^2 )+C_0\omega_{k+1}^2+ 4Q\Delta_k\omega_{k+1}.
\end{split}
\end{equation*}

Given the step size condition in ($\ref{a1}$), we have $$\small{\omega_{k+1}-\omega_{k}+2 \omega_{k}\omega_{k+1}\phi \geq C_1 \omega_{k+1}^2},$$ 
where $\small{C_1=\lim \inf 2\phi  \dfrac{\omega_{k}}{\omega_{k+1}}+\dfrac{\omega_{k+1}-\omega_{k}}{\omega^2_{k+1}}>0}$. Combining $-\sup_{i\geq k_0}\Delta_i\leq \Delta_k$, it suffices to prove
\begin{equation}
\begin{split}
\label{loss_control-2}
\lambda \left(C_1-G\omega_{k}\right)\omega^2_{k+1}\geq  \left(\frac{2GQ}{\phi} \sup_{i\geq k_0}\Delta_i+C_0\right)\omega^2_{k+1}.
\end{split}
\end{equation}

It is clear that for a large enough $k_0$ and $\lambda_0$ such that $\omega_{k_0}\leq \frac{C_1}{2G}$, $\lambda_0=\frac{4GQ\sup_{i\geq k_0} \Delta_i + 2C_0\phi}{C_1\phi}$, the desired conclusion ($\ref{loss_control-2}$) holds for all such $k\geq k_0$ and $\lambda\geq \lambda_0$.
\end{proof}
\end{lemma}

The following lemma is a restatement of Lemma 25 (page 247) from \citet{Albert90}.
\begin{lemma}
\label{lemma:2}
Suppose $k_0$ is an integer satisfying
$\inf_{k> k_0} \dfrac{\omega_{k+1}-\omega_{k}}{\omega_{k}\omega_{k+1}}+2\phi-G\omega_{k+1}>0$ 
for some constant $G$. 
Then for any $k>k_0$, the sequence $\{\Lambda_k^K\}_{k=k_0, \ldots, K}$ defined below is increasing and uppered bounded by $2\omega_{k}$
\begin{equation}  
\Lambda_k^K=\left\{  
             \begin{array}{lr}  
             2\omega_{k}\prod_{j=k}^{K-1}(1-2\omega_{j+1}\phi+G\omega_{j+1}^2) & \text{if $k<K$},   \\  
              & \\
             2\omega_{k} &  \text{if $k=K$}.
             \end{array}  
\right.  
\end{equation} 
\end{lemma}

\begin{lemma}
\label{lemma:3-all}
Let $\{\psi_{k}\}_{k> k_0}$ be a series that satisfies the following inequality for all $k> k_0$
\begin{equation}
\begin{split}
\label{lemma:3-a}
\psi_{k+1}\geq &\psi_{k}\left(1-2\omega_{k+1}\phi+G\omega^2_{k+1}\right)+C_0\omega^2_{k+1} + 4Q \Delta_k\omega_{k+1},
\end{split}
\end{equation}
and assume there exists such $k_0$ that 
\begin{equation}
\begin{split}
\label{lemma:3-b}
\E\left[\|\bm{T}_{k_0}\|^2\right]\leq \psi_{k_0}.
\end{split}
\end{equation}
Then for all $k> k_0$, we have
\begin{equation}
\begin{split}
\label{result}
\E\left[\|\bm{T}_{k}\|^2\right]\leq \psi_{k}+\sum_{j=k_0+1}^{k}\Lambda_j^k (z_{j-1}-z_j).
\end{split}
\end{equation}
\end{lemma}

\begin{proof}
We prove by the induction method. Assuming (\ref{result}) is true and applying (\ref{key_eqn}), we have that 
\begin{equation*}
\small
\begin{split}
    \E\left[\|\bm{T}_{k+1}\|^2\right]&\leq (1-2\omega_{k+1}\phi+\omega^2_{k+1} G)(\psi_{k}+\sum_{j=k_0+1}^{k}\Lambda_j^k (z_{j-1}-z_j))+C_0\omega^2_{k+1} +4Q \Delta_k\omega_{k+1}+2\omega_{k+1}\E[z_{k}-z_{k+1}]\\
\end{split}
\end{equation*}

Combining (\ref{key_ieq}) and Lemma.\ref{lemma:2}, respectively, we have
\begin{equation*}
\small
\begin{split}
    \E\left[\|\bm{T}_{k+1}\|^2\right]&\leq  \psi_{k+1}+(1-2\omega_{k+1}\phi+\omega^2_{k+1} G)\sum_{j=k_0+1}^{k}\Lambda_j^k (z_{j-1}-z_j)+2\omega_{k+1}\E[z_{k}-z_{k+1}]\\
    & \leq \psi_{k+1}+\sum_{j=k_0+1}^{k}\Lambda_j^{k+1} (z_{j-1}-z_j)+\Lambda_{k+1}^{k+1}\E[z_{k}-z_{k+1}]\\
    & \leq \psi_{k+1}+\sum_{j=k_0+1}^{k+1}\Lambda_j^{k+1} (z_{j-1}-z_j).\\
\end{split}
\end{equation*}
\end{proof}

\section{Ergodicity and weighted averaging estimators}
\label{ergodic}
Our interest is to analyze the deviation between the weighted averaging estimator $\frac{1}{k}\sum_{i=1}^k\theta_{i}^{\zeta}( \tilde{J}(\bx_i)) f(\bx_i)$ and posterior average $\int_{\mX}f(\bx)\pi(d\bx)$ for an integrable function $f$. To accomplish this analysis, we first study the convergence of the 
posterior sample mean $\frac{1}{k}\sum_{i=1}^k f(\bx_i)$ 
to the posterior expectation $\bar{f}=\int_{\mX}f(\bx)\varpi_{\btheta_{\star}}(d\bx)$. 
The key tool for ergodic theory is still the Poisson equation 
which is used to characterize the fluctuation 
between $f(\bx)$ and $\bar f$: 
\begin{equation}
    \mathcal{L}g(\bx)=f(\bx)-\bar f,
\end{equation}
where $g(\bx)$ is the solution of the Poisson equation, and $\mathcal{L}$ is the infinitesimal generator of the Langevin diffusion 
\begin{equation*}
    \mathcal{L}g:=\nabla g \nabla L(\cdot, \btheta_{\star})+\tau\nabla^2g.
\end{equation*}

By imposing the following regularity conditions on the function $g(\bx)$, we can control the perturbations of $\frac{1}{k}\sum_{i=1}^k f(\bx_i)-\bar f$ and enables convergence of the ergodic average.
\begin{assump}[Regularity] 
\label{ass6}
Given a sufficiently smooth function $g(\bx)$ and a function $\mathcal{V}(\bx)$, such that $\|D^k g\|\lesssim \mathcal{V}^{p_k}(\bx)$ and $p_k>0$ for $k\in\{0,1,2,3\}$. In addition, $\mathcal{V}^p$ has a bounded expectation: $\sup_{\bx} \E[\mathcal{V}^p(\bx)]<\infty$ and $\mathcal{V}$ is smooth, i.e. $\sup_{s\in\{0, 1\}} \mathcal{V}^p(s\bx+(1-s)\by)\lesssim \mathcal{V}^p(\bx)+\mathcal{V}^p(\by)$ for all $\bx,\by\in\mX$ and $p\leq 2\max_k\{p_k\}$.
\end{assump}

Now, we present a lemma, which is majorly adapted from Theorem 2 \citep{Chen15} with a fixed learning rate $\epsilon$. For stronger but verifiable conditions, we refer readers to \citet{VollmerZW2016}.
\begin{lemma}[Convergence of the Averaging Estimators]
\label{avg_converge_appendix}
Assume Assumptions $\ref{ass2a}$-$\ref{ass6}$ hold.  For any integrable function $f$, we have
\begin{equation*}
\small
\begin{split}
    \left|\E\left[\frac{\sum_{i=1}^k f(\bx_i)}{k}\right]-\int_{\mX}f(\bx)\varpi_{\bm{\theta}_{\star}}(\bx)d\bx\right|&= \mathcal{O}\left(\frac{1}{k\epsilon}+\epsilon+\sqrt{\frac{\sum_{i=1}^k \omega_k}{k}}+\sup_{i\geq k_0}\E[\|\delta(m,n, \epsilon, \btheta_{i}, \bx_{i+1})\|]^{0.5}\right). \\
\end{split}
\end{equation*}
\end{lemma}

\begin{proof}

To study the ergodic average, we can view the adaptive algorithm as a standard sampling algorithm with fixed latent variable $\btheta_{\star}$ as follows
\begin{equation*}
\begin{split}
    \bm{x}_{k+1}&=\bx_k- \epsilon_k\nabla_{\bx} \widetilde L(\bx_k, \btheta_k)+\mathcal{N}({0, 2\epsilon_k \tau\bm{I}})\\
    &=\bx_k- \epsilon_k\left(\nabla_{\bx} \widetilde L(\bx_k, \btheta_{\star})+\Upsilon(\bx_k,\btheta_{\star})\right)+\mathcal{N}({0, 2\epsilon_k \tau\bm{I}}),
\end{split}
\end{equation*}
where the bias term is denoted by $$\Upsilon(\bx_k,\btheta_k)=\nabla L(\bx_k,\btheta_k)-\nabla L(\bx_k,\btheta_{\star}).$$

According to the smoothness assumption \ref{ass2} and Jensen's inequality, we have
\begin{equation}
\label{latent_bias}
\small
    \|\E[\Upsilon(\bx_k,\btheta_k)]\|\leq \E[\|\Upsilon(\bx_k,\btheta_k)\|]=\E[\|\nabla_{\bx} L(\bx, \btheta_k)-\nabla_{\bx}  L(\bx, \btheta_{\star})\|] \leq M\E[\|\btheta_k-\btheta_{\star}\|] \leq M\sqrt{\E[\|\btheta_k-\btheta_{\star}\|^2]}.
\end{equation}

The ergodic average based on biased gradients and a fixed learning rate is a direct result of Theorem 2 \citep{Chen15} by imposing regularity condition \ref{ass6}. Combining (\ref{latent_bias}) and Theorem \ref{latent_convergence}, we know that 
\begin{equation*}
\small
\begin{split}
    \left|\E\left[\frac{\sum_{i=1}^k f(\bx_i)}{k}\right]-\int_{\mX}f(\bx)\varpi_{\bm{\theta}_{\star}}(\bx)d\bx\right|&\leq \mathcal{O}\left(\frac{1}{k\epsilon}+\epsilon+\frac{\sum_{i=1}^k \|\E[\Upsilon(\bx_k,\btheta_k)]\|}{k}\right)\\
    &\lesssim \mathcal{O}\left(\frac{1}{k\epsilon}+\epsilon+\frac{\sum_{i=1}^k \left(\omega_k+\sup_{i\geq k_0}\E[\|\delta(m, n, \epsilon, \btheta_{i}, \bx_{i+1})\|]\right)^{0.5}}{k}\right), \\
    % &\leq \mathcal{O}\left(\frac{1}{k\epsilon}+\epsilon+\frac{\sum_{i=1}^k \omega_k^{0.5}}{k}+\sup_{i\geq k_0}\E[\|\delta(m, n, \epsilon, \btheta_{i}, \bx_{i+1})\|]^{0.5}\right) \\
    &\leq \mathcal{O}\left(\frac{1}{k\epsilon}+\epsilon+\sqrt{\frac{\sum_{i=1}^k \omega_k}{k}}+\sup_{i\geq k_0}\E[\|\delta(m, n, \epsilon, \btheta_{i}, \bx_{i+1})\|]^{0.5}\right)
\end{split}
\end{equation*}
where the last inequality follows from $\small{(\omega_k+\Delta)^{0.5}\leq \omega_k^{0.5}+\Delta^{0.5}}$ and $\small{\sum_{i=1}^k \omega_i^{0.5}\leq \sqrt{k\sum_{i=1}^k \omega_i}}$.
\end{proof}

Now we are ready to show the convergence of the weighted averaging estimator $\frac{\sum_{i=1}^k\theta_{i}
     ^{\zeta}(\tilde{J}(\bx_i)) f(\bx_i)}{\sum_{i=1}^k\theta_{i}^{\zeta}( 
      \tilde{J}(\bx_i))}$ to the posterior mean $\int_{\mX}f(\bx)\pi(d\bx)$.
\begin{theorem}[Convergence of the Weighted Averaging Estimators] Assume Assumptions $\ref{ass2a}$-$\ref{ass6}$ hold. For any bounded function $f$, we have that 
\label{w_avg_converge_appendix}
\begin{equation*}
\small
\begin{split}
    \left|\E\left[\frac{\sum_{i=1}^k\theta_{i}
     ^{\zeta}(\tilde{J}(\bx_i)) f(\bx_i)}{\sum_{i=1}^k\theta_{i}^{\zeta}( 
      \tilde{J}(\bx_i))}\right]-\int_{\mX}f(\bx)\pi(d\bx)\right|&= \mathcal{O}\left(\frac{1}{k\epsilon}+\epsilon+\sqrt{\frac{\sum_{i=1}^k \omega_k}{k}}+\sup_{n\geq k_0}\E[\|\delta(m, n, \epsilon, \btheta_{n}, \bx_{n+1})\|]^{0.5}\right). \\
\end{split}
\end{equation*}
\end{theorem}

\begin{proof}

Applying triangle inequality and $|\E[x]|\leq \E[|x|]$, we have
\begin{equation*}
\small
    \begin{split}
        &\left|\E\left[\frac{\sum_{i=1}^k\theta_{i}
        ^{\zeta}( \tilde{J}(\bx_i)) f(\bx_i)}{\sum_{i=1}^k\theta_{i}^{\zeta}(
         \tilde{J}(\bx_i))}\right]-\int_{\mX}f(\bx)\pi(d\bx)\right|\\
        \leq &\underbrace{\E\left[\left|\frac{\sum_{i=1}^k\theta_{i}^{\zeta}
         (\tilde{J}(\bx_i))f(\bx_i)}
         {\sum_{i=1}^k\theta_{i}^{\zeta}(\tilde{J}(\bx_i)) }-\frac{Z_{\btheta_{\star}}\sum_{i=1}^k\theta_{i}^{\zeta} (\tilde{J}(\bx_i)) f(\bx_i)}{k}\right|\right]}_{\text{I}_1}\\
        &\ \ \ \ + \underbrace{\E\left[\frac{Z_{\btheta_{\star}}}{k}\sum_{i=1}^k\left|\theta_i^{\zeta} (\tilde{J}(\bx_i))-\theta_{\star}^{\zeta}
      (\tilde{J}(\bx_i))  \right| \cdot |f(\bx_i)|\right]}_{\text{I}_2} +\underbrace{\left|\E\left[\frac{Z_{\btheta_{\star}}}{k}\sum_{i=1}^k\theta_{\star}^{\zeta}
     (\tilde{J}(\bx_i)) f(\bx_i)\right]-\int_{\mX}f(\bx)\pi(d\bx)\right|}_{\text{I}_3}.
    \end{split}
\end{equation*}

% For the first term $\text{I}_1$, by Cauchy–Schwarz inequality, we have
% \begin{equation*}
% \small
% \begin{split}
%     \text{I}_1=&\E\left[\left|\frac{\sum_{i=1}^k\btheta_{i}(I_{\widetilde U}(\bx_i))^{\zeta} f(\bx_i)}{\sum_{i=1}^k\btheta_{i}(I_{\widetilde U}(\bx_i))^{\zeta}}\left(1-\sum_{i=1}^k\frac{\btheta_{i}(I_{\widetilde U}(\bx_i))^{\zeta}}{k}Z_{\btheta_{\star}}\right)\right|\right]\leq \sqrt{\E\left[\left|1-\frac{\sum_{i=1}^k\btheta_{i}(I_{\widetilde U}(\bx_i))^{\zeta}}{k}Z_{\btheta_{\star}}\right|^2\right]\E\left[\left|\frac{\sum_{i=1}^k\btheta_{i}(I_{\widetilde U}(\bx_i))^{\zeta}f(\bx_i)}{\sum_{i=1}^k\btheta_{i}(I_{\widetilde U}(\bx_i))^{\zeta}}\right|^2\right]}\\
% \end{split}
% \end{equation*}

For the first term $\text{I}_1$, 
by the boundedness of $\bTheta$ and $f$ and $\inf_{\btheta, i}\theta(i)^{\zeta}>0$, we have
\begin{equation*}
\small
\begin{split}
    \text{I}_1=&\E\left[\left|\frac{\sum_{i=1}^k\theta_{i}^{\zeta}(\tilde{J}(\bx_i))  f(\bx_i)}{\sum_{i=1}^k\theta_{i}^{\zeta}
    (\tilde{J}(\bx_i))
    }\left(1-\sum_{i=1}^k\frac{\theta_i^{\zeta}
    (\tilde{J}(\bx_i))
    }{k}Z_{\btheta_{\star}}\right)\right|\right]\\
    %%%%%%
    \lesssim & \E\left[\left|Z_{\btheta_{\star}}\frac{{\sum_{i=1}^k\theta_{i}^{\zeta}
    (\tilde{J}(\bx_i))
    }}{k}-1\right|\right]\\
    =&\E\left[\left|Z_{\btheta_{\star}}\sum_{i=1}^m \frac{\sum_{j=1}^k\left( \theta_j^{\zeta}(i)-\theta_{\star}^{\zeta}(i)+\theta_{\star}^{\zeta}(i)\right)1_{
    \tilde{J}(\bx_j)=i}}{k}-1\right|\right]\\
    %%%%%%%%%%%
    \leq & \underbrace{\E\left[Z_{\btheta_{\star}}\sum_{i=1}^m \frac{\sum_{j=1}^k\left| \theta_j^{\zeta}(i)-\theta_{\star}^{\zeta}(i)\right| 1_{\tilde{J}(\bx_j)=i}}{k} \right]}_{\text{I}_{11}} + \underbrace{\E\left[\left| Z_{\btheta_{\star}}\sum_{i=1}^m \frac{\theta_{\star}^{\zeta}(i)\sum_{j=1}^k  1_{\tilde{J}(\bx_j)=i}}{k}-1\right|\right]}_{\text{I}_{12}}\\
\end{split}
\end{equation*}

Regarding $\text{I}_{11}$, first applying $|x^{\zeta}-y^{\zeta}|\leq |x-y| z^{\zeta}$ for any $\zeta>0, x\leq y$ and $z\in[x, y]$ based on the mean-value theorem and then using Cauchy–Schwarz inequality
\begin{equation}
    \text{I}_{11}\lesssim \frac{1}{k}\E\left[ \sum_{j=1}^k\sum_{i=1}^m\left| \theta_j^{\zeta}(i)-\theta_{\star}^{\zeta}(i)\right| \right]\lesssim  \frac{1}{k}\E\left[ \sum_{j=1}^k\sum_{i=1}^m\left| \theta_j(i)-\theta_{\star}(i)\right| \right]\lesssim  \frac{1}{k}\sqrt{\sum_{j=1}^k\E\left[\left\| \btheta_j-\btheta_{\star}\right\|^2\right]},
\end{equation}
where the compactness of $\Theta$ has been 
used in deriving the second inequality. 

For $\text{I}_{12}$, considering the following relation $$
    1=\sum_{i=1}^m\int_{\mX_i} \pi(\bx)d\bx=\sum_{i=1}^m\int_{\mX_i} \theta_{\star}^{\zeta}(i) \frac{\pi(\bx)}{\theta_{\star}^{\zeta}(i)}d\bx
    =Z_{\btheta_{\star}}\sum_{i=1}^m \theta_{\star}^{\zeta}(i)\int_{\mX} 1_{\tilde{J}(\bx)=i}\varpi_{\btheta_{\star}}(\bx)d\bx,$$ then we have
\begin{equation}
\begin{split}
    \text{I}_{12}&=\E\left[\left|Z_{\btheta_{\star}}\sum_{i=1}^m \theta_{\star}^{\zeta}(i)\left(\frac{\sum_{j=1}^k 1_{
    \tilde{J}(\bx_j)=i}}{k}-\int_{\mX}1_{
    \tilde{J}(\bx)=i}\varpi_{\btheta_{\star}}(\bx)d\bx\right)\right|\right]\\
    %%%%%
    &\lesssim \sum_{i=1}^m\E\left[\left|\frac{\sum_{j=1}^k 1_{ \tilde{J}(\bx_j)=i}
    }{k}-\int_{\mX} 1_{\tilde{J}(\bx)=i }\varpi_{\btheta_{\star}}(\bx)d\bx\right|\right]\\
    &=\mathcal{O}\left(\frac{1}{k\epsilon}+\epsilon+\sqrt{\frac{\sum_{i=1}^k \omega_k}{k}}+\sup_{n\geq k_0}\E[\|\delta(m,n, \epsilon, \btheta_{n}, \bx_{n+1})\|]^{0.5}\right),
\end{split}
\end{equation}
where the last equality follows  from Lemma \ref{avg_converge_appendix} as the indicator 
function $1_{ \tilde{J}(\bx)=i}$ is integrable.

For $\text{I}_2$, by the boundedness of $f$,  
the mean value theorem and Cauchy-Schwarz inequality, 
we have 
\begin{equation*}
\small
    \begin{split}
        \text{I}_2&\lesssim \E\left[\frac{1}{k}\sum_{i=1}^k\left|\theta_{i}
        ^{\zeta}(\tilde{J}(\bx_i)) -\theta_{\star}^{\zeta}(
        \tilde{J}(\bx_i))\right|\right]\lesssim  \frac{1}{k}\E\left[ \sum_{j=1}^k\sum_{i=1}^m\left| \theta_j(i)-\theta_{\star}(i)\right| \right]\lesssim  \frac{1}{k}\sqrt{\sum_{j=1}^k\E\left[\left\| \btheta_j-\btheta_{\star}\right\|^2\right]}.\\
    \end{split}
\end{equation*}

For the last term $\text{I}_3$, we first decompose $\int_{\mX} f(\bx) \pi(d\bx)$ into $m$ disjoint regions to facilitate the analysis
\begin{equation}
\label{split_posterior}
\small
\begin{split}
      \int_{\mX} f(\bx) \pi(d\bx)=\int_{\cup_{j=1}^m \mX_j}  f(\bx) \pi(d\bx)=\sum_{j=1}^m\theta_{\star}^{\zeta}(j)\int_{\mX_j}  f(\bx) \frac{\pi(d\bx)}{\theta_{\star}^{\zeta}(j)}=Z_{\btheta_{\star}}\sum_{j=1}^m\theta_{\star}(j)^{\zeta}\int_{\mX_j}  f(\bx) \varpi_{\bm{\theta}_{\star}}(d\bx).\\
\end{split}
\end{equation}

Plugging (\ref{split_posterior}) into the third term $\text{I}_3$, we have
\begin{equation}
\label{final_i2}
\small
    \begin{split}
        \text{I}_3&=\left|\E\left[\frac{Z_{\btheta_{\star}}}{k}\sum_{i=1}^k\sum_{j=1}^m\theta_{\star}(j)^{\zeta} f(\bx_i)1_{  \tilde{J}(\bx_i)=j
        }\right]-\int_{\mX}f(\bx)\pi(d\bx)\right|\\
        %%%%
        &=Z_{\btheta_{\star}}\left|\sum_{j=1}^m\theta_{\star}^{\zeta}(j)\E\left[\frac{1}{k}\sum_{i=1}^k f(\bx_i)
        1_{ \tilde{J}(\bx_i)=j 
        }\right]-\sum_{j=1}^m\theta_{\star}^{\zeta}(j)\int_{\mX_j}  f(\bx) \varpi_{\bm{\theta}_{\star}}(d\bx)\right|\\
        %%%%%%
        &\leq Z_{\btheta_{\star}}\sum_{j=1}^m\theta_{\star}^{\zeta}(j)\left|\E\left[\frac{1}{k}\sum_{i=1}^k f(\bx_i)1_{
        \tilde{J}(\bx_i)=j 
        }\right]-\int_{\mX_j}  f(\bx) \varpi_{\bm{\theta}_{\star}}(d\bx)\right|.\\
    \end{split}
\end{equation}

Given any $j\in \{1,2,...,m\}$, applying the function $f(\bx)1_{\tilde{J}(\bx)=j}$ 
to Lemma \ref{avg_converge_appendix} yields
\begin{equation}
\label{almost_i2}
\small
\begin{split}
      \left|\E\left[\frac{1}{k}\sum_{i=1}^k f(\bx_i)1_{
      \tilde{J}(\bx_i)=j 
      }\right]-\int_{\mX_j}  f(\bx) \varpi_{\bm{\theta}_{\star}}(d\bx)\right|\leq \mathcal{O}\left(\frac{1}{k\epsilon}+\epsilon+\sqrt{\frac{\sum_{i=1}^k \omega_k}{k}}+\sup_{i\geq k_0}\E[\|\delta(m, n, \epsilon, \btheta_{i}, \bx_{i+1})\|]^{0.5}\right).\\
\end{split}
\end{equation}

Plugging (\ref{almost_i2}) into (\ref{final_i2}) and combining $\text{I}_{11}$, $\text{I}_{12}$, $\text{I}_2$ and Theorem \ref{latent_convergence}, we have
\begin{equation}
\small
\begin{split}
      \left|\E\left[\frac{\sum_{i=1}^k\theta_{i}
     ^{\zeta}(\tilde{J}(\bx_i)) f(\bx_i)}{\sum_{i=1}^k\theta_{i}^{\zeta}( 
      \tilde{J}(\bx_i))}\right]-\int_{\mX}f(\bx)\pi(d\bx)\right|\leq \mathcal{O}\left(\frac{1}{k\epsilon}+\epsilon+\sqrt{\frac{\sum_{i=1}^k \omega_k}{k}}+\sup_{i\geq k_0}\E[\|\delta(m, n, \epsilon, \btheta_{i}, \bx_{i+1})\|]^{0.5}\right).\\
\end{split}
\end{equation}

\end{proof}

\bibliography{myref}
\bibliographystyle{plainnat}